\documentclass[11pt]{article}
\usepackage[a4paper,margin=1in]{geometry}
\usepackage{amsmath,amssymb,amsthm,mathtools,bm}
\usepackage{booktabs,array,multirow}
\usepackage{graphicx}
\graphicspath{{figures/}}
\usepackage{xcolor}
\usepackage{microtype}
\microtypesetup{expansion=false}
\usepackage{enumitem}
\usepackage{subcaption}
\usepackage{float}
\usepackage{hyperref}
\usepackage{wrapfig}
\usepackage{url}
\usepackage{cleveref} % must come after hyperref
\usepackage{booktabs,longtable}
\usepackage{natbib}
\hypersetup{colorlinks=true,citecolor=blue!55!black,linkcolor=blue!55!black,urlcolor=blue!55!black}
\setlist{nosep,leftmargin=*}

\newtheorem{theorem}{Theorem}[section]

\newtheorem{corollary}[theorem]{Corollary}

\newtheorem{assumption}[theorem]{Assumption}
\theoremstyle{definition}
\newtheorem{remark}[theorem]{Remark}

\newcommand{\R}{\mathbb{R}}
\newcommand{\rj}{\mathcal{R}}
\newcommand{\lj}{\mathcal{L}}

\newcommand{\N}{\mathcal{N}}

\newcommand{\KL}{\operatorname{KL}}
\newcommand{\dd}{\mathrm{d}}
\newcommand{\push}{{}_{\#}}
\newcommand{\prior}{\pi_0}

\newcommand{\loss}{\ell_y}
\newcommand{\clip}{\mathrm{CLIP}}

\newcommand{\RFM}{\mathcal{R}_{\mathrm{FM}}}
\newcommand{\LFM}{\mathcal{L}_{\mathrm{FM}}}

\newcommand{\RMF}{\mathcal{R}_{\mathrm{MF}}}
\newcommand{\LMF}{\mathcal{L}_{\mathrm{MF}}}

\newcommand{\RiMF}{\mathcal{R}_{\mathrm{iMF}}}
\newcommand{\LiMF}{\mathcal{L}_{\mathrm{iMF}}}

\newcommand{\epsopt}{\epsilon_{\mathrm{opt}}}
\newcommand{\epsapp}{\epsilon_{\mathrm{app}}}

\newlength{\panelwidth}
\title{Posterior sampling by source-space  MCMC via prior-based few-step transport maps}

\author{
Hoang Phuc Hau Luu \thanks{
Division of Mathematical Sciences,  School of Physical and Mathematical Sciences, Nanyang Technological University, Singapore.
\texttt{hoangph.luu@ntu.edu.sg}
}
\and
Marcelo Hartmann\thanks{
Department of Computer Science, Faculty of Science, University of Helsinki, Finland.\\ \texttt{marcelo.hartmann@helsinki.fi}
}
\and
Zhongjian Wang\thanks{
Division of Mathematical Sciences, School of Physical and Mathematical Sciences, Nanyang Technological University, Singapore.
\texttt{zhongjian.wang@ntu.edu.sg} (Corresponding)
}
}

\date{\today}

\begin{document}

\maketitle

\begin{abstract}
Bayesian inference increasingly uses informative but implicit priors represented only by samples, such as historical ensembles, simulator outputs, and pretrained generative models. The same computational problem appears in the test-time guidance task (generalized Bayes), where an explicit positive weight, e.g., an exponentiated reward, tilts an implicit prior. We develop a framework for source-space generalized Bayesian inference that combines inexpensive few-step prior transports with posterior stability guarantees. Specifically, we represent the prior using a one- or few-step improved MeanFlow (iMF) map and perform posterior sampling in its Gaussian source space. We establish Wasserstein error bounds between the exact and learned posteriors in terms of the joint population iMF and auxiliary-velocity loss, decomposed into training suboptimality and model-class approximation error. In the iMF source space, we adopt parallel tempering with preconditioned Crank–Nicolson updates and introduce a hybrid variant that incorporates split Hamiltonian Monte Carlo to improve sampling efficiency. Synthetic experiments show that the proposed framework can approximate posterior distributions accurately and efficiently, while CLIP-guided ImageNet experiments demonstrate its ability to steer a pretrained iMF image prior toward text-specified preferences.
\end{abstract}

\section{Introduction}

Bayesian inference combines prior knowledge with observations to characterize uncertainty about unknown quantities \citep{stuart2010inverse}. In many applications, however, prior information is available through samples, e.g., historical ensembles, simulator outputs, or pretrained generative models rather than an explicit density \citep{patel2019bayesian,chung2022diffusion,lunz2018adversarial}. Such prior samples can represent complex structures that are difficult to encode analytically, but their densities and scores are not directly available for posterior Markov chain Monte-Carlo (MCMC).

Formally, let $x\in\R^{d_x}$ be an unknown with prior $\prior$ available only through i.i.d.\ samples, and let $y$ denote observations or a task specification. We encompass Bayesian and generalized Bayesian inference through the generalized posterior
\begin{equation}
\label{eq:genposterior}
\pi^y(\dd x)
= Z_y^{-1} \ell_y(x)\prior(\dd x),
\qquad
Z_y = \int \ell_y(x)\prior(\dd x),
\end{equation}
where $\ell_y$ is an analytically given nonnegative weight and $0<Z_y<\infty$. Standard Bayesian inference takes $\ell_y(x)=p(y | x)$, while generalized Bayes \citep{bissiri2016general} uses $\ell_y(x)=\exp{(-\tau L_y(x))}$ for a task loss $L_y$ and $\tau>0$. Reward guidance \citep{korbak2022rl} fits the same formulation with $\ell_y(x)=\exp{(\tau R_y(x))}$, favoring samples with larger reward $R_y$. For example, $y$ may be a text prompt, with $R_y(x)=\clip(x,y)$ measuring the cosine similarity between the CLIP embeddings of image $x$ and prompt $y$ \citep{radford2021learning}. Henceforth, we use $\ell_y$ for either a likelihood or a general weight and refer to $\pi^y$ simply as the posterior.

An approach is to learn the data-space prior score by denoising score
matching (DSM) \citep{vincent2011connection} and combine it with the
gradient of an analytic log-likelihood or reward to construct Langevin
or Hamiltonian dynamics. These samplers are typically unadjusted because
a learned score alone does not provide the density ratios needed for
Metropolis correction. Moreover, sharp variation in the score near
complex data manifolds can produce stiff dynamics, requiring small
integration steps.

Another approach is to first learn a deterministic generator $S_\theta:\R^{d_z}\to\R^{d_x}$ satisfying $(S_\theta)\push\gamma\approx\prior$, where $\gamma=\N(0,I)$ denotes the source distribution, $\mathbb{R}^{d_z}$ is called the source space and $_\#$ denotes the pushforward. Under the learned model $S_{\theta}$, the complex data-space prior is represented by the simple Gaussian reference measure $\gamma$. The corresponding \emph{learned} posterior in the source space is
\begin{equation}
  \nu_{\theta}^y(\dd z)
  \propto \loss(S_\theta(z))\gamma(\dd z).
  \label{eq:learned-latent-posterior}
\end{equation}
We sample from $\nu_{\theta}^y$ using MCMC and map the resulting samples to data space through $S_\theta$. In contrast to the conventional data-space approach, $\nu_{\theta}^y$ is available up to normalization, enabling Metropolis correction and thus an MCMC kernel that leaves $\nu_{\theta}^y$ invariant. Furthermore, the source-space posterior is expected to have more favorable geometry for MCMC: its prior is Gaussian, while $S_\theta$ has already absorbed much of the curvature and anisotropy of $\pi_0$. Nevertheless, the likelihood/weight $\loss\circ S_\theta$ may reintroduce curvature, stiffness, anisotropy, or multimodality, particularly when the observation model is highly informative or nonlinear. Early work by \citet{patel2022solution} used a Wasserstein GAN to construct $S_\theta$ and applied Hamiltonian Monte Carlo (HMC) with the No-U-Turn Sampler \citep{hoffman2014no} to the resulting low-dimensional latent posterior. However, learning $S_\theta$ with a WGAN requires solving an adversarial min-max optimization problem, which can be challenging to stabilize. Flow-based models \citep{chen2018neural} offer an attractive alternative: flow matching learns a time-dependent vector field through a standard regression objective \citep{lipman2023flow,tong2024improving}. This provides a scalable approach to learning expressive transport maps for complex and multimodal priors. Under suitable regularity conditions on the vector field, the resulting ODE flow also defines a diffeomorphism between source and data spaces \citep{papamakarios2021normalizing}. Source-space guidance for flow-matching models has been explored through both optimization and sampling. D-Flow \citep{ben2024d} optimizes the source variable, while Source-Guided Flow Matching (SGFM) \citep{wang2026source} samples from the source posterior distribution \eqref{eq:learned-latent-posterior}. Source Parallel Tempering (SPT) \citep{wangtest} combines local preconditioned Crank-Nicolson (pCN) \citep{beskos2008mcmc} updates with replica exchanges \citep{earl2005parallel} across tempered source distributions to address multimodality. These source-space methods are instances of \emph{test-time guidance}, incorporating user-specified likelihoods or weights during inference.

For methods that represent the prior through a learned transport \citep{wang2026source,wangtest}, a central theoretical question is how prior-learning error controls the discrepancy between the learned posterior $(S_\theta)_{\#}\nu_{\theta}^y$ and the true posterior $\pi^y$ in \eqref{eq:genposterior}. We refer to this as \emph{posterior stability}. Existing works have tried to address this question, but a gap remains. \citet{wang2026source} assume a uniform vector-field approximation bound, $\Vert v_\theta(x,t)-v(x,t)\Vert_{\infty}\leq\epsilon$ where $v(x,t)$ is the true velocity field and $v_\theta(x,t)$ is the learned velocity field. However, the flow-matching objective controls an average squared error rather than a uniform error. \citet{wangtest} address this mismatch by establishing posterior stability in terms of the $L^2$ transport-map error, i.e., $\mathbb{E}_{z \sim \gamma}\Vert S(z)-S_{\theta}(z)\Vert^2\leq\epsilon^2$, where $S$ is the exact transport map induced by $v(x,t)$. Their connection to the flow matching objective takes $S_{\theta}$ to be the map induced by the \emph{exactly integrated} learned flow, via Grönwall's inequality. In practice, however, numerical ODE solvers are employed, introducing an additional discretization error. Under standard regularity and stability assumptions, a numerical ODE solver of order $k$ has global error $O(h^k)$ \citep{hairer1993solving}. The training objective alone does not control this error. Reducing $h$ requires more integration steps, increasing the cost of source-space MCMC. Minibatch optimal-transport coupling \citep{tong2024improving}, as used by \citet{wang2026source}, can encourage straighter trajectories and facilitate few-step integration, but does not by itself provide a bound on this numerical error; their theoretical bound likewise concerns the \emph{exactly integrated} learned flow. Thus, guarantees for the exactly integrated learned flow do not by themselves establish posterior stability for a fixed few-step implementation. Likewise, \citet{purohit2024posterior} use few-step consistency models but do not establish how the consistency-distillation training objective controls the resulting posterior error.

In this work, we pursue two complementary objectives: making source-space inference computationally efficient through few-step transports \emph{and} establishing posterior stability guarantees for the implemented maps via the true population training objective. To this end, we use improved MeanFlow (iMF) \citep{geng2026improved}, which learns interval-average velocities without distillation and natively enables one-step or few-step generation without numerical ODE integration. We further explain why the original MeanFlow objective \citep{geng2026mean} does not directly yield the same guarantee, highlighting that native few-step generation alone can fall short of establishing such stability results. Practically, the iMF maps make source-space gradients inexpensive to evaluate when available, allowing MCMC kernels to optionally exploit them to improve mixing.

\paragraph{Contributions} We use iMF transport mappings to represent complex priors available only through samples in generalized Bayesian inference settings. Under stated regularity assumptions, we connect population training suboptimality and model-class approximation error for the joint iMF and auxiliary-velocity objective to posterior stability for the implemented one-step or few-step maps. Building on these transports, we adopt SPT with pCN updates \citep{wangtest} to facilitate exploration across modes while exploiting the Gaussian reference structure of the high-dimensional source space. When likelihood or reward gradients are available, we additionally introduce Gaussian-reference split HMC \citep{beskos2011hybrid} at the coldest temperature of SPT to exploit local posterior geometry and potentially improve mixing. Both local kernels preserve the Gaussian reference measure in their reference dynamics, making them natural choices for inference in high-dimensional source spaces. We evaluate the proposed framework on synthetic posterior sampling and CLIP-guided ImageNet generation, demonstrating accurate and efficient posterior approximation in the synthetic settings and the ability to steer a pretrained image prior toward text-specified preferences.

\subsection*{Other related work}

\paragraph{Transport MCMC}
Transport maps have long been used to improve MCMC by transforming difficult posterior target distributions into coordinates easier for sampling \citep{parno2018transport,Benjamin}. NeuTra HMC learns a variational transport tailored to posterior geometry, with an amortized extension that accommodates different observations within a specified probabilistic model \citep{hoffman2019neutra}.

\paragraph{Score-based and diffusion inverse solvers}
Diffusion-based inverse solvers offer another way to combine a learned sample-only prior with observations, typically by incorporating likelihood information into the reverse generative process. Diffusion posterior sampling (DPS) \citep{chung2022diffusion} is a prominent example, enabling posterior-oriented sampling. Its approximate guidance, however, does not generally preserve the desired posterior and can introduce \emph{systematic} bias beyond that arising from prior learning and numerical discretization. Denoising-score MCMC \citep{ramzi2020denoising} learns the prior score from samples using denoising score matching \citep{vincent2011connection} and combines it with an analytically available likelihood score to perform MCMC directly in data space.

\paragraph{Training-time posterior learning.}
Complementary to test-time guidance, amortized posterior learning trains observation-conditional generators to directly approximate posterior distributions. Neural posterior estimation and related methods learn from simulated parameter-observation pairs, enabling rapid inference for new observations \citep{greenberg2019automatic,radev2023bayesflow,cheng2026preconditioned}. The training cost is shared across subsequent inference tasks, but changes to the observation model beyond those covered during training may require retraining or adaptation.

\section{Background}

\subsection{Flow matching}

Flow matching (FM) \citep{lipman2023flow} learns a velocity field approximating the marginal ODE $\dot w_t=v(w_t,t)$, which transports a simple source distribution $p_{\mathrm{src}}$, typically Gaussian, to a target distribution $p_{\mathrm{tgt}}$ available through samples. Let $X\sim p_{\mathrm{tgt}}$ and $Z\sim p_{\mathrm{src}}$
be independent target and source samples, and define
$W_t=(1-t)X+tZ$ and the samplewise interpolation velocity $V=Z-X$. FM approximates the marginal instantaneous velocity $v(w,t)=\mathbb{E}[V\mid W_t=w]$ with a learned field $v_\theta$, whose approximation error is measured by the marginal FM risk
\begin{align}
\label{eq:RFM}
    \RFM(\theta)=\mathbb{E} \Vert v_{\theta}(W_T,T) - v(W_T,T) \Vert^2_2.
\end{align}
where $T \sim \operatorname{Unif}[0,1]$ independent of $X,Z$ and the above expectation is taken over $T,X,Z$. Since $v$ is intractable, we replace $v$ by $V$, resulting in the population training loss
\begin{equation}
\label{eq:RCFM}
    \LFM(\theta)
    =\mathbb{E}
    \left\|v_\theta(W_T,T)-V\right\|_2^2.
\end{equation}
Minimizing $\LFM$ is equivalent to minimizing $\RFM$ since $\LFM(\theta) = \RFM(\theta) + C$ where $C =\mathbb{E} \Vert V-v(W_T,T)\Vert_2^2$ does not depend on $\theta$ \citep{lipman2023flow}. Once trained, generation integrates $\dot w_t=v_\theta(w_t,t)$ from $t=1$ to $t=0$, using a numerical ODE solver.

\subsection{MeanFlow and improved MeanFlow}
\label{subsect:MFiMF}
\emph{MeanFlow} \citep{geng2026mean} instead parameterizes the \emph{average velocity} $$u(w_t,r,t)
    =\frac{1}{t-r}\int_r^t v(w_\tau,\tau)\,d\tau.$$ Fundamental calculus yields the so-called MeanFlow identity
\begin{align}
    \label{eq:mfid}
    u(w_t,r,t) = v(w_t,t) - (t-r) \left(\partial_t u + (\partial_w u)v(w_t,t) \right).
\end{align}
Based on this identity, the original MeanFlow objective regresses an average velocity $u_{\theta}$ to the right hand side of \eqref{eq:mfid}, and since $\partial_t u$ and $\partial_w u$ are unknown, it replaces those by $\partial_t u_{\theta}$ and $\partial_w u_{\theta}$ to get
\begin{equation}
\label{rmf}
    \mathcal{R}_{\mathrm{MF}}(\theta)
    =\mathbb{E}
    \left\|
      u_\theta(W_T,R,T)
      -\!\left[v(W_T,T)-(T-R)\left(\partial_t u_{\theta} + (\partial_w u_{\theta}) v(W_T,T) \right)\right]
    \right\|_2^2,
\end{equation}
where the expectation is taken over $R,T,Z,X$ (recall $W_t = (1-t)X + tZ$). We omit the \texttt{stop-gradient} operator when defining population risks, since they leave objective values unchanged. Replacing $v$ by $V$ gives a tractable training loss,
\begin{equation}
\label{rcmf}
    \LMF(\theta)
    =\mathbb{E}
    \left\|
      u_\theta(W_T,R,T)
      -\left[V-(T-R)\left(\partial_t u_{\theta} + (\partial_w u_{\theta}) V \right)\right]
    \right\|_2^2.
\end{equation}
Unlike flow matching, the gap $\LMF(\theta) - \RMF(\theta)$ generally depends on $\theta$, making them inequivalent as scalar objectives\footnote{The \texttt{stop-gradient} makes the two objectives share expected gradients \citep{lu2026variance}, but this is an optimization device; our theory concerns objective values, not training dynamics.}. Moreover, the training objective \eqref{rcmf} is not standard since its effective target depends on the network itself. To address this, \emph{improved MeanFlow (iMF)} \citep{geng2026improved} replaces the interpolation velocity in $(\partial_w u_{\theta})V$ in $\LMF$ by a learnable marginal velocity $v_\theta$, defining
\begin{align}
\LiMF(\theta)
    &=\mathbb{E}\left\|u_\theta(W_T,R,T)
      +(T-R)\!\left[ \partial_t u_{\theta} +(\partial_w u_{\theta}) v_{\theta}(W_T,T) \right]-V\right\|_2^2.
\end{align}
The corresponding marginal iMF risk is then
\begin{align}
\label{def:rimf}
    \RiMF(\theta) = \mathbb{E}\left\|u_\theta(W_T,R,T)
      +(T-R)\!\left[ \partial_t u_{\theta} +(\partial_w u_{\theta}) v_{\theta}(W_T,T) \right]-v(W_T,T)\right\|_2^2
\end{align}
Importantly, $ \LiMF(\theta) - \RiMF(\theta)$ is independent of $\theta$ (Appendix \ref{app:equvi}), recovering the equivalence between these objectives. $v_\theta(w_t,t)$ can be represented by the boundary value $u_\theta(w_t,t,t)$ or an auxiliary velocity head. In this work, we consider the latter, and the auxiliary head is explicitly supervised through FM loss \eqref{eq:RCFM}, $\LFM(\theta)=\mathbb{E}\left\Vert v_\theta(W_T,T)-V\right\Vert_2^2$, yielding the joint training loss
\begin{align}
\label{eq:JC}
\lj(\theta) =\LiMF(\theta)
+\lambda\LFM(\theta),
\qquad \lambda>0.
\end{align}
We also denote the joint risk 
\begin{align}
\label{eq:J}
    \rj(\theta) =\mathcal{R}_{\mathrm{iMF}}(\theta)
+\lambda\mathcal{R}_{\mathrm{FM}}(\theta).
\end{align}
The auxiliary velocity field $v_\theta$ is used only during training; sampling relies solely on $u_\theta$, with the \emph{one-step} map given by
\begin{equation}
\label{eq:imf1}
    S_\theta(z)
    =z-u_\theta(z,0,1),
    \qquad z\sim p_{\mathrm{src}},
\end{equation}
More generally, a fixed schedule $1=t_K>t_{K-1}>\cdots>t_0=0$
defines a \emph{few-step} map through
\begin{equation}
\label{eq:fsimf}
w_K=z,\qquad
w_{k-1}
=w_k-(t_k-t_{k-1})
u_\theta(w_k,t_{k-1},t_k),
\quad k=K,\ldots,1,
\qquad
S_\theta(z)=w_0.
\end{equation}

\section{Source-space sampling via improved MeanFlow}

\subsection{Bayesian inference through a learned prior transport}\label{sec:source-posterior}

Let $\pi_0$ be a prior accessed through samples and $\ell_y$ a nonnegative likelihood or weight associated with observations or a task $y$. We aim to sample from the posterior $\pi^y(dx) \propto \ell_y(x) \pi_0(dx)$. The idea of source-space sampling \citep{purohit2024posterior,wang2026source} is to learn a transport map $({S_{\theta}})_{\#} \gamma \approx \pi_0$ where $\gamma = \mathcal{N}(0,I)$ and perform posterior sampling in the source space of $\gamma$, targeting the pull-back posterior $\nu_{\theta}^y(dz) \propto \ell_y(S_{\theta}(z)) \gamma(dz)$ then transport these source samples to the data space via $S_{\theta}$. In case the transport map is exact, i.e., $({S_{\theta}})_{\#} \gamma = \pi_0$, an exact posterior source sample $z \sim \nu_{\theta}^y(dz)$ will give an exact sample, $S_{\theta}(z)$, from $\pi^y(dx)$ \citep[Thm. 1]{wang2026source}. It is then critical to learn a good prior transport $S_{\theta}$ and perform efficient sampling in the source space of $S_{\theta}$.

When $S_{\theta}$ is not perfect, we seek to quantify the resulting posterior error. Let $S$ denote the exact transport, i.e., $S_{\#} \gamma = \pi_0$, associated with the generative construction, assumed to exist under suitable regularity conditions. Define the ideal source posterior by $\nu^y(dz) \propto \ell_y(S(z))\gamma(dz)$. Our goal is to control the source-space discrepancy $W_2(\nu_{\theta}^y, \nu^y)$ and the data-space discrepancy $W_2((S_{\theta})_{\#}\nu_{\theta}^y, \pi^y)$ in terms of the \emph{training error associated with learning} $S_\theta$ (for iMF, it links to $\rj(\theta)$, see next section). Here $W_2$ denotes the Wasserstein distance \citep{villani2009optimal}.

\subsection{Improved MeanFlow source and its posterior stability}
\label{sec:theory}

We parameterize $S_\theta$ using improved MeanFlow \eqref{eq:imf1} and \eqref{eq:fsimf}, trained with the joint training loss $\lj(\theta) = \LiMF(\theta) + \lambda \LFM(\theta)$ in \eqref{eq:JC}. iMF enables few-step generation and ensures that $\lj$ differs from the marginal risk $\rj$ \eqref{eq:J} only by an additive constant independent of $\theta$. Consequently,
for a fixed model class $\Theta$,
\begin{align}
\label{eq:nothing}
    \rj(\theta)
    =
    \underbrace{
        \lj(\theta)
        - \inf_{\vartheta\in\Theta}\lj(\vartheta)
    }_{:=\epsopt,~\text{population training suboptimality}}
    +
    \underbrace{
        \inf_{\vartheta\in\Theta}\mathcal{\rj}(\vartheta)
    }_{:=\epsapp,~\text{model-class approximation error}}:=\delta.
\end{align}
See App. \ref{apx:nothing} for proof. The decomposition \eqref{eq:nothing} shows that the marginal risk $\rj(\theta)$ is controlled by population suboptimality of the training objective $\lj$ used in practice and model-class approximation error. By the MeanFlow identity \eqref{eq:mfid}, the exact instantaneous and interval-average velocity fields jointly achieve zero marginal risk \citep{geng2026improved}, so a sufficiently expressive model class can in principle make $\epsapp$ small. We therefore seek to bound posterior error in terms of $\rj(\theta)$.

We also note that, for original MeanFlow \citep{geng2026mean}, the parameter-dependent gap $\LMF(\theta)-\RMF(\theta)$ prevents the decomposition \eqref{eq:nothing} in general. Small training suboptimality and approximation error therefore need not imply small marginal risk.

We establish posterior stability by first controlling the error of the implemented iMF transport, building on the flow-map error analysis of \citep{boffi2024flow}. For the analysis, we use independent coupling $X\sim\pi_0, Z\sim\gamma$ and time $R,T \overset{\mathrm{iid}}\sim\operatorname{Unif}[0,1]$\footnote{A more general schedule is possible, but we use a simpler one for clarity of presentation.
}. Also, let $V=Z-X$, $W_t=(1-t)X+tZ$, $v(w,t)=\mathbb{E}[V|W_t=w]$, and $p_t=\operatorname{Law}(W_t)$. Thus, $p_0=\pi_0$ and $p_1=\gamma$.
\begin{assumption}[Prior and marginal flow]
\label{assm:prior-flow}
The prior $\pi_0\in\mathcal P_2(\mathbb R^d)$.
For each $t\in[0,1]$ and $p_t$-a.e. $w$, the ODE $\frac{d}{ds}\xi_s^{t,w}=v(\xi_s^{t,w},s),~\xi_t^{t,w}=w$, admits absolutely continuous solutions on $[0,t]$ whose solution maps $\varphi_{s,t}(w):=\xi_s^{t,w}$ form a
measurable flow satisfying $(\varphi_{s,t})_\#p_t=p_s, 0\leq s\leq t\leq 1$.
\end{assumption}

Under Assumption~\ref{assm:prior-flow}, the map
$S:=\varphi_{0,1}$ is an exact prior transport:
$S_\#\gamma=\pi_0$.

\begin{assumption}[Regularity of the learned average velocity]
\label{assm:u-regularity}
The field $u_\theta$ is continuously differentiable on
$\mathbb R^d\times(0,1)^2$, and both $u_\theta$ and its
first derivatives extend continuously to $\mathbb R^d\times[0,1]^2$.
\end{assumption}

For $r,t\in[0,1]$, define the marginal iMF regression residual
\begin{align*}
\bar q_\theta(w,r,t)
&:=u_\theta(w,r,t)-v(w,t)
+(t-r)\left[
\partial_tu_\theta(w,r,t)
+\partial_wu_\theta(w,r,t)\,v_\theta(w,t)
\right]\\
G_\theta(r)
&:=\int_r^1
\|\bar q_\theta(\cdot,r,t)\|_{L^2(p_t)}^2\,dt.
\end{align*}
Thus, the iMF risk in \eqref{def:rimf} is rewritten as $\RiMF(\theta)=\mathbb E\|\bar q_\theta(W_T,R,T)\|^2$. We impose the following.

\begin{assumption}[Regularity across starting times]
\label{assm:LG}
$G_\theta:[0,1]\to[0,\infty)$ is
$L_G$-Lipschitz, i.e.,
\[
|G_\theta(r)-G_\theta(r')|
\leq L_G|r-r'|,
\qquad \forall r,r'\in[0,1].
\]
\end{assumption}

\begin{remark}
\label{rmk:Gthetareg}
Assumption~\ref{assm:LG} holds if $\bar q_\theta$ is
continuously differentiable in $r$ and both
$\bar q_\theta$ and $\partial_r\bar q_\theta$ are uniformly
bounded in $L^2(p_t)$ over $0\leq r\leq t\leq1$.
Under the finite-second-moment assumption above, these conditions hold for finite affine-SiLU networks with direct spatial and time inputs; the resulting constants may depend on the network parameters. See App. \ref{app:rmkGthetareg} for details and proof.
\end{remark}

Fix a generation schedule $1=t_K>t_{K-1}>\cdots>t_0=0$, put $h_k=t_k-t_{k-1}$.
\begin{assumption}
\label{assm:timeavg}
\begin{equation}
\sum_{k=1}^K\int_{t_{k-1}}^{t_k}(t-t_{k-1})^2
\big\|(\partial_wu_\theta)(\cdot,t_{k-1},t)
(v-v_\theta)(\cdot,t)\big\|_{L^2(p_t)}^2\,dt
\leq B_u^2\mathcal R_{\mathrm{FM}}(\theta).\label{eq:jacobian}
\end{equation}
where $\RFM$ is the marginal flow-matching risk \eqref{eq:RFM} of the auxiliary velocity head. 
\end{assumption}

\begin{remark}
\label{rmk:averagebound}
A sufficient condition for \eqref{eq:jacobian} is $\sup_{w,\,0\leq r<t\leq1}
(t-r)\|\partial_wu_\theta(w,r,t)\|_{\mathrm{op}}
\leq B_u$. In particular, if $u_\theta$ is $L_u$-Lipschitz in $w$, uniformly over $(r,t)$ (e.g., affine–SiLU networks), then this condition holds with $B_u=L_u$, since $t-r\leq1$.
The weighted bound in Assumption~\ref{assm:timeavg} is weaker:
it only controls the Jacobian acting on $v-v_\theta$ in the $L^2(p_t)$ sense. See Appendix \ref{proof:rmkaveragebound} for more details.
\end{remark}

\begin{theorem}
\label{thm:fsimf}
Let $F_{k,\theta}(w)=w-h_ku_\theta(w,t_{k-1},t_k)$ be $L_k$-Lipschitz,
and let $S_\theta=F_{1,\theta}\circ\cdots\circ F_{K,\theta}$ be the few-step iMF map \eqref{eq:fsimf}. Define $\Lambda_k:=\prod_{j<k}L_j, A_K:=\sum_{k=1}^K\Lambda_k\sqrt{h_k}, B_K:=\left(\sum_{k=1}^K\Lambda_k^2h_k\right)^{1/2}, \Lambda_1=1$. Under Assumptions \ref{assm:prior-flow}, \ref{assm:u-regularity}, \ref{assm:LG} and \ref{assm:timeavg}, if $\delta\leq L_G/2$,
\begin{equation}
\mathbb{E}_{\gamma}\|S_\theta(Z)-S(Z)\|^2 \leq \left[ A_K(2L_G \delta)^{1/4}
+\frac{B_uB_K}{\sqrt\lambda}\delta^{1/2} \right]^2.\label{eq:few-step}
\end{equation}
\end{theorem}

If $u_\theta$ is uniformly $L_u$-Lipschitz in space, then $\Lambda_k\leq e^{L_u}$, $A_K\leq e^{L_u}\sqrt K$ and
$B_K\leq e^{L_u}$. Proof of Theorem \ref{thm:fsimf} is given in Appendix \ref{proof:thm:fsimf}. The bound in \eqref{eq:few-step} links transport error to
population training suboptimality and model-class approximation
error through $\delta$ in \eqref{eq:nothing}. The bound vanishes when $\delta=0$,
\emph{without} requiring the step sizes $h_k$ to tend to zero. This shows the few-step nature of the result. In particular, $K=1$ gives $A_1=B_1=1$ and a transport guarantee for the one-step
map~\eqref{eq:imf1}, with no Lipschitz condition on $S_\theta$ needed in \eqref{eq:few-step}. We support Theorem \ref{thm:fsimf} with an experiment in Appendix~\ref{app:loss_map_error}, showing that decreasing $\delta$ is accompanied by decreasing transport-map error $\mathbb{E}_{\gamma}\|S_\theta(Z)-S(Z)\|^2$ for both one- and six-step generation. Define terms in Theorem \ref{thm:fsimf}:
\begin{equation}
\mathcal E:=\mathbb E_{Z\sim\gamma}\|S_\theta(Z)-S(Z)\|^2,
\qquad
b_K(\delta):=A_K(2L_G\delta)^{1/4}
+\frac{B_uB_K}{\sqrt\lambda}\delta^{1/2}.
\label{eq:transport-error-bound}
\end{equation}

We now connect the posterior error to the transport error. Fix $y$, we make the following assumption.

\begin{assumption}[Weight regularity]
\label{assm:weight}
There exist finite constants $M_y>0$ and $L_y\geq0$ such that
\[
0\leq\ell_y(x)\leq M_y,\qquad
|\ell_y(x)-\ell_y(x')|\leq L_y\|x-x'\|,
\qquad \forall x,x'\in\mathbb R^d.
\]
Moreover, $Z_y:=\int\ell_y(x)\,\pi_0(dx)>0$.
\end{assumption}

With the notation of Section~\ref{sec:source-posterior},
the normalized source posteriors are
\[
\nu^y(dz)=\frac{\ell_y(S(z))}{Z_y}\gamma(dz),\qquad
\nu_\theta^y(dz)=
\frac{\ell_y(S_\theta(z))}{Z_{\theta,y}}\gamma(dz),
\quad
Z_{\theta,y}:=\int\ell_y(S_\theta(z))\,\gamma(dz),
\]
where $\nu_\theta^y$ is defined whenever $Z_{\theta,y}>0$.
Here $S_\#\nu^y=\pi^y$. When $\ell_y>0$, write
$\Phi_y=-\log\ell_y$.

\begin{theorem}\label{prop:posterior-stability}
Under Assumption~\ref{assm:weight}, if
$L_y\mathcal E^{1/2}\leq Z_y/2$, then
$Z_{\theta,y}\geq Z_y/2>0$ and
\begin{equation}
W_2(\nu_\theta^y,\nu^y)\leq C_y(d)^{1/2}\mathcal E^{1/4},\qquad
C_y(d):=\frac{4L_y}{Z_y}
\left[\sqrt{d^2+4d+1}+\frac{M_y(1+d)}{Z_y}\right].
\label{eq:source-stability}
\end{equation}
If, additionally, $S_\theta$ is $L_{S_\theta}$-Lipschitz, then $W_2\big((S_\theta)_\#\nu_\theta^y,\pi^y\big)
\leq L_{S_\theta}C_y(d)^{1/2}\mathcal E^{1/4}
+\sqrt{\frac{M_y}{Z_y}}\mathcal E^{1/2}$.
\end{theorem}
Proof is in Appendix~\ref{proof:posterior-stability}. Combining these results, we derive the following.

\begin{corollary}
\label{cor:end-to-end}
Suppose the hypotheses of Theorem~\ref{thm:fsimf} and
Assumption~\ref{assm:weight} hold. Let $L_{S_\theta}$ be
a Lipschitz constant of $S_\theta$, and let $C_y(d)$ be
as in Theorem~\ref{prop:posterior-stability}.
If $\delta\leq L_G/2, L_yb_K(\delta)\leq Z_y/2$, then $Z_{\theta,y}\geq Z_y/2>0$ and
\begin{align}
W_2(\nu_\theta^y,\nu^y)
&\leq C_y(d)^{1/2}b_K(\delta)^{1/2},
\label{eq:end-to-end-source}\\
W_2\big((S_\theta)_\#\nu_\theta^y,\pi^y\big)
&\leq L_{S_\theta}C_y(d)^{1/2}b_K(\delta)^{1/2}
+\sqrt{\frac{M_y}{Z_y}}b_K(\delta).
\label{eq:end-to-end-data}
\end{align}

\end{corollary}

\subsection{Source parallel tempering with improved MeanFlow}
\label{sec:sampler}

Given the learned iMF, we aim to sample from $\nu_\theta^y$. Although the source reference
is Gaussian, the pulled-back likelihood can induce multimodality,
and the source retains the data dimension. We use source parallel tempering \citep{wangtest} to improve
mixing between modes, together with local kernels like pCN and Gaussian-reference split HMC
to address high dimensionality. Both admit formulations directly
on infinite-dimensional Hilbert spaces, motivating their use
for high-dimensional Gaussian-reference targets
\citep{beskos2008mcmc,cotter2013mcmc,beskos2011hybrid}.

Parallel tempering \citep{geyer1991markov} runs replicas
at inverse temperatures $0=\beta_0<\cdots<\beta_J=1$, targeting
$\nu_{\theta,\beta}^y(\dd z)
\propto \ell_y(S_\theta(z))^\beta\gamma(\dd z)$.
Local MCMC kernels explore each target, while Metropolis swaps
between adjacent replicas promote transitions between modes. Each sampling sweep locally updates every replica once, then attempts one Metropolis swap per adjacent temperature pair. The coldest replica  ($\beta=1$) targets  $\nu_\theta^y$. The two variants below differ in their local MCMC kernels.

\paragraph{iMF+SPT+pCN.}
We use pCN at every replica as in \citep{wangtest}. Writing $\Phi_y=-\log\ell_y$ and
$\Psi_\theta=\Phi_y\circ S_\theta$, the proposal is
$z'=\cos(h_\beta)z+\sin(h_\beta)\xi$,
where $\xi\sim\gamma$ and $h_\beta\in(0,\pi/2]$,
with acceptance probability
$1\wedge\exp\{-\beta[\Psi_\theta(z')-\Psi_\theta(z)]\}$.
At $\beta=0$, every proposal is accepted; and choosing
$h_0=\pi/2$ gives an independent Gaussian draw $z'=\xi\sim\gamma$.

\paragraph{iMF+SPT+hybrid.}
When $\Psi_\theta$ is differentiable, we replace the \emph{coldest}
pCN update ($\beta=1$) with split HMC
\citep{beskos2011hybrid} to exploit local posterior geometry,
while retaining pCN at hotter replicas to limit computational
cost.  Refresh $p\sim\mathcal \gamma$ and define the Hamiltonian $H(z,p)=\tfrac12\|z\|^2+\tfrac12\|p\|^2+\Psi_\theta(z)$.
Starting from $(z_0,p_0)=(z,p)$, apply $L$ steps of Strang's splitting:

\begin{equation}
\begin{aligned}
p_{\ell+1/2}
&=p_\ell-\tfrac{h}{2}\nabla\Psi_\theta(z_\ell),
\quad \rhd~\text{half momentum kick}\\
\begin{pmatrix}
z_{\ell+1}\\
\widetilde p_{\ell+1}
\end{pmatrix}
&=
\begin{pmatrix}
\cos(h)I & \sin(h)I\\
-\sin(h)I & \cos(h)I
\end{pmatrix}
\begin{pmatrix}
z_\ell\\
p_{\ell+1/2}
\end{pmatrix},
\quad \rhd~\text{exact harmonic rotation}\\
p_{\ell+1}
&=\widetilde p_{\ell+1}
-\tfrac{h}{2}\nabla\Psi_\theta(z_{\ell+1}),\quad \rhd~\text{half momentum kick}
\end{aligned}
\label{eq:split-proposal}
\end{equation}
Accept $(z',p')=(z_L,-p_L)$ with probability
$1\wedge\exp\{H(z,p)-H(z',p')\}$; otherwise retain $z$.

\paragraph{Sampler geometric convergence.}
\citet{wangtest} establish geometric $W_2$ convergence to the learned posterior for SPT with all pCN updates. We extend this guarantee to SPT+hybrid. Let $\nu_n=\operatorname{Law}(z_n)$, where $z_n$ is the coldest replica's state after $n$ sampling sweeps. Under the conditions in Appendix~\ref{app:sampler-convergence}, including
a source weight bounded above and away from zero, the cold marginal
of either sampler satisfies $W_2(\nu_n,\nu_\theta^y)\leq C\rho^n$
for some $C<\infty$ and $\rho\in(0,1)$. Together with posterior stability (Theorem \ref{prop:posterior-stability}, Corollary \ref{cor:end-to-end}), this yields an end-to-end error
bound with a geometrically decaying sampling term and a transport
term that vanishes with $\delta$; see
Corollary~\ref{cor:end-to-end-convergence}.

\section{Experiments}
\label{sec:experiments}
We evaluate the proposed samplers on synthetic inference problems and CLIP-guided ImageNet generation. Experiments were conducted on LUMI\footnote{\url{https://lumi-supercomputer.eu/}} using one Slurm-allocated GPU device per run, with AMD Instinct MI250X hardware, eight CPU cores, and 64 GB of host memory. Our code will be made available soon.

\subsection{Synthetic experiments}

\paragraph{Problem setting.}
We consider two priors on $x\in\mathbb{R}^{32}$, each formed
from 16 independent two-dimensional blocks. For
$u_j,v_j\overset{\mathrm{iid}}{\sim}\mathcal{N}(0,1)$, the priors are
\[
x_{2j-1}=u_j,\qquad
x_{2j}=
\begin{cases}
0.62(u_j^2-1)+0.30v_j, & \text{Banana},\\
1.05\sin(1.65u_j)+0.15v_j, & \text{Sine},
\end{cases}
\quad j=1,\ldots,16.
\]
Both problems use linear observations with independent Gaussian
noise, giving the likelihood
\[
p(y\mid x)=\prod_{j=1}^{16}
\mathcal{N}\!\left(y_j;\,x_{2j}+0.35x_{2j-1},\,0.2^2\right).
\]
We condition on the fixed observations $y_1=y_2=2$ and
$y_j=-1$ for $j=3,\ldots,16$.

\paragraph{Prior learning and sampling.}
We evaluate iMF+SPT+pCN and iMF+SPT+hybrid against two baselines: FM+SPT+pCN and diffusion posterior sampling (DPS) \citep{chung2022diffusion}. The two iMF methods share a six-step transport. The FM baseline uses a 100-step RK4 transport. For each problem, the iMF, FM, and diffusion prior models are trained on the same fixed dataset of 20,000 prior samples. Each method returns 6,000 samples. Figure~\ref{fig:prior} compares samples from the learned priors with those from the true banana and sine priors. See Appendix \ref{app:synthetic} for more details.

\paragraph{Results.}
The blockwise structure of the problems enables accurate numerical posterior references for assessing approximation quality. Table~\ref{tab:synthetic_results} summarizes the results, and Figures~\ref{fig:post_banana} and~\ref{fig:post_sine} compare posterior sample histograms from each method with numerical references for the true banana and sine. Using the same learned iMF transport, the hybrid sampler increases mean coordinate ESS (effective sample size) by 3.24$\times$ on banana and 4.44$\times$ on sine, and coordinate ESS per second by 2.25$\times$ and 3.13$\times$, respectively, compared with iMF+SPT+pCN. It also achieves the lowest mean sliced Wasserstein and binned TV errors, although the differences in sliced Wasserstein error are small relative to the variability across seeds. With its configured 100-step RK4 transport, FM+SPT+pCN requires 33.7$\times$ and 34.5$\times$ the sampling time of the iMF hybrid on banana and sine, respectively. DPS exhibits substantially larger distributional errors, with systematic deviations from the reference posterior visible in the histograms.

\begin{table}[t]
    \centering
    \caption{
        Synthetic posterior sampling. See Appendix~\ref{app:synthetic} for more details.
    }
    \label{tab:synthetic_results}
    \small
    \setlength{\tabcolsep}{3pt}
    \begin{tabular}{lrrrrr}
        \toprule
        Method
        & $\mathrm{SW}_2\downarrow$
        & Bin TV $\downarrow$
        & Coord. ESS $\uparrow$
        & Coord. ESS/s $\uparrow$
        & Time (s) $\downarrow$ \\
        \midrule
        \multicolumn{6}{l}{\textit{Banana}} \\
        DPS
        & $0.462\pm0.071$ & $0.311\pm0.002$ & --- & --- & $\mathbf{38.4\pm1.0}$ \\
        FM+SPT+pCN
        & $0.219\pm0.081$ & $0.106\pm0.008$ & $372\pm16$
        & $0.052\pm0.002$ & $7087\pm45$ \\
        iMF+SPT+pCN
        & $0.154\pm0.012$ & $0.109\pm0.006$ & $346\pm4$
        & $2.37\pm0.12$ & $146.1\pm6.0$ \\
        iMF+SPT+hybrid
        & $\mathbf{0.142\pm0.031}$ & $\mathbf{0.085\pm0.005}$
        & $\mathbf{1122\pm48}$
        & $\mathbf{5.33\pm0.08}$ & $210.4\pm5.8$ \\
        \midrule
        \multicolumn{6}{l}{\textit{Sine}} \\
        DPS
        & $0.931\pm0.025$ & $0.286\pm0.001$ & --- & --- & $\mathbf{38.9\pm0.2}$ \\
        FM+SPT+pCN
        & $0.104\pm0.021$ & $0.129\pm0.014$ & $395\pm16$
        & $0.0526\pm0.0023$ & $7505\pm78$ \\
        iMF+SPT+pCN
        & $0.093\pm0.016$ & $0.131\pm0.010$ & $378\pm35$
        & $2.47\pm0.32$ & $153.3\pm5.8$ \\
        iMF+SPT+hybrid
        & $\mathbf{0.088\pm0.011}$ & $\mathbf{0.127\pm0.004}$
        & $\mathbf{1679\pm47}$
        & $\mathbf{7.73\pm0.31}$ & $217.3\pm5.4$ \\
        \bottomrule
    \end{tabular}
\end{table}
\subsection{CLIP-guided ImageNet generation}
\label{sec:clip_guidance_experiment}
We evaluate CLIP-guided $256\times256$ class-conditional ImageNet generation: for a fixed ImageNet class (e.g., bald eagle) and text prompt, the goal is to sample images from that class that favor high image-text alignment. We use the cosine similarity from a frozen pretrained CLIP model \citep{radford2021learning} as the guidance reward. Let $z \sim \mathcal{N}(0,I)$ and let $G_c(z)$ denote a frozen pretrained transport from the Gaussian source distribution to images in class $c$. For a text prompt $p$, the learned source posterior is $\nu^p_{\theta}(z) \propto \exp\!\left(-\frac{1}{2}\|z\|^2 + \tau \clip(G_c(z),p)\right)$ with $\tau=200$.

\paragraph{Comparison}
We use the pretrained iMF for class-conditional ImageNet generation \citep{geng2026improved}. We compare iMF+SPT+pCN and iMF+SPT+hybrid against their multi-step counterpart SiT+SPT+pCN, where SiT \citep{ma2024sit} is a pretrained class-conditional flow-matching prior integrated using 125 Heun steps. These methods use eight temperature replicas and eight independent ladders. For iMF-based methods, we run 250 adaptation sweeps, 250 burn-in sweeps, and retain 1,000 post-burn-in samples per ladder, yielding 8,000 retained samples per method. SiT+SPT+pCN is extremely slow in this experiment, we evaluate it under a matched per-prompt sampling-time budget equal to the measured hybrid sampling time (more than $10$ hours), with an additional one-hour allowance for initialization, scoring, and image export. We also include best-of-$K$ \citep{karthik2023if}, $K=12$, as a reward-based selection
baseline, retaining the image with the highest CLIP reward
among $K$ learned-prior independent samples. We produce $8,000$ best-of-K images.

\paragraph{Results.} Table~\ref{tab:clip_guidance_selected} reports CLIP rewards and sampling costs for two class-prompt pairs, while Figures~\ref{fig:goldenautum} and~\ref{fig:eaglemoun} show samples ($R$ is the reward) from two ladders per method. JPEG compression is for display only and does not affect reported rewards. The two iMF methods achieve comparable rewards and both outperform SiT+SPT+pCN. Their samples more faithfully depict a \emph{running} golden retriever amid \emph{fallen autumn leaves}, whereas SiT+SPT+pCN often misses the autumn setting within its sampling budget of 10-11 hours. For the bald eagle prompt, all three methods depict eagles with \emph{outstretched wings}, but the iMF methods more consistently include \emph{snow-capped mountains} whereas SiT+SPT+pCN depicts the eagle against the plain blue sky. In these examples, the hybrid sampler offers no clear improvement in reward or sample quality over iMF+SPT+pCN. Finally, best-of-$K$ achieves the lowest rewards, suggesting guidance is more effective than independent prior draws. The full results are given in Appendix \ref{fullimagenet}.

\section{Conclusion}
We develop few-step iMF source-space inference, linking population
training suboptimality and approximation error to posterior stability
and end-to-end error bounds under stated assumptions. Hybrid SPT
improves efficiency on synthetic benchmarks, while one-step iMF
reduces CLIP-guided ImageNet sampling costs relative to multi-step
SiT. Quantifying how model-class approximation error $\epsapp$ decreases with
capacity while controlling the regularity constants remains open.

\begin{table*}[t]
\centering
\small
\caption{CLIP reward and measured sampling cost.}
\label{tab:clip_guidance_selected}
\begin{tabular}{lccc}
\toprule
Method & Mean CLIP reward $\uparrow$ & Sampling time (h) & Mean sweep time (s) \\
\midrule
\multicolumn{4}{l}{\textbf{Golden retriever:} \textit{``A photograph of a golden retriever running through fallen autumn leaves."}} \\
Best-of-K & 0.2864 & 2.49 & -- \\
SiT+SPT+pCN & 0.3122 & 10.51 & 370.89 \\
iMF+SPT+pCN & 0.3809 & 2.40 & 5.75 \\
iMF+SPT+hybrid & \textbf{0.3824} & 10.95 & 26.28 \\
\midrule
\multicolumn{4}{l}{\textbf{Bald eagle:} \textit{``A photograph of a bald eagle soaring above snow-capped mountains, wings spread wide."}} \\
Best-of-K & 0.3278 & 2.49 & -- \\*
SiT+SPT+pCN   & 0.3481 & 10.70 & 370.27 \\
iMF+SPT+pCN    & \textbf{0.3936} & 2.44  & 5.85   \\
iMF+SPT+hybrid & 0.3914 & 11.15 & 26.76  \\
\bottomrule
\end{tabular}
\end{table*}

\begin{figure}[H]
\centering

% Left: golden retriever
\begin{minipage}[t]{0.47\linewidth}
    \vspace{0pt}
    \captionsetup{type=figure}
    \centering

    \begin{subfigure}{\linewidth}
        \centering
        \includegraphics[width=\linewidth]{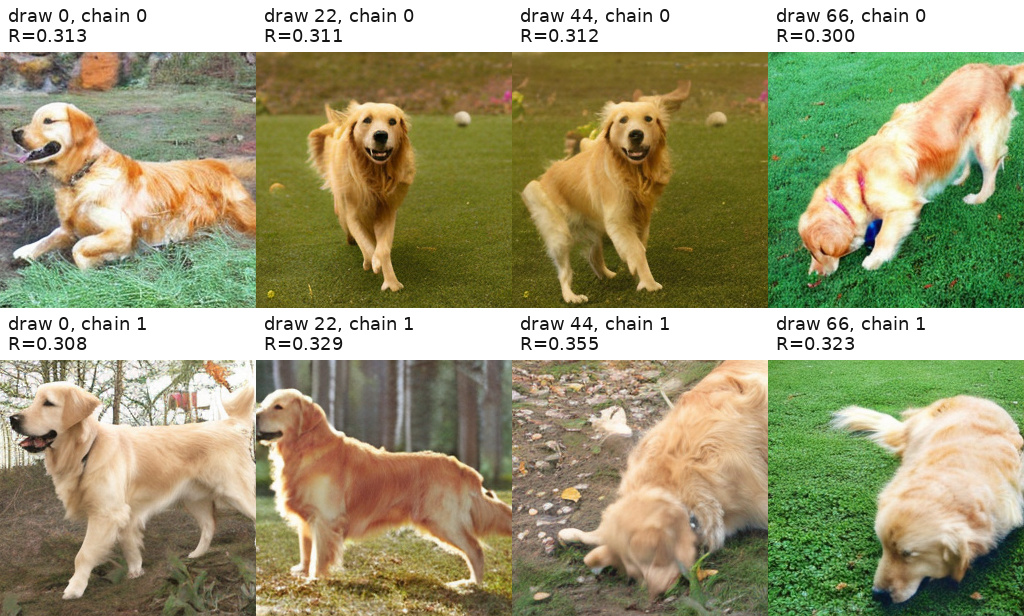}
        \caption{SiT + SPT + pCN.}
        \label{fig:third}
    \end{subfigure}

    \par\medskip

    \begin{subfigure}{\linewidth}
        \centering
        \includegraphics[width=\linewidth]{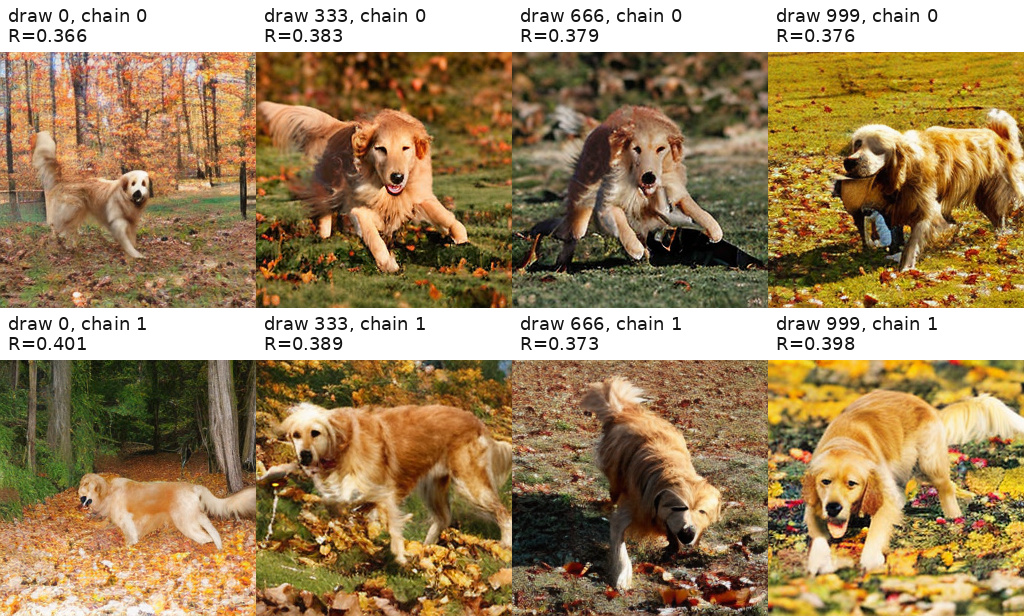}
        \caption{iMF + SPT + pCN.}
        \label{fig:second}
    \end{subfigure}

    \par\medskip

    \begin{subfigure}{\linewidth}
        \centering
        \includegraphics[width=\linewidth]{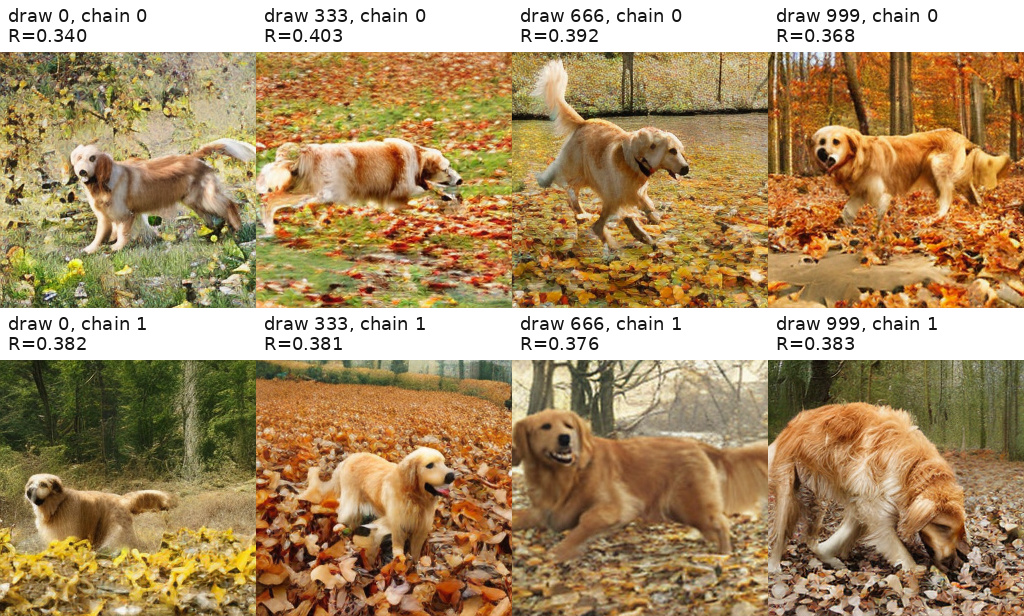}
        \caption{iMF + SPT + hybrid.}
        \label{fig:first}
    \end{subfigure}

    \caption{Class: golden retriever, prompt: \textit{``A photograph of a golden retriever running through fallen autumn leaves.''}}
    \label{fig:goldenautum}
\end{minipage}%
\hfill
% Right: bald eagle
\begin{minipage}[t]{0.47\linewidth}
    \vspace{0pt}
    \captionsetup{type=figure}
    \centering

    \begin{subfigure}{\linewidth}
        \centering
        \includegraphics[width=\linewidth]{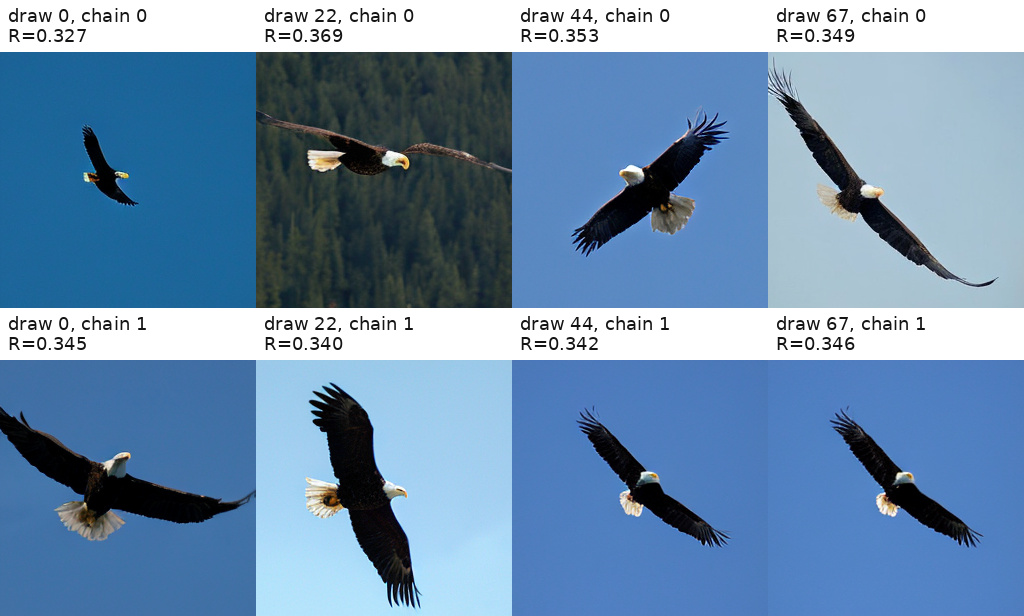}
        \caption{SiT + SPT + pCN.}
        \label{fig:eagle-third}
    \end{subfigure}

    \par\medskip

    \begin{subfigure}{\linewidth}
        \centering
        \includegraphics[width=\linewidth]{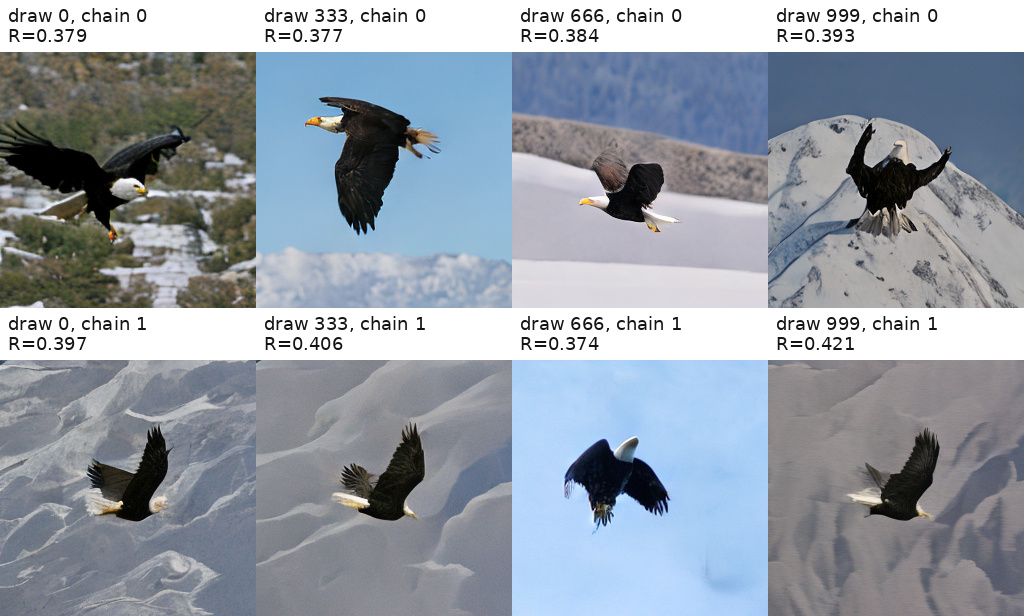}
        \caption{iMF + SPT + pCN.}
        \label{fig:eagle-second}
    \end{subfigure}

    \par\medskip

    \begin{subfigure}{\linewidth}
        \centering
        \includegraphics[width=\linewidth]{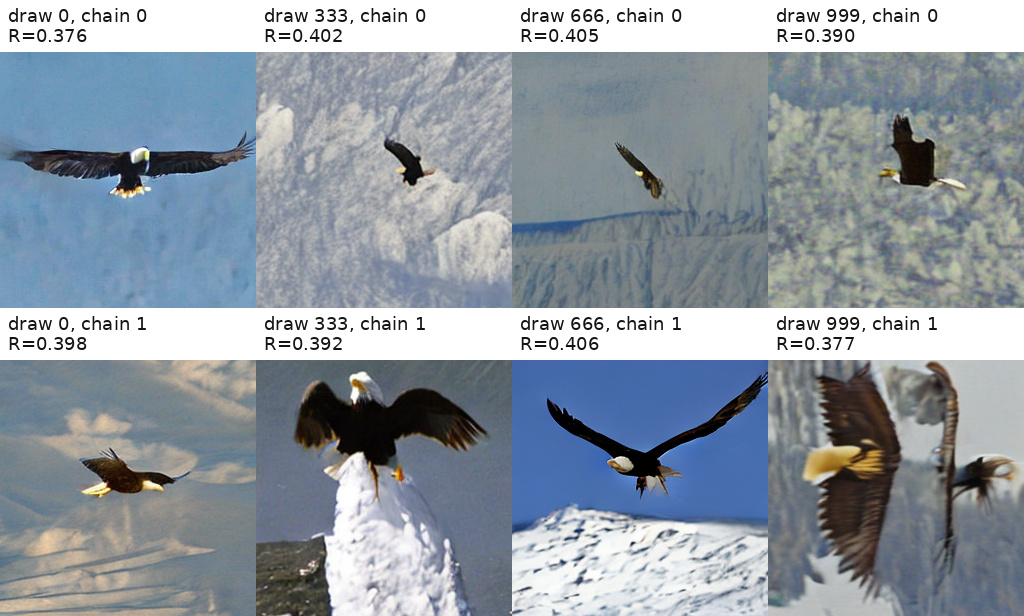}
        \caption{iMF + SPT + hybrid.}
        \label{fig:eagle-first}
    \end{subfigure}

    \caption{Class: bald eagle, prompt: \textit{``A photograph of a bald eagle soaring above snow-capped mountains, wings spread wide.''}}
    \label{fig:eaglemoun}
\end{minipage}

\end{figure}

\newpage 

\subsection*{AI use statement}

%(This section is \textbf{required} and does not count toward the page limit.)

We used generative AI tools (GPT6-Astra, Claude Opus 5, Codex, Copilot Pro+) to assist with language editing; literature search and paper summarization; algorithm implementation, comparative experiments, and code testing; preliminary (unusable) drafts and templates; and developing and checking mathematical arguments; and reviewing the manuscript. We did not use generative AI to generate synthetic datasets. The authors reviewed the AI-assisted material, including text, code, and mathematical arguments. We take full responsibility for the final manuscript, its claims, and all accompanying artifacts, including those produced with generative AI assistance.

\subsection*{Acknowledgement} This work is supported by NTU-SUG and Singapore Ministry of Education (MOE) AcRF Tier 1 RG17/24 and the Research Council of Finland grant number 369502. We acknowledge computational resources on LUMI supercomputer, owned by the EuroHPC Joint Undertaking and hosted by CSC (Finland) and the LUMI consortium.

\bibliography{ref}

@article{bissiri2016general,
  title={A general framework for updating belief distributions},
  author={Bissiri, Pier Giovanni and Holmes, Chris C and Walker, Stephen G},
  journal={Journal of the Royal Statistical Society Series B: Statistical Methodology},
  volume={78},
  number={5},
  pages={1103--1130},
  year={2016},
  publisher={Oxford University Press}
}

@article{beskos2011hybrid,
  title={Hybrid {M}onte {C}arlo on {Hi}lbert spaces},
  author={Beskos, Alexandros and Pinski, Frank J and Sanz-Serna, Jes{\'u}s Mar{\i}a and Stuart, Andrew M},
  journal={Stochastic Processes and their Applications},
  volume={121},
  number={10},
  pages={2201--2230},
  year={2011},
  publisher={Elsevier}
}

@inproceedings{chung2022diffusion,
  title={Diffusion Posterior Sampling for General Noisy Inverse Problems},
  author={Chung, Hyungjin and Kim, Jeongsol and Mccann, Michael Thompson and Klasky, Marc Louis and Ye, Jong Chul},
  booktitle={The Eleventh International Conference on Learning Representations},
  year={2023},
}

@article{cotter2013mcmc,
  title={{MCMC} methods for functions: modifying old algorithms to make them faster},
  author={Cotter, Simon L and Roberts, Gareth O and Stuart, Andrew M and White, David},
  journal={Statistical Science},
  pages={424--446},
  year={2013},
  publisher={JSTOR}
}

@article{geng2026mean,
  title={Mean flows for one-step generative modeling},
  author={Geng, Zhengyang and Deng, Mingyang and Bai, Xingjian and Kolter, Zico and He, Kaiming},
  journal={Advances in Neural Information Processing Systems},
  volume={38},
  pages={75460--75482},
  year={2026}
}

@inproceedings{geng2026improved,
  title={Improved mean flows: On the challenges of fastforward generative models},
  author={Geng, Zhengyang and Lu, Yiyang and Wu, Zongze and Shechtman, Eli and Kolter, J Zico and He, Kaiming},
  booktitle={Proceedings of the IEEE/CVF Conference on Computer Vision and Pattern Recognition},
  pages={30467--30476},
  year={2026}
}

@article{hoffman2019neutra,
  title={Neutra-lizing bad geometry in hamiltonian monte carlo using neural transport},
  author={Hoffman, Matthew and Sountsov, Pavel and Dillon, Joshua V and Langmore, Ian and Tran, Dustin and Vasudevan, Srinivas},
  journal={arXiv preprint arXiv:1903.03704},
  year={2019}
}

@inproceedings{lipman2023flow,
  title={Flow Matching for Generative Modeling},
  author={Lipman, Yaron and Chen, Ricky T. Q. and Ben-Hamu, Heli and Nickel, Maximilian and Le, Matt},
  booktitle={The Eleventh International Conference on Learning Representations},
  year={2023}
}

@article{parno2018transport,
  title={Transport map accelerated markov chain monte carlo},
  author={Parno, Matthew D and Marzouk, Youssef M},
  journal={SIAM/ASA Journal on Uncertainty Quantification},
  volume={6},
  number={2},
  pages={645--682},
  year={2018},
  publisher={SIAM}
}

@article{ramzi2020denoising,
  title={Denoising score-matching for uncertainty quantification in inverse problems},
  author={Ramzi, Zaccharie and Remy, Benjamin and Lanusse, Francois and Starck, Jean-Luc and Ciuciu, Philippe},
  journal={arXiv preprint arXiv:2011.08698},
  year={2020}
}

@inproceedings{radford2021learning,
  title={Learning transferable visual models from natural language supervision},
  author={Radford, Alec and Kim, Jong Wook and Hallacy, Chris and Ramesh, Aditya and Goh, Gabriel and Agarwal, Sandhini and Sastry, Girish and Askell, Amanda and Mishkin, Pamela and Clark, Jack and others},
  booktitle={International conference on machine learning},
  pages={8748--8763},
  year={2021},
  organization={PmLR}
}

@inproceedings{wang2026source,
  title={Source-guided flow matching},
  author={Wang, Zifan and Harting, Alice and Barreau, Matthieu and Zavlanos, Michael and Johansson, Karl H},
  booktitle={International Conference on Learning Representations},
  volume={2026},
  pages={106484--106514},
  year={2026}
}

@inproceedings{wangtest,
  title={Test-Time Guidance for Flow-Based Generative Models via Parallel Tempering on Source Distributions},
  author={Wang, Shih-Hsin and Keller, Joel and Transue, Taos and Brown, Drake Benjamin and Strohmer, Thomas and Wang, Bao},
  booktitle={International Conference on Machine Learning},
  year={2026}
}

@article{patel2022solution,
  title={Solution of physics-based Bayesian inverse problems with deep generative priors},
  author={Patel, Dhruv V and Ray, Deep and Oberai, Assad A},
  journal={Computer Methods in Applied Mechanics and Engineering},
  volume={400},
  pages={115428},
  year={2022},
  publisher={Elsevier}
}

@article{vincent2011connection,
  title={A connection between score matching and denoising autoencoders},
  author={Vincent, Pascal},
  journal={Neural computation},
  volume={23},
  number={7},
  pages={1661--1674},
  year={2011},
  publisher={MIT Press}
}

@article{tong2024improving,
  title   = {Improving and Generalizing Flow-Based Generative Models
             with Minibatch Optimal Transport},
  author={Tong, Alexander and Fatras, Kilian and Malkin, Nikolay and Huguet, Guillaume and Zhang, Yanlei and Rector-Brooks, Jarrid and Wolf, Guy and Bengio, Yoshua},
  journal = {Transactions on Machine Learning Research},
  year    = {2024}
}

@article{beskos2008mcmc,
  title={{MCMC} methods for diffusion bridges},
  author={Beskos, Alexandros and Roberts, Gareth and Stuart, Andrew and Voss, Jochen},
  journal={Stochastics and Dynamics},
  volume={8},
  number={03},
  pages={319--350},
  year={2008},
  publisher={World Scientific}
}

@article{earl2005parallel,
  title={Parallel tempering: Theory, applications, and new perspectives},
  author={Earl, David J and Deem, Michael W},
  journal={Physical Chemistry Chemical Physics},
  volume={7},
  number={23},
  pages={3910--3916},
  year={2005},
  publisher={The Royal Society of Chemistry}
}

@inproceedings{greenberg2019automatic,
  title={Automatic posterior transformation for likelihood-free inference},
  author={Greenberg, David and Nonnenmacher, Marcel and Macke, Jakob},
  booktitle={International conference on machine learning},
  pages={2404--2414},
  year={2019},
  organization={PMLR}
}

@article{radev2023bayesflow,
  title={BayesFlow: Amortized Bayesian workflows with neural networks},
  author={Radev, Stefan T and Schmitt, Marvin and Schumacher, Lukas and Elsem{\"u}ller, Lasse and Pratz, Valentin and Sch{\"a}lte, Yannik and K{\"o}the, Ullrich and B{\"u}rkner, Paul-Christian},
  journal={arXiv preprint arXiv:2306.16015},
  year={2023}
}

@article{Benjamin,
title = {Transport map unadjusted {L}angevin algorithms: Learning and discretizing perturbed samplers},
journal = {Foundations of Data Science},
volume = {7},
number = {3},
pages = {705-736},
year = {2025},
author = {Benjamin J. Zhang and Youssef M. Marzouk and Konstantinos Spiliopoulos},
}

@article{cheng2026preconditioned,
  title={Preconditioned One-Step Generative Modeling for {B}ayesian Inverse Problems in Function Spaces},
  author={Cheng, Zilan and Wang, Li-Lian and Wang, Zhongjian},
  journal={arXiv preprint arXiv:2603.14798},
  year={2026}
}

@article{stuart2010inverse,
  title={Inverse problems: a {B}ayesian perspective},
  author={Stuart, Andrew M},
  journal={Acta numerica},
  volume={19},
  pages={451--559},
  year={2010},
  publisher={Cambridge University Press}
}

@article{patel2019bayesian,
  title={Bayesian inference with generative adversarial network priors},
  author={Patel, Dhruv and Oberai, Assad A},
  journal={arXiv preprint arXiv:1907.09987},
  year={2019}
}

@article{lunz2018adversarial,
  title={Adversarial regularizers in inverse problems},
  author={Lunz, Sebastian and {\"O}ktem, Ozan and Sch{\"o}nlieb, Carola-Bibiane},
  journal={Advances in neural information processing systems},
  volume={31},
  year={2018}
}

@article{hoffman2014no,
  title={The {No-U-Turn} sampler: adaptively setting path lengths in {H}amiltonian {M}onte {C}arlo},
  author={Hoffman, Matthew D and Gelman, Andrew and others},
  journal={Journal of Machine Learning Research},
  volume={15},
  number={1},
  pages={1593--1623},
  year={2014}
}

@article{chen2018neural,
  title={Neural ordinary differential equations},
  author={Chen, Ricky TQ and Rubanova, Yulia and Bettencourt, Jesse and Duvenaud, David K},
  journal={Advances in neural information processing systems},
  volume={31},
  year={2018}
}

@article{papamakarios2021normalizing,
  title={Normalizing flows for probabilistic modeling and inference},
  author={Papamakarios, George and Nalisnick, Eric and Rezende, Danilo Jimenez and Mohamed, Shakir and Lakshminarayanan, Balaji},
  journal={Journal of Machine Learning Research},
  volume={22},
  number={57},
  pages={1--64},
  year={2021}
}

@article{ben2024d,
  title={D-flow: Differentiating through flows for controlled generation},
  author={Ben-Hamu, Heli and Puny, Omri and Gat, Itai and Karrer, Brian and Singer, Uriel and Lipman, Yaron},
  journal={arXiv preprint arXiv:2402.14017},
  year={2024}
}

@book{hairer1993solving,
  title={Solving ordinary differential equations I: Nonstiff problems},
  author={Hairer, Ernst and Wanner, Gerhard and N{\o}rsett, Syvert P},
  year={1993},
  publisher={Springer}
}

@inproceedings{ma2024sit,
  title={Sit: Exploring flow and diffusion-based generative models with scalable interpolant transformers},
  author={Ma, Nanye and Goldstein, Mark and Albergo, Michael S and Boffi, Nicholas M and Vanden-Eijnden, Eric and Xie, Saining},
  booktitle={European Conference on Computer Vision},
  pages={23--40},
  year={2024},
  organization={Springer}
}

@InProceedings{purohit2024posterior,
    author    = {Purohit, Vishal and Repasky, Matthew and Lu, Jianfeng and Qiu, Qiang and Xie, Yao and Cheng, Xiuyuan},
    title     = {Consistency Posterior Sampling for Diverse Image Synthesis},
    booktitle = {Proceedings of the IEEE/CVF Conference on Computer Vision and Pattern Recognition (CVPR)},
    month     = {June},
    year      = {2025},
    pages     = {28327-28336}
}

@article{lu2026variance,
  title={On Variance Reduction in Learning Mean Flows},
  author={Lu, Juanwu and Wang, Ziran},
  journal={arXiv preprint arXiv:2605.09235},
  year={2026}
}

@inproceedings{korbak2022rl,
  title={{RL} with {KL} penalties is better viewed as {B}ayesian inference},
  author={Korbak, Tomasz and Perez, Ethan and Buckley, Christopher},
  booktitle={Findings of the Association for Computational Linguistics: EMNLP 2022},
  pages={1083--1091},
  year={2022}
}

@book{villani2009optimal,
  title={Optimal transport: old and new},
  author={Villani, C{\'e}dric and others},
  volume={338},
  year={2009},
  publisher={Springer}
}

@article{geyer1991markov,
  title={Markov chain {M}onte {C}arlo maximum likelihood},
  author={Geyer, Charles J},
  year={1991},
  publisher={Interface Foundation of North America}
}

@article{otto2000generalization,
  title={Generalization of an inequality by {T}alagrand and links with the logarithmic {S}obolev inequality},
  author={Otto, Felix and Villani, C{\'e}dric},
  journal={Journal of Functional Analysis},
  volume={173},
  number={2},
  pages={361--400},
  year={2000},
  publisher={Elsevier}
}

@article{chen2016accelerated,
  title={Accelerated dimension-independent adaptive Metropolis},
  author={Chen, Yuxin and Keyes, David and Law, Kody JH and Ltaief, Hatem},
  journal={SIAM Journal on Scientific Computing},
  volume={38},
  number={5},
  pages={S539--S565},
  year={2016},
  publisher={SIAM}
}

@article{betancourt2014optimizing,
  title={Optimizing the integrator step size for {H}amiltonian {M}onte {C}arlo},
  author={Betancourt, MJ and Byrne, Simon and Girolami, Mark},
  journal={arXiv preprint arXiv:1411.6669},
  year={2014}
}

@article{boffi2024flow,
  title={Flow map matching with stochastic interpolants: A mathematical framework for consistency models},
  author={Boffi, Nicholas M and Albergo, Michael S and Vanden-Eijnden, Eric},
  journal={arXiv preprint arXiv:2406.07507},
  year={2024}
}

@article{karthik2023if,
  title={If at first you don't succeed, try, try again: Faithful diffusion-based text-to-image generation by selection},
  author={Karthik, Shyamgopal and Roth, Karsten and Mancini, Massimiliano and Akata, Zeynep},
  journal={arXiv preprint arXiv:2305.13308},
  year={2023}
}
\bibliographystyle{iclr2027_conference}

\appendix
\section{Proofs and clarifications}
\subsection{Equivalence between $\LiMF$ and $\RiMF$}
\label{app:equvi}
Write $\LiMF(\theta)=\mathbb E\|A_\theta-V\|_2^2$, where
the predictor $$A_\theta :=
u_\theta(W_T,R,T)
+(T-R)\left[
\partial_tu_\theta(W_T,R,T)
+(\partial_wu_\theta)(W_T,R,T)\,v_\theta(W_T,T)
\right]$$ is measurable with respect to $(W_T,R,T)$.
Since $(R,T)$ is independent of $(X,Z)$,
\[
\mathbb E[V\mid W_T,R,T]=v(W_T,T).
\]
Hence
\[
\mathbb E\langle A_\theta-v(W_T,T),\,V-v(W_T,T)\rangle=0.
\]
Expanding the squared norm therefore gives
\[
\LiMF(\theta)
=\RiMF(\theta)
+\underbrace{\mathbb E\|V-v(W_T,T)\|_2^2}_{C},
\]
where $C$ is independent of $\theta$.
\subsection{Proof of relation \eqref{eq:nothing}}

\label{apx:nothing}

Since $\rj(\theta) = \lj(\theta) + C$, it holds $\inf \rj = \inf \lj + C$. Therefore, $\rj(\theta)=\lj(\theta) - \inf \lj + \inf \rj$.

\subsection{Details on Remark \ref{rmk:Gthetareg}}
\label{app:rmkGthetareg}

Set
\[
A:=\sup_{0\leq r\leq t\leq1}
\|\bar q_\theta(\cdot,r,t)\|_{L^2(p_t)},
\qquad
B:=\sup_{0\leq r\leq t\leq1}
\|\partial_r\bar q_\theta(\cdot,r,t)\|_{L^2(p_t)}.
\]
By hypothesis, $A,B<\infty$. For $0\leq r<s\leq t\leq1$,
the fundamental theorem of calculus and Minkowski's inequality give
\[
\|\bar q_\theta(\cdot,r,t)-\bar q_\theta(\cdot,s,t)\|_{L^2(p_t)}
\leq
\int_r^s
\|\partial_r\bar q_\theta(\cdot,a,t)\|_{L^2(p_t)}\,da
\leq B(s-r).
\]
Consequently,
\[
\left|
\|\bar q_\theta(\cdot,r,t)\|_{L^2(p_t)}^2
-\|\bar q_\theta(\cdot,s,t)\|_{L^2(p_t)}^2
\right|
\leq 2AB(s-r).
\]
Splitting the integrals defining $G_\theta(r)$ and $G_\theta(s)$
therefore yields
\begin{align*}
|G_\theta(r)-G_\theta(s)|
&\leq
\int_r^s
\|\bar q_\theta(\cdot,r,t)\|_{L^2(p_t)}^2\,dt\\
&\quad+
\int_s^1
\left|
\|\bar q_\theta(\cdot,r,t)\|_{L^2(p_t)}^2
-\|\bar q_\theta(\cdot,s,t)\|_{L^2(p_t)}^2
\right|\,dt\\
&\leq A^2(s-r)+2AB(1-s)(s-r)\\
&\leq (A^2+2AB)|s-r|.
\end{align*}

Now suppose that both $u_\theta$ and $v_\theta$ are finite
compositions of affine maps and coordinatewise SiLU activations,
with direct spatial and time inputs. The SiLU activation
$\rho(a)=a/(1+e^{-a})$ is smooth and has bounded first and
second derivatives. By the chain rule, these networks have
bounded first and second input derivatives and at most linear
growth. Since the time inputs lie in $[0,1]$, there is a finite
constant $C_\theta$ such that
\[
\|u_\theta(w,r,t)\|+\|v_\theta(w,t)\|
\leq C_\theta(1+\|w\|),
\]
and the first and second input derivatives of $u_\theta$ are
uniformly bounded.

Because $v$ and $v_\theta$ do not depend on $r$,
\[
\partial_r\bar q_\theta
=
\partial_r u_\theta-\partial_tu_\theta
-(\partial_wu_\theta)v_\theta
+(t-r)\left[
\partial_{rt}u_\theta
+(\partial_{rw}u_\theta)v_\theta
\right].
\]
The preceding network bounds imply, after enlarging $C_\theta$
if necessary,
\begin{align*}
\|\bar q_\theta(w,r,t)\|
&\leq C_\theta(1+\|w\|)+\|v(w,t)\|,\\
\|\partial_r\bar q_\theta(w,r,t)\|
&\leq C_\theta(1+\|w\|),
\end{align*}
uniformly over $0\leq r\leq t\leq1$.

Let $M_2:=\mathbb E\|X\|^2<\infty$ where $X \sim \pi_0$ (Assumption \ref{assm:prior-flow}). Independence of $X$ and
$Z\sim\mathcal N(0,I_d)$ gives
\[
\mathbb E\|W_t\|^2
=(1-t)^2M_2+t^2d
\leq M_2+d.
\]
Moreover, since $v(W_t,t)=\mathbb E[Z-X\mid W_t]$,
conditional Jensen's inequality gives
\[
\|v(\cdot,t)\|_{L^2(p_t)}^2
\leq \mathbb E\|Z-X\|^2
=M_2+d.
\]
These estimates establish $A,B<\infty$, proving the network
claim. The resulting constants may depend on the network
parameters.

\subsection{Details on Remark \ref{rmk:averagebound}}

\label{proof:rmkaveragebound}

Write $e_\theta(w,t):=v(w,t)-v_\theta(w,t)$. The operator-norm
bound implies, for $r<t$,
\[
(t-r)\|\partial_wu_\theta(w,r,t)e_\theta(w,t)\|
\leq B_u\|e_\theta(w,t)\|.
\]
Squaring, integrating against $p_t$, and summing over the
schedule intervals gives
\begin{align*}
&\sum_{k=1}^K\int_{t_{k-1}}^{t_k}(t-t_{k-1})^2
\|\partial_wu_\theta(\cdot,t_{k-1},t)
e_\theta(\cdot,t)\|_{L^2(p_t)}^2\,dt\\
&\qquad\leq B_u^2
\sum_{k=1}^K\int_{t_{k-1}}^{t_k}
\|e_\theta(\cdot,t)\|_{L^2(p_t)}^2\,dt
=B_u^2\mathcal R_{\mathrm{FM}}(\theta),
\end{align*}
where the last equality uses $T\sim\operatorname{Unif}[0,1]$
and the fact that the schedule intervals partition $[0,1]$.

If $u_\theta$ is uniformly
$L_u$-Lipschitz in $w$, then
$\|\partial_wu_\theta(w,r,t)\|_{\mathrm{op}}\leq L_u$.
Since $t-r\leq1$, the sufficient condition follows with
$B_u=L_u$.

Finally, \eqref{eq:jacobian} controls only the Jacobian acting
on $e_\theta$, averaged over space and time; it does not
require a uniform operator-norm bound. For example, when
$v_\theta=v$, its left-hand side vanishes.

For finite affine-SiLU networks with direct spatial and time
inputs, the spatial Jacobian is uniformly bounded (similar to the argument in Appendix \ref{app:rmkGthetareg}):
\[
L_u:=\sup_{w\in\mathbb R^d,\,(r,t)\in[0,1]^2}
\|\partial_wu_\theta(w,r,t)\|_{\mathrm{op}}<\infty.
\]

\subsection{Proof of Theorem~\ref{thm:fsimf}}
\label{proof:thm:fsimf}

\underline{Population risk control:}

By \eqref{eq:nothing} and the definition of the joint marginal risk,
\[
\RiMF(\theta)+\lambda\RFM(\theta)=\rj(\theta)=\delta.
\]
Therefore, $\RiMF(\theta)\leq\delta$ and
$\RFM(\theta)\leq\delta/\lambda$.

\underline{Recovering fixed-time slices:}

Independent uniform time sampling on $[0,1]^2$ implies
\[
\int_0^1G_\theta(r)\,dr
=\int_0^1\int_r^1
\|\bar q_\theta(\cdot,r,t)\|_{L^2(p_t)}^2\,dt\,dr
\leq \int_0^1\int_0^1
\|\bar q_\theta(\cdot,r,t)\|_{L^2(p_t)}^2\,dt\,dr = \RiMF(\theta).
\]
Fix $r\in[0,1]$ and $h\in(0,1]$. Choose an interval
$I=[a,b]\subset[0,1]$ such that $a\leq r\leq b$ and $b-a=h$.
By Assumption~\ref{assm:LG},
\[
G_\theta(r)\leq G_\theta(s)+L_G|s-r|,
\qquad s\in I.
\]
Averaging this inequality over $I$ gives
\[
G_\theta(r)
\leq \frac1h\int_I G_\theta(s)\,ds
+\frac{L_G}{h}\int_I|s-r|\,ds.
\]
The first term is at most $\delta/h$, since $G_\theta\geq0$
and $\int_0^1G_\theta(s)\,ds\leq\RiMF(\theta)\leq\delta$.
For the second term, note that
\[
\int_a^b|s-r|\,ds
=\frac{(r-a)^2+(b-r)^2}{2}
\leq\frac{((r-a)+(b-r))^2}{2}
=\frac{h^2}{2}.
\]
Consequently,
\begin{align}
\label{Gthetar}
    G_\theta(r)\leq\frac{\delta}{h}+\frac{L_Gh}{2}.
\end{align}
If $0<\delta\leq L_G/2$, choosing
$h=\sqrt{2\delta/L_G}\in(0,1]$ to minimize the bound in \eqref{Gthetar}, we get
\begin{equation}
G_\theta(r)\leq(2L_G\delta)^{1/2},
\qquad \forall r\in[0,1].
\label{eq:trace}
\end{equation}
If $\delta=0$, the same bound follows by letting $h\downarrow0$ in \eqref{Gthetar}.

\underline{Local transport error:}

Define
\begin{align*}
q_\theta(w,r,t)
&:=u_\theta(w,r,t)-v(w,t)
+(t-r)\left[
\partial_tu_\theta(w,r,t)
+\partial_wu_\theta(w,r,t)v(w,t)\right]\\
&=\bar q_\theta(w,r,t)
+(t-r)\partial_wu_\theta(w,r,t)
\bigl(v(w,t)-v_\theta(w,t)\bigr).
\end{align*}
Fix $0\leq r<t\leq1$ and, for $p_t$-almost every $w$, set (see Assumption \ref{assm:prior-flow} for the flow definition)
\[
w_s:=\varphi_{s,t}(w),
\qquad
D_s:=(s-r)u_\theta(w_s,r,s)-w_s,
\qquad s\in[r,t].
\]
Assumptions~\ref{assm:prior-flow} and~\ref{assm:u-regularity}
ensure that $D_s$ is absolutely continuous. The chain rule gives
$\frac{d}{ds}D_s=q_\theta(w_s,r,s)$ for almost every $s$.
Since $D_r=-\varphi_{r,t}(w)$ and
$D_t=(t-r)u_\theta(w,r,t)-w$, integration gives
\begin{equation}
(t-r)u_\theta(w,r,t)-w+\varphi_{r,t}(w)
=\int_r^t q_\theta(\varphi_{s,t}(w),r,s)\,ds.
\label{eq:characteristic-residual}
\end{equation}

For the fixed generation schedule, let
$F_k:=\varphi_{t_{k-1},t_k}$ be the exact step map, and define
\begin{align*}
Q_k&:=\int_{t_{k-1}}^{t_k}
\|\bar q_\theta(\cdot,t_{k-1},t)\|_{L^2(p_t)}^2\,dt,\\
J_{v,k}&:=\int_{t_{k-1}}^{t_k}(t-t_{k-1})^2
\|\partial_wu_\theta(\cdot,t_{k-1},t)
(v-v_\theta)(\cdot,t)\|_{L^2(p_t)}^2\,dt.
\end{align*}
Equation~\eqref{eq:characteristic-residual} gives
\[
F_{k,\theta}(w)-F_k(w)
=-\int_{t_{k-1}}^{t_k}
q_\theta(\varphi_{s,t_k}(w),t_{k-1},s)\,ds.
\]
Using $(\varphi_{s,t_k})_\#p_{t_k}=p_s$, Minkowski's inequality, and Cauchy-Schwarz in time, we obtain

\begin{align}
\|F_{k,\theta}-F_k\|_{L^2(p_{t_k})}
&\leq \int_{t_{k-1}}^{t_k}
\|q_\theta(\varphi_{s,t_k}(\cdot),t_{k-1},s)\|_{L^2(p_{t_k})}\,ds
\notag\\
&= \int_{t_{k-1}}^{t_k}
\|q_\theta(\cdot,t_{k-1},s)\|_{L^2(p_s)}\,ds
\notag\\
&\leq \int_{t_{k-1}}^{t_k}
\|\bar q_\theta(\cdot,t_{k-1},s)\|_{L^2(p_s)}\,ds
\notag\\
&\quad+\int_{t_{k-1}}^{t_k}(s-t_{k-1})
\|\partial_wu_\theta(\cdot,t_{k-1},s)
(v-v_\theta)(\cdot,s)\|_{L^2(p_s)}\,ds
\notag\\
&\leq \sqrt{h_k}\bigl(\sqrt{Q_k}+\sqrt{J_{v,k}}\bigr).
\label{eq:local-step-error}
\end{align}

\underline{Composing the step maps:}

The flow property gives
$S=F_1\circ\cdots\circ F_K=\varphi_{0,1}$.
Define the intermediate compositions
\[
H_k:=F_{1,\theta}\circ\cdots\circ F_{k,\theta}
\circ F_{k+1}\circ\cdots\circ F_K,
\qquad k=0,\ldots,K,
\]
where empty compositions are the identity. Thus $H_0=S$,
$H_K=S_\theta$, and
\[
S_\theta(z)-S(z)
=\sum_{k=1}^K\bigl(H_k(z)-H_{k-1}(z)\bigr).
\]
The maps $H_k$ and $H_{k-1}$ differ only at the $k$th step.
Their common inner composition is the exact flow
$F_{k+1}\circ\cdots\circ F_K=\varphi_{t_k,1}$,
while their common outer composition
$F_{1,\theta}\circ\cdots\circ F_{k-1,\theta}$
has Lipschitz constant at most
$\Lambda_k=\prod_{j<k}L_j$. Hence
\[
\|H_k(z)-H_{k-1}(z)\|
\leq\Lambda_k
\|F_{k,\theta}(\varphi_{t_k,1}(z))
-F_k(\varphi_{t_k,1}(z))\|.
\]
Since $(\varphi_{t_k,1})_\#\gamma=p_{t_k}$,
taking $L^2(\gamma)$ norms and applying the triangle inequality
yields
\[
\|S_\theta-S\|_{L^2(\gamma)}
\leq\sum_{k=1}^K\Lambda_k
\|F_{k,\theta}-F_k\|_{L^2(p_{t_k})}.
\]
Consequently,
\begin{align*}
\|S_\theta-S\|_{L^2(\gamma)}
&\leq\sum_{k=1}^K\Lambda_k
\|F_{k,\theta}-F_k\|_{L^2(p_{t_k})}\\
&\leq\sum_{k=1}^K\Lambda_k\sqrt{h_k}
\bigl(\sqrt{Q_k}+\sqrt{J_{v,k}}\bigr)\\
&\leq A_K(2L_G\delta)^{1/4}
+B_K\left(\sum_{k=1}^KJ_{v,k}\right)^{1/2}\\
&\leq A_K(2L_G\delta)^{1/4}
+\frac{B_uB_K}{\sqrt\lambda}\delta^{1/2}.
\end{align*}
The third line uses Cauchy--Schwarz over $k$ for the
$J_{v,k}$ terms. Squaring proves \eqref{eq:few-step}.
For $K=1$, we have $h_1=\Lambda_1=A_1=B_1=1$, so the same
argument gives the one-step transport bound. In this case there
is no outer learned map in the composition estimate, and no
Lipschitz condition on $F_{1,\theta}$ is needed.

Finally, if $u_\theta$ is uniformly $L_u$-Lipschitz in space, then
\[
L_k\leq1+h_kL_u,
\qquad
\Lambda_k\leq
\exp\left(L_u\sum_{j<k}h_j\right)\leq e^{L_u}.
\]
Since $\sum_kh_k=1$, Cauchy--Schwarz yields
$A_K\leq e^{L_u}\sqrt K$ and $B_K\leq e^{L_u}$.
Also $L_{S_\theta}\leq\prod_kL_k\leq e^{L_u}$, as used
in the data-space posterior bounds. All estimates concern the
fixed generation schedule.

\subsection{Proof of Theorem~\ref{prop:posterior-stability}}\label{proof:posterior-stability}
Write $e:=\mathcal E^{1/2}$, $a(z):=\ell_y(S_\theta(z))$ and
$b(z):=\ell_y(S(z))$. The weight assumptions give
\[
|Z_{\theta,y}-Z_y|\leq\int|a-b|\,d\gamma\leq L_ye.
\]
Hence $L_ye\leq Z_y/2$ ensures $Z_{\theta,y}\geq Z_y/2>0$.
Let $U_2(z)=1+\|z\|^2$ and use the weighted-variation convention
\[
\|\mu-\nu\|_{U_2}:=\sup_{|f|\leq U_2}\left|\int f\,d\mu-\int f\,d\nu\right|.
\]
For each such $f$, decompose the normalized densities as
\[
\int f\,d(\nu_\theta^y-\nu^y)
=\frac1{Z_{\theta,y}}\int f(a-b)\,d\gamma
+\frac{Z_y-Z_{\theta,y}}{Z_yZ_{\theta,y}}\int fb\,d\gamma.
\]
Since $\mathbb E_\gamma U_2=1+d$ and
$\mathbb E_\gamma U_2^2=d^2+4d+1$, Cauchy-Schwarz gives
\begin{align*}
\|\nu_\theta^y-\nu^y\|_{U_2}
&\leq\frac{L_ye}{Z_{\theta,y}}
\left[\sqrt{d^2+4d+1}+\frac{M_y(1+d)}{Z_y}\right]\\
&\leq\frac{C_y(d)}2e.
\end{align*}
For completeness, couple the common part of two measures identically and
couple their remaining positive parts. The inequality
$\|z-z'\|^2\leq2\|z\|^2+2\|z'\|^2$ then gives
\[
W_2^2(\mu,\nu)\leq2\int\|z\|^2\,d|\mu-\nu|(z)
\leq2\|\mu-\nu\|_{U_2}.
\]
This proves \eqref{eq:source-stability}. All source second moments are finite because the normalized
weights are bounded relative to $\gamma$.

For the data-space statement, first note that $S_\#\gamma=\pi_0$ implies
$S_\#\nu^y=\pi^y$ by the definition of a pushforward. Insert the intermediate
measure $(S_\theta)_\#\nu^y$. Lipschitz propagation and the triangle inequality give
\[
W_2\big((S_\theta)_\#\nu_\theta^y,\pi^y\big)
\leq L_{S_\theta}W_2(\nu_\theta^y,\nu^y)
+W_2\big((S_\theta)_\#\nu^y,S_\#\nu^y\big).
\]
The common-input coupling bounds the second term by
\[
W_2^2\big((S_\theta)_\#\nu^y,S_\#\nu^y\big)
\leq\mathbb E_{\nu^y}\|S_\theta-S\|^2
\leq\frac{M_y}{Z_y}\mathcal E,
\]
which proves the data-space stability. Substituting \eqref{eq:transport-error-bound} yields \eqref{eq:end-to-end-source}--\eqref{eq:end-to-end-data}. 

\subsection{Geometric convergence of the source samplers}
\label{app:sampler-convergence}

Fix $\theta$ and $y$, and write
\[
\omega(z):=\ell_y(S_\theta(z)),\qquad
\mu_j(\dd z):=Z_j^{-1}\omega(z)^{\beta_j}\gamma(\dd z),\qquad
Z_j:=\int\omega^{\beta_j}\,\dd\gamma,\qquad
\Pi:=\bigotimes_{j=0}^J\mu_j.
\]
Thus $\mu_0=\gamma$ and $\mu_J=\nu_\theta^y$.
Let $\eta_n$ be the joint law after $n$ production sweeps and
$\nu_n$ its cold marginal. We use
$\chi^2(\eta\|\Pi):=\int(\dd\eta/\dd\Pi-1)^2\,\dd\Pi$.

\begin{assumption}[Source weight and production sampler]
\label{assm:sampler-convergence}
The source weight is measurable and satisfies
$0<m\leq\omega(z)\leq M<\infty$ for every $z\in\mathbb R^d$,
where $m<M$. The ladder $0=\beta_0<\cdots<\beta_J=1$, $J\geq1$,
and all local-kernel parameters are fixed during production.
Each sweep independently applies the local kernels to all replicas,
then attempts each adjacent Metropolis swap once in a prescribed order;
this order may vary with the sweep index.
Every pCN angle satisfies $h_{\beta_j}\in(0,\pi/2]$.
For the hybrid sampler, $\Psi_\theta=-\log\omega$ is continuously
differentiable, and the cold update is the Metropolis-corrected
split HMC kernel of Section~\ref{sec:sampler}, with fresh independent
Gaussian momentum at each update. The initial joint law satisfies
$D_0:=\chi^2(\eta_0\|\Pi)<\infty$.
\end{assumption}

The positive lower bound is an additional condition for the convergence
result. It holds, for example, for exponentiated bounded rewards.
Geometric convergence of SPT+pCN is established by \citet{wangtest}.
The following statement gives explicit bounds for both local-kernel
choices and extends the guarantee to the hybrid sampler.

\begin{theorem}[Geometric convergence of SPT+pCN and SPT+hybrid]
\label{thm:sampler-convergence}
Under Assumption~\ref{assm:sampler-convergence}, set $\kappa:=M/m$.
For \emph{iMF+SPT+pCN}, define
\[
g_{\mathrm{pCN}}:=\min_{0\leq j\leq J}
\kappa^{-\beta_j}(1-\cos h_{\beta_j}),\qquad
q_{\mathrm{pCN}}:=1-g_{\mathrm{pCN}}\in(0,1).
\]
Then, for every integer $n\geq0$,
\begin{equation}
\chi^2(\eta_n\|\Pi)\leq D_0q_{\mathrm{pCN}}^{2n},\qquad
W_2(\nu_n,\nu_\theta^y)
\leq\sqrt{2\kappa D_0}\,q_{\mathrm{pCN}}^n.
\label{eq:pcn-geometric}
\end{equation}
For \emph{iMF+SPT+hybrid}, define
\[
g_{\mathrm{hot}}:=\min_{0\leq j<J}
\kappa^{-\beta_j}(1-\cos h_{\beta_j}),\qquad
a:=\kappa^{-(1-\beta_{J-1})},\qquad
q_{\mathrm{hyb}}:=1-g_{\mathrm{hot}}^2a\in(0,1).
\]
Then
\begin{equation}
\begin{aligned}
\chi^2(\eta_n\|\Pi)
&\leq D_0q_{\mathrm{hyb}}^{2\lfloor n/2\rfloor},\\
W_2(\nu_n,\nu_\theta^y)
&\leq\sqrt{2\kappa D_0}\,q_{\mathrm{hyb}}^{\lfloor n/2\rfloor}.
\end{aligned}
\label{eq:hybrid-geometric}
\end{equation}
The hybrid bounds remain valid if the cold update is replaced by any
fixed $\mu_J$-invariant Markov kernel. In particular, no separate
spectral-gap assumption on split HMC is required.
If $\eta_0=\gamma^{\otimes(J+1)}$, then
$D_0\leq\kappa^{\sum_{j=0}^J\beta_j}-1$.
\end{theorem}

\begin{proof}
All operator norms below are on the centered space
$L_0^2(\Pi):=\{f\in L^2(\Pi):\int f\,\dd\Pi=0\}$,
unless another measure is specified.
Each local kernel and each Metropolis swap preserves $\Pi$, so each
is an $L^2(\Pi)$ contraction. Their composition therefore preserves
the joint target, whose cold marginal is $\nu_\theta^y$.

\paragraph{Contraction of a pCN update.}
Let $\mathsf Q_j$ be the Gaussian pCN proposal kernel and
$\mathsf A_j$ its Metropolis correction for $\mu_j$.
The Gaussian kernel is self-adjoint and positive on $L^2(\gamma)$,
with centered norm $\cos h_{\beta_j}$.
Writing $a_j:=\dd\mu_j/\dd\gamma$, its accepted transition measure is
\[
\min\{a_j(z),a_j(z')\}\,\gamma(\dd z)\mathsf Q_j(z,\dd z').
\]
Hence comparison of Dirichlet forms and variances gives
\begin{align*}
\langle f,(I-\mathsf A_j)f\rangle_{\mu_j}
&\geq (\inf a_j)\langle f,(I-\mathsf Q_j)f\rangle_\gamma\\
&\geq\frac{\inf a_j}{\sup a_j}
(1-\cos h_{\beta_j})\operatorname{Var}_{\mu_j}(f)\\
&\geq\kappa^{-\beta_j}(1-\cos h_{\beta_j})
\operatorname{Var}_{\mu_j}(f).
\end{align*}
The corrected kernel is also positive. Indeed, the identity
\[
\min\{a_j(z),a_j(z')\}
=\int_0^\infty
\mathbf 1_{\{a_j(z)\geq s\}}\mathbf 1_{\{a_j(z')\geq s\}}\,\dd s
\]
expresses the accepted part of its quadratic form as an integral of
nonnegative quadratic forms of $\mathsf Q_j$; the rejection part
is nonnegative as well. Reversibility and positivity therefore imply
\begin{equation}
\|\mathsf A_j\|_{L_0^2(\mu_j)\to L_0^2(\mu_j)}
\leq1-\kappa^{-\beta_j}(1-\cos h_{\beta_j}).
\label{eq:pcn-operator-contraction}
\end{equation}

\paragraph{All-pCN sweeps.}
When every replica uses pCN, tensorization of
\eqref{eq:pcn-operator-contraction} bounds the centered norm of the
joint local update by $q_{\mathrm{pCN}}$.
The subsequent swap sweep is a contraction, so the centered norm
after $n$ sweeps is at most $q_{\mathrm{pCN}}^n$.
Applying the adjoint operator to $\dd\eta_0/\dd\Pi-1$ proves the
first bound in \eqref{eq:pcn-geometric}.

\paragraph{Hybrid sweeps.}
Let $\mathsf H$ be the cold invariant kernel and write
$\mathsf L=(\bigotimes_{j<J}\mathsf A_j)\otimes\mathsf H$
for the local update.
Let $\mathsf E$ be the orthogonal projection onto functions of the
cold coordinate, obtained by integrating all hot coordinates against
$\bigotimes_{j<J}\mu_j$. The operators $\mathsf L$ and $\mathsf E$
commute, and \eqref{eq:pcn-operator-contraction} gives
\[
\|\mathsf L\mathsf E\|\leq1,\qquad
\|\mathsf L(I-\mathsf E)\|\leq r,
\qquad r:=1-g_{\mathrm{hot}}\in[0,1).
\]
For a complete swap sweep $\mathsf S$, every swap involving only hot
replicas satisfies $\mathsf E\mathsf S_j=\mathsf S_j\mathsf E=\mathsf E$.
Since the edge $(J-1,J)$ is attempted once,
$\mathsf E\mathsf S\mathsf E=\mathsf E\mathsf S_{J-1}\mathsf E$.
On functions of the cold coordinate this is an independence
Metropolis kernel with target $\mu_J$ and proposal $\mu_{J-1}$.
Moreover,
\[
\frac{\dd\mu_J}{\dd\mu_{J-1}}(z)
=\frac{Z_{J-1}}{Z_J}\omega(z)^{1-\beta_{J-1}}
\leq\kappa^{1-\beta_{J-1}}=a^{-1}.
\]
Its transition kernel thus minorizes $a\mu_J$, yielding
$\|\mathsf E\mathsf S\mathsf E\|\leq1-a$ on $L_0^2(\Pi)$.

Set $\mathsf D:=\mathsf E+r(I-\mathsf E)$.
For $r>0$, the operator
$\mathsf T:=\mathsf L\mathsf E+r^{-1}\mathsf L(I-\mathsf E)$
is a contraction and satisfies $\mathsf L=\mathsf T\mathsf D
=\mathsf D\mathsf T$. For $r=0$ the same factorization holds with
$\mathsf T=\mathsf L$. Consequently,
\begin{align*}
\|\mathsf L\mathsf S\mathsf L\|
&\leq\|\mathsf D\mathsf S\mathsf D\|\\
&\leq r^2+2r(1-r)+(1-r)^2(1-a)\\
&=1-g_{\mathrm{hot}}^2a=q_{\mathrm{hyb}}.
\end{align*}
Here we expanded $\mathsf D=rI+(1-r)\mathsf E$ and used that
$\mathsf S$, $\mathsf E\mathsf S$, and $\mathsf S\mathsf E$
are contractions. Each production sweep acts on functions as
$\mathsf L\mathsf S_n$, so every two consecutive sweeps contract
the centered norm by at most $q_{\mathrm{hyb}}$, regardless of their
prescribed swap orders. Applying the adjoint product to the initial
density difference proves the $\chi^2$ bound in
\eqref{eq:hybrid-geometric}. This argument only uses invariance of
$\mathsf H$, which proves the assertion for general cold kernels.

\paragraph{Cold-marginal Wasserstein bounds.}
The density of $\mu_J$ relative to $\gamma$ has upper-to-lower ratio
at most $\kappa$. The Gaussian logarithmic Sobolev inequality,
bounded perturbation, and the implication from logarithmic Sobolev
to transport inequalities give \citep{otto2000generalization}
\[
W_2^2(\nu,\mu_J)\leq2\kappa\KL(\nu\|\mu_J)
\leq2\kappa\chi^2(\nu\|\mu_J);
\]
see also \citet[Lemma~A.10]{wangtest}.
Conditional Jensen's inequality gives
$\chi^2(\nu_n\|\mu_J)\leq\chi^2(\eta_n\|\Pi)$.
Combining these inequalities with the joint bounds proves
\eqref{eq:pcn-geometric} and \eqref{eq:hybrid-geometric}.
In particular, the latter has the form $C\rho^n$ with
$\rho=\sqrt{q_{\mathrm{hyb}}}$ and
$C=\sqrt{2\kappa D_0/q_{\mathrm{hyb}}}$.

Finally, for independent Gaussian initialization,
\[
D_0+1
=\prod_{j=0}^J\left[
\left(\int\omega^{\beta_j}\,\dd\gamma\right)
\left(\int\omega^{-\beta_j}\,\dd\gamma\right)\right]
\leq\kappa^{\sum_{j=0}^J\beta_j}.
\]
\end{proof}

\begin{corollary}[End-to-end convergence]
\label{cor:end-to-end-convergence}
Suppose the assumptions of Theorem~\ref{thm:fsimf},
Proposition~\ref{prop:posterior-stability}, and
Theorem~\ref{thm:sampler-convergence} hold. Assume also that
$S_\theta$ is $L_{S_\theta}$-Lipschitz,
$0\leq\delta\leq L_G/2$, and $L_yb_K(\delta)\leq Z_y/2$.
Define
\[
s_n:=\sqrt{2\kappa D_0}
\begin{cases}
q_{\mathrm{pCN}}^n, & \text{for iMF+SPT+pCN},\\
q_{\mathrm{hyb}}^{\lfloor n/2\rfloor},
& \text{for iMF+SPT+hybrid}.
\end{cases}
\]
Then the data-space output law satisfies
\begin{equation}
\begin{aligned}
W_2\big((S_\theta)_\#\nu_n,\pi^y\big)
\leq{}&L_{S_\theta}s_n
+L_{S_\theta}C_y(d)^{1/2}b_K(\delta)^{1/2}\\
&+\sqrt{\frac{M_y}{Z_y}}b_K(\delta).
\end{aligned}
\label{eq:complete-convergence}
\end{equation}
\end{corollary}

\section{Additional experiments}
\label{app:experiments}

\subsection{Empirical loss-to-map-error relationship}
\label{app:loss_map_error}

We examine the relationship between the marginal joint risk and transport-map error in Theorem~3.7 using the two-dimensional banana prior
\[
X=\bigl(U,\;0.62(U^2-1)+0.30V\bigr),
\qquad U,V\overset{\mathrm{iid}}{\sim}\mathcal{N}(0,1).
\]
For independent $X\sim\pi_0$ and $Z\sim\mathcal{N}(0,I_2)$, we use the interpolation $W_t=(1-t)X+tZ$ and marginal velocity $v(w,t)=\mathbb{E}[Z-X\mid W_t=w]$. The reference map $S=\varphi_{0,1}$ is the transport induced by this marginal velocity, rather than the explicit triangular map used to generate banana samples.

\paragraph{Numerical reference.}
Conditioning on the first coordinate of $W_t$ gives a Gaussian conditional distribution for $U$. Integrating out $V$ analytically reduces evaluation of $v(w,t)$ to one-dimensional Gaussian quadrature, with weights determined by the second coordinate of $W_t$. We use 128-point Gauss--Hermite quadrature and integrate the marginal ODE from $t=1$ to $t=0$ using DOP853 in double precision, with relative and absolute tolerances $10^{-8}$ and $10^{-10}$. Doubling the quadrature order changes the reference velocities by $6.9\times10^{-11}$ RMS over the risk-evaluation inputs. On 128 source samples, doubling the quadrature order and tightening both ODE tolerances tenfold change the reference maps by $8.3\times10^{-15}$ and $8.7\times10^{-9}$ RMS, respectively. These sensitivity checks indicate that numerical reference error is negligible at the observed map-error scale.

\paragraph{Training and evaluation.}
We train an iMF model with three hidden SiLU layers of width 128 and separate average-velocity and auxiliary-velocity heads on 20,000 fixed prior samples. Training uses Adam with learning rate $10^{-3}$, weight decay $10^{-5}$, batch size 512, and 16,000 updates. We use independent data--noise coupling, independent unsorted uniform times $(r,t)\in[0,1]^2$, and auxiliary-loss weight $\lambda=1$. The figure reports a single training run with seed 0. At initialization and every 500 updates, we estimate
\[
\delta=R_{\mathrm{iMF}}+\lambda R_{\mathrm{FM}}
\]
directly from the marginal residuals using 8,192 held-out interpolation samples and the numerical reference velocity. Thus, $\widehat{\delta}$ estimates the marginal risk appearing in Theorem~3.7, rather than the raw training loss containing irreducible regression noise. We also estimate the squared map error
\[
\widehat E_K
=\frac{1}{2048}\sum_{i=1}^{2048}
\bigl\|S_{\theta,K}(z_i)-S(z_i)\bigr\|_2^2,
\qquad z_i\sim\mathcal{N}(0,I_2),
\]
for $K\in\{1,6\}$ uniformly spaced generation steps. The evaluation draws are independent of the training dataset and fixed across checkpoints.

\paragraph{Results.}
Figure~\ref{fig:loss_map_error} shows that decreasing marginal risk is generally accompanied by decreasing map error for both generation schedules. Over training, $\widehat{\delta}$ decreases from $2.743$ to $0.0653$, while the squared map error decreases from $0.554$ to $0.0198$ for one-step generation and from $0.554$ to $0.00732$ for six-step generation. The relationship is not strictly monotonic across checkpoints. These results provide empirical support for the loss-to-map-error connection in Theorem~3.7, but do not establish its asymptotic rate or uniform control of the regularity constants, which may change during training.

\begin{figure}[t]
    \centering
    \includegraphics[width=0.5\linewidth]{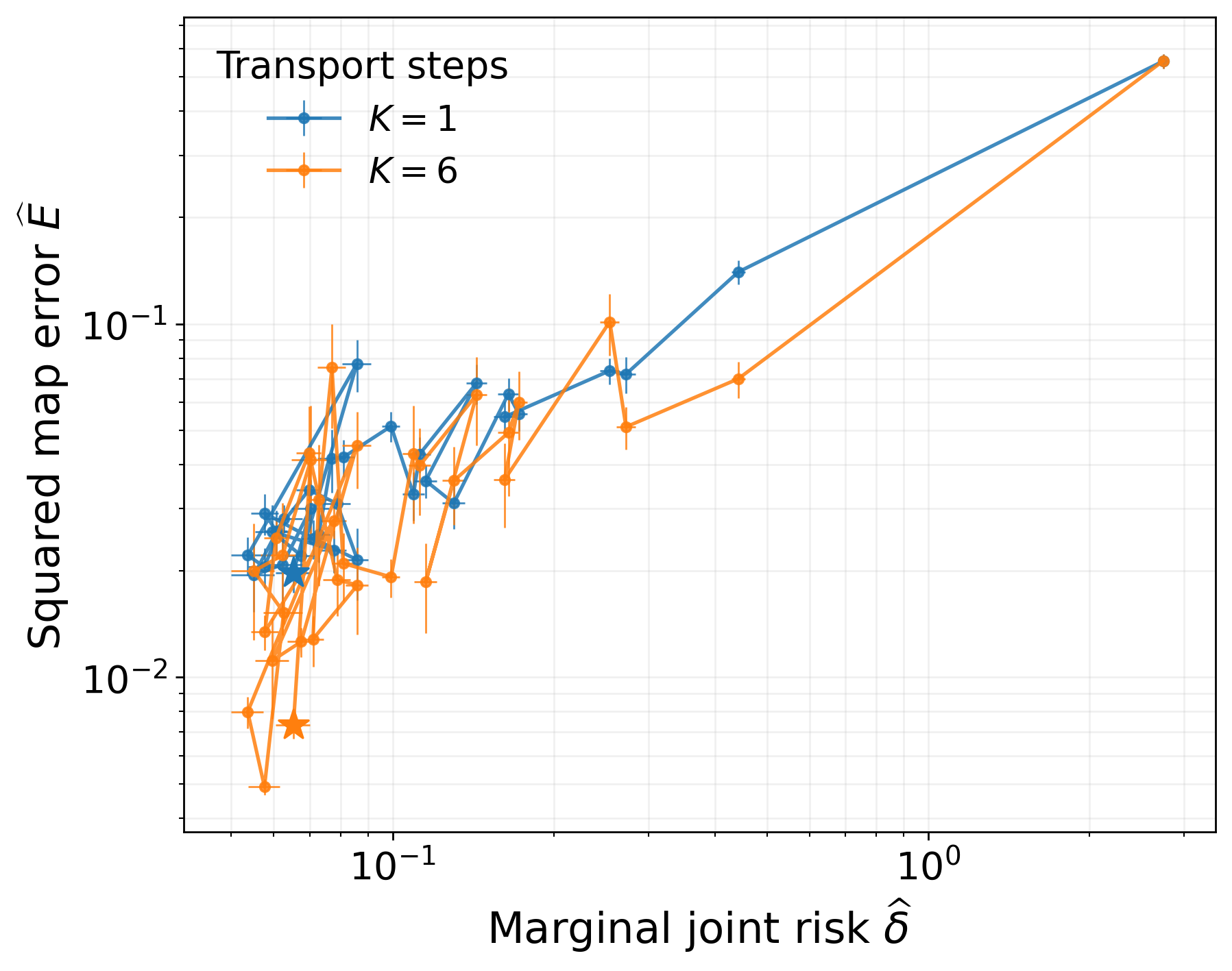}
    \caption{Marginal joint risk versus squared transport-map error on the two-dimensional banana prior for one- and six-step iMF generation. Points represent training checkpoints, connected in training order; stars mark the final checkpoint. Error bars show one Monte Carlo standard error, conditional on the fitted model and numerical reference, and do not represent variability across training runs.}
    \label{fig:loss_map_error}
\end{figure}

\subsection{Synthetic experiments}
\label{app:synthetic}

\paragraph{Prior training.}
For each scenario, all models are trained on the same fixed dataset of 20,000 prior samples using four-hidden-layer SiLU networks of width 512, with 16,000 optimizer updates at learning rate $10^{-3}$. The iMF model uses raw time inputs $(r,t)$ and separate average-velocity and auxiliary-velocity heads; FM and diffusion use four Fourier time harmonics. Both iMF and FM use independent data-noise coupling. For iMF, we sample $r,t$ independently from $\mathrm{Unif}[0,1]$, and set the auxiliary-loss weight to $\lambda=1$. We use Adam for iMF with weight decay $10^{-5}$ and AdamW for FM and diffusion with zero weight decay, both with $(\beta_1,\beta_2)=(0.9,0.999)$. Gradient clipping at norm 5 is applied only to FM and diffusion. Figure~\ref{fig:prior} shows that all three models capture the main geometric structure of the banana and sine priors.

\begin{figure}[t]
    \centering
    \begin{subfigure}[t]{0.85\linewidth}
        \centering
        \includegraphics[width=\linewidth]{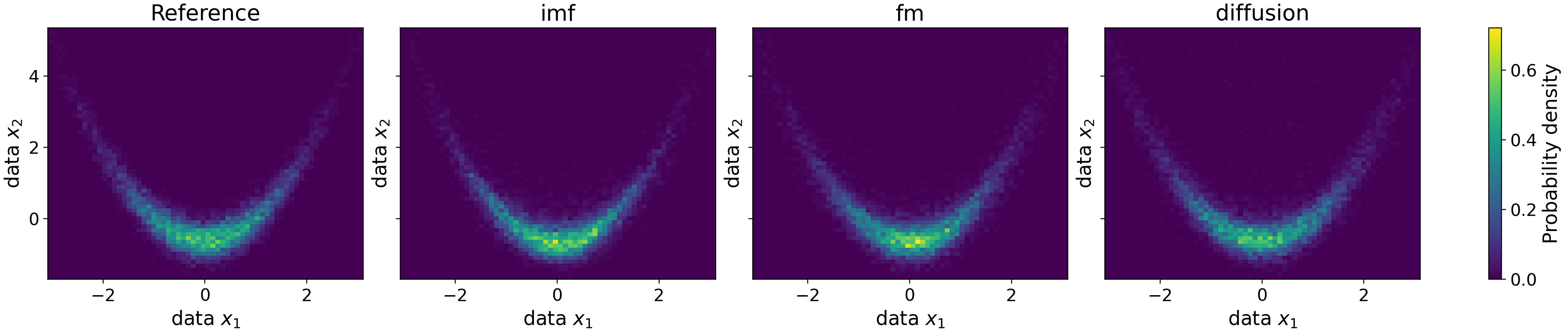}
        \caption{Banana prior.}
        \label{fig:synthetic_banana}
    \end{subfigure}
    \hfill
    \begin{subfigure}[t]{0.85\linewidth}
        \centering
        \includegraphics[width=\linewidth]{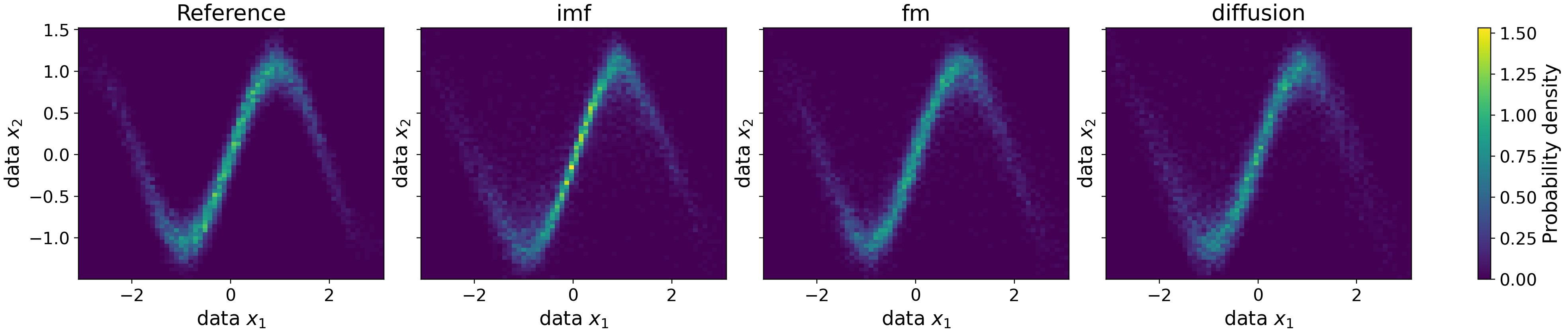}
        \caption{Sine prior.}
        \label{fig:synthetic_sine}
    \end{subfigure}

    \caption{
    Learning the priors by iMF, FM, and diffusion model.
    }
    \label{fig:prior}
\end{figure}

\paragraph{Posterior sampling.}
The two iMF samplers share a six-step transport with uniformly spaced times, while FM+SPT+pCN uses 100 RK4 steps (400 network evaluations per transport). All SPT methods use ten parallel ladders with 24 inverse temperatures $\beta_\ell=(\ell/23)^2$, $\ell=0,\ldots,23$, initialized from $\mathcal{N}(0,I_{32})$. The pCN variants update every replica with pCN; the hybrid replaces the coldest update with split HMC using six integration steps and initial step size $0.05$. During 500 adaptation sweeps, proposal parameters are shared across ladders and updated separately at each temperature using the observed acceptance fraction across the ten ladders. We apply additive updates to the logit pCN noise scale and log HMC step size, proportional to the difference between observed and target acceptance, with respective gains $0.10/\sqrt{1+s/30}$ and $0.055/\sqrt{1+s/30}$ at sweep $s=0,\ldots,499$. We target acceptance rates of $0.30$ for pCN and $0.75$ for split HMC. These are practical choices consistent with the pCN tuning range $0.3$--$0.5$ used by \citet{chen2016accelerated} and the general (leapfrog) HMC range $0.6$-$0.9$ discussed by \citet{betancourt2014optimizing}. The hottest pCN noise scale remains fixed at $0.999$. Each sweep includes odd and even adjacent-swap passes, with their order alternating across sweeps. After adaptation, all proposal parameters are frozen, 500 additional burn-in sweeps are discarded, and 600 consecutive cold states per ladder are retained without thinning, yielding 6,000 samples per method and seed. DPS generates the same number of samples using 1,000 DDPM steps with a linear noise schedule from $10^{-4}$ to $2\times10^{-2}$. Scaling DPS guidance by $\sigma_y^{-2}$ caused numerical instability in our experiments. Following the residual-dependent step-size scheme of \citet[Appendix D.1]{chung2022diffusion} and its official implementation, we use guidance $-0.3\nabla_{x_t}\|G(\hat{x}_0(x_t))-y\|_2$, without explicit noise-variance scaling. Sampling is repeated with seeds $0,1,2,3$; full implementation settings are provided with the released code. Figures~\ref{fig:post_banana} and~\ref{fig:post_sine} compare sample histograms with numerical posterior references. Both iMF- and FM-based SPT recover the main posterior well, while DPS exhibits systematic deviations from the reference distributions.

\paragraph{Evaluation metrics.}
Let $d$ be the dimension, $C$ the number of ladders, and $N$ the number of retained samples per ladder. Writing $\rho_{j,c}(\ell)$ for the lag-$\ell$ autocorrelation of coordinate $j$ in ladder $c$, the mean coordinate ESS is
\[
\mathrm{Coord.\ ESS}
=
\frac{1}{d}\sum_{j=1}^{d}\sum_{c=1}^{C}
\frac{N}{1+2\sum_{\ell=1}^{\infty}\rho_{j,c}(\ell)},
\qquad
\mathrm{Coord.\ ESS/s}
=
\frac{\mathrm{Coord.\ ESS}}{T_{\mathrm{sampling}}}.
\]
For distributions $\mu,\nu$ on $\mathbb{R}^{d}$, the sliced Wasserstein distance is
\[
\mathrm{SW}_2(\mu,\nu)
=
\left(
\int_{\mathbb{S}^{d-1}}
W_2^2\bigl((P_\omega)_\#\mu,(P_\omega)_\#\nu\bigr)
\,d\sigma(\omega)
\right)^{1/2},
\qquad P_\omega(x)=\omega^\top x,
\]
where $\sigma$ is the uniform probability measure on the unit sphere.
For $J$ two-dimensional blocks and a common partition
$\{B_{j,b}\}_{b=1}^{B}$ of each block's first-coordinate marginal, define
$p_{j,b}=\mu\{x:x_{j,1}\in B_{j,b}\}$ and
$q_{j,b}=\nu\{x:x_{j,1}\in B_{j,b}\}$. Then
\[
\mathrm{Bin\ TV}(\mu,\nu)
=
\frac{1}{J}\sum_{j=1}^{J}
\frac12\sum_{b=1}^{B}|p_{j,b}-q_{j,b}|.
\]

\begin{figure}[t]
    \centering
    \begin{subfigure}[t]{0.95\linewidth}
        \centering
        \includegraphics[width=\linewidth]{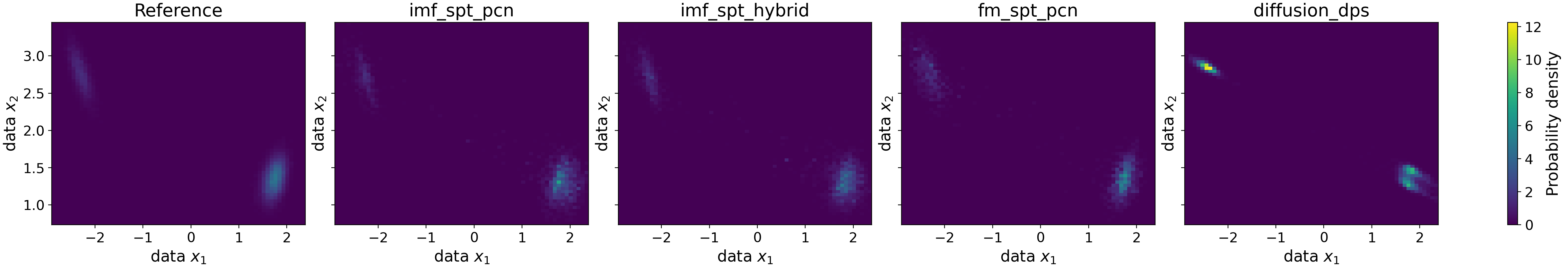}
    \end{subfigure}
    \hfill
    \begin{subfigure}[t]{0.95\linewidth}
        \centering
        \includegraphics[width=\linewidth]{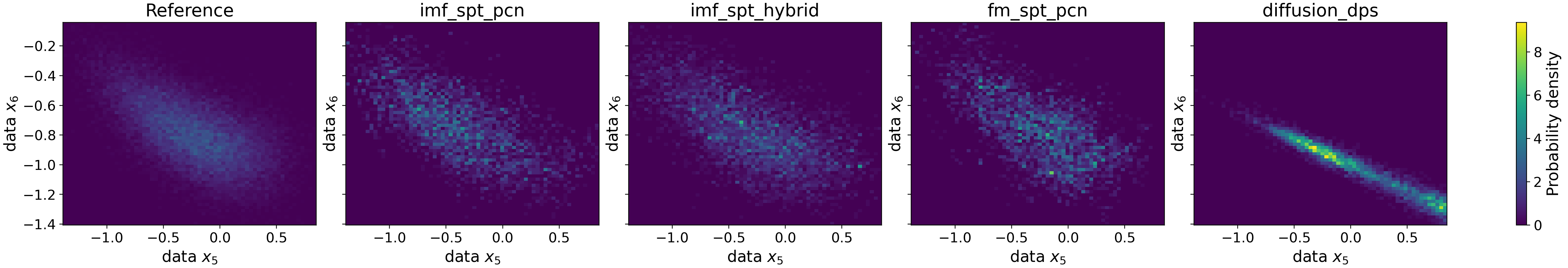}
    \end{subfigure}

    \caption{
    Posterior sample histograms on the banana prior.
    }
    \label{fig:post_banana}
\end{figure}

\begin{figure}[t]
    \centering
    \begin{subfigure}[t]{0.95\linewidth}
        \centering
        \includegraphics[width=\linewidth]{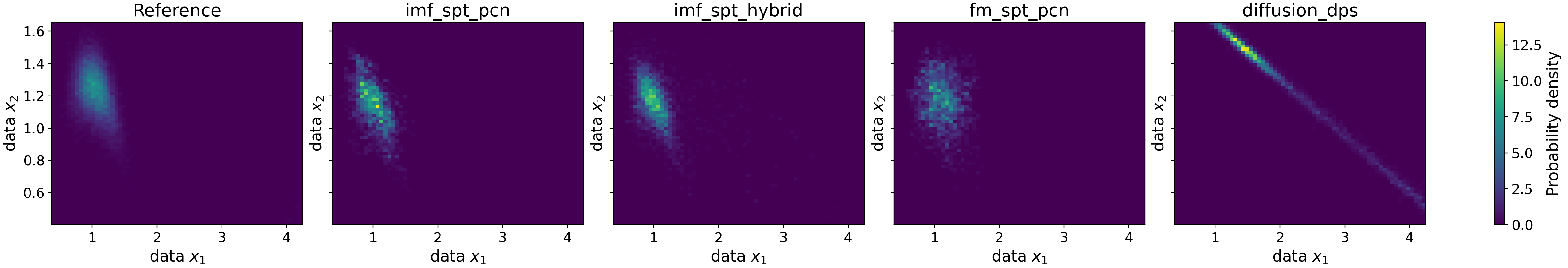}
    \end{subfigure}
    \hfill
    \begin{subfigure}[t]{0.95\linewidth}
        \centering
        \includegraphics[width=\linewidth]{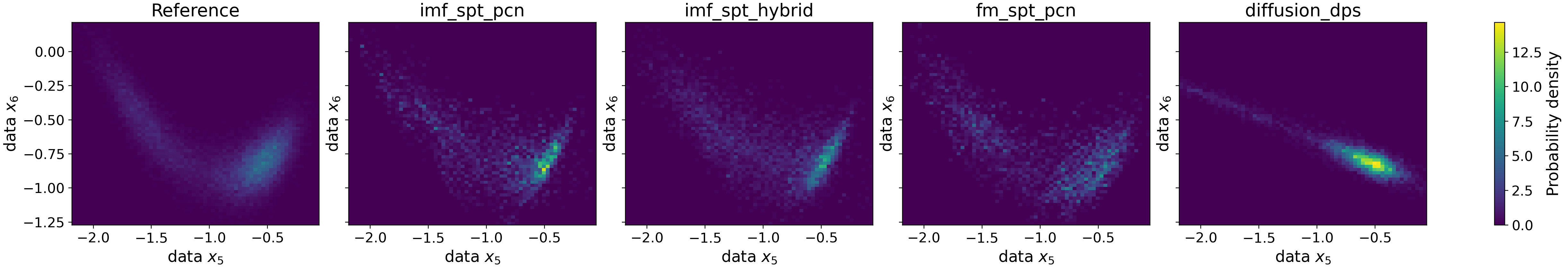}
    \end{subfigure}

    \caption{
    Posterior sample histograms on the sine prior.
    }
    \label{fig:post_sine}
\end{figure}

\subsection{CLIP-guided ImageNet experiments}
\label{app:imagenet}
\subsubsection{Experiment details}

 We use the pretrained iMF-XL/2\footnote{\url{https://github.com/Lyy-iiis/imeanflow/tree/torch}} with one transport step, $\omega=8$, and guidance interval $[0.4,0.65]$, and frozen SiT-XL/2\footnote{\url{https://huggingface.co/sairights/sit-xl-2-256x256-sde-cfg/resolve/5a15dcc/pretrained_models/SiT-XL-2-256x256.safetensors}} with 125 Heun steps (250 velocity evaluations) and classifier-free guidance scale $4$.

For each class-prompt pair, we use eight independent ladders of eight replicas with linear inverse temperatures $\beta_k=k/7$, initialized from $\mathcal N(0,I)$. The iMF methods run 250 adaptation sweeps, 250 burn-in sweeps, and 1,000 retained sweeps per ladder, yielding 8,000 retained cold states. At each adaptation sweep, we average the local proposal acceptance rates over all eight ladders. We then make progressively smaller updates to the pCN proposal scale to target $0.30$ acceptance and, for iMF+SPT+hybrid, to the split-HMC step size at the coldest replica to target $0.75$ acceptance. The resulting proposal parameters are fixed before burn-in. iMF+SPT+hybrid uses six-step split HMC at the coldest replica and pCN at all remaining replicas. SiT+SPT+pCN is run after the two iMF methods under a prompt-specific budget: its sampling budget equals the measured iMF+SPT+hybrid sampling time, with an additional one-hour allowance for loading, preflight, scoring, and export. Its available sampling time is allocated as $1/6$, $1/6$, and $2/3$ to adaptation, burn-in, and retained sampling; only complete sweeps are performed.

\subsubsection{Full results}
\label{fullimagenet}

Table~\ref{tab:clip_guidance_full} reports the complete results, and Figures~\ref{fig:goldenmeadow}-\ref{fig:reward_final_aircraft_carrier_watercolor_tau200_r8_l8_v1} display samples from the first two ladders for each method.
\begingroup
\small
\setlength{\tabcolsep}{4pt}
\begin{longtable}{@{}lccc@{}}
\caption{CLIP reward over retained samples and sampling cost.}
\label{tab:clip_guidance_full}\\
\toprule
Method & Mean CLIP reward $\uparrow$ & Sampling time (h) & Mean sweep time (s) \\
\midrule
\endfirsthead
\multicolumn{4}{c}{\tablename\ \thetable{} (continued)} \\
\toprule
Method & Mean CLIP reward $\uparrow$ & Sampling time (h) & Mean sweep time (s) \\
\midrule
\endhead
\midrule
\multicolumn{4}{r}{Continued on next page} \\
\endfoot
\bottomrule
\endlastfoot
\multicolumn{4}{@{}p{0.94\textwidth}@{}}{\textbf{Golden retriever}: \textit{``A photograph of a golden retriever running through fallen autumn leaves."}} \\*
Best-of-K & 0.2864 & 2.49 & --\\*
SiT+SPT+pCN & 0.3122 & 10.51 & 370.89 \\*
iMF+SPT+pCN & 0.3809 & 2.40 & 5.75 \\*
iMF+SPT+hybrid & \textbf{0.3824} & 10.95 & 26.28 \\
\midrule
\multicolumn{4}{@{}p{0.94\textwidth}@{}}{\textbf{Golden retriever}: \textit{``A beautiful professional photograph of a golden retriever in a sunlit meadow, sharp focus, natural colors."}} \\*
Best-of-K & 0.3073 & 2.48 & --\\*
SiT+SPT+pCN & 0.3368 & 10.79 & 373.28 \\*
iMF+SPT+pCN & 0.3741 & 2.43 & 5.83 \\*
iMF+SPT+hybrid & \textbf{0.3754} & 11.26 & 27.02 \\
\midrule
\multicolumn{4}{@{}p{0.94\textwidth}@{}}{\textbf{Golden retriever}: \textit{``A colorful painting of a golden retriever standing in falling snow."}} \\*
Best-of-K & 0.2885 & 2.45 & -- \\*
SiT+SPT+pCN & 0.3355 & 10.96 & 372.17 \\*
iMF+SPT+pCN & 0.3945 & 2.50 & 6.01 \\*
iMF+SPT+hybrid & \textbf{0.4011} & 11.50 & 27.58 \\
\midrule
\multicolumn{4}{@{}p{0.94\textwidth}@{}}{\textbf{Golden retriever}: \textit{``A photograph of a golden retriever swimming in a clear blue lake."}} \\*
Best-of-K & 0.3064 & 2.45 & -- \\*
SiT+SPT+pCN & 0.3675 & 10.69 & 369.97 \\*
iMF+SPT+pCN & \textbf{0.3847} & 2.45 & 5.88 \\*
iMF+SPT+hybrid & 0.3845 & 11.16 & 26.69 \\
\midrule
\multicolumn{4}{@{}p{0.94\textwidth}@{}}{\textbf{Bald eagle}: \textit{``A photograph of a bald eagle perched on a branch surrounded by red and golden autumn leaves."}} \\*
Best-of-K & 0.3086 & 2.50 & -- \\*
SiT+SPT+pCN & 0.3156 & 10.73 & 371.46 \\*
iMF+SPT+pCN & 0.3743 & 2.46 & 5.90 \\*
iMF+SPT+hybrid & \textbf{0.3781} & 11.20 & 26.87 \\
\midrule
\multicolumn{4}{@{}p{0.94\textwidth}@{}}{\textbf{Bald eagle}: \textit{``A photograph of a bald eagle catching a fish from a lake, with splashing water."}} \\*
Best-of-K & 0.3247 & 2.59 & -- \\*
SiT+SPT+pCN & 0.3407 & 10.78 & 376.60 \\*
iMF+SPT+pCN & \textbf{0.3625} & 2.46 & 5.90 \\*
iMF+SPT+hybrid & 0.3596 & 11.17 & 26.81 \\
\midrule
\multicolumn{4}{@{}p{0.94\textwidth}@{}}{\textbf{Bald eagle}: \textit{``A photograph of a bald eagle soaring above snow-capped mountains, wings spread wide."}} \\*
Best-of-K & 0.3278 & 2.49 & -- \\*
SiT+SPT+pCN & 0.3481 & 10.70 & 370.27 \\*
iMF+SPT+pCN & \textbf{0.3936} & 2.44 & 5.85 \\*
iMF+SPT+hybrid & 0.3914 & 11.15 & 26.76 \\
\midrule
\multicolumn{4}{@{}p{0.94\textwidth}@{}}{\textbf{Bald eagle}: \textit{``A colorful painting of a bald eagle perched on a pine branch in falling snow."}} \\*
Best-of-K & 0.2984 & 2.58 & -- \\*
SiT+SPT+pCN & 0.3139 & 10.74 & 371.58 \\*
iMF+SPT+pCN & \textbf{0.3727} & 2.41 & 5.77 \\*
iMF+SPT+hybrid & 0.3713 & 11.16 & 26.78 \\
\midrule
\multicolumn{4}{@{}p{0.94\textwidth}@{}}{\textbf{Mink}: \textit{``A photograph of a mink beside a clear forest stream, surrounded by green moss."}} \\*
Best-of-K & 0.3629 & 2.75 & -- \\*
SiT+SPT+pCN & 0.3946 & 10.72 & 371.04 \\*
iMF+SPT+pCN & \textbf{0.4190} & 2.45 & 5.89 \\*
iMF+SPT+hybrid & 0.4112 & 11.12 & 26.70 \\
\midrule
\multicolumn{4}{@{}p{0.94\textwidth}@{}}{\textbf{Mink}: \textit{``A colorful painting of a mink standing in falling snow."}} \\*
Best-of-K & 0.3160 & 2.49 & -- \\*
SiT+SPT+pCN & 0.3487 & 10.47 & 369.64 \\*
iMF+SPT+pCN & \textbf{0.4186} & 2.36 & 5.66 \\*
iMF+SPT+hybrid & 0.4162 & 10.99 & 26.37 \\
\midrule
\multicolumn{4}{@{}p{0.94\textwidth}@{}}{\textbf{Castle}: \textit{``A colorful painting of a medieval castle covered in snow."}} \\*
Best-of-K & 0.2489 & 2.48 & -- \\*
SiT+SPT+pCN & 0.2797 & 10.53 & 371.82 \\*
iMF+SPT+pCN & \textbf{0.3618} & 2.42 & 5.79 \\*
iMF+SPT+hybrid & 0.3569 & 10.92 & 26.21 \\
\midrule
\multicolumn{4}{@{}p{0.94\textwidth}@{}}{\textbf{Castle}: \textit{``A photograph of a medieval castle on a hill at sunrise, surrounded by green countryside."}} \\*
Best-of-K & 0.3023 & 2.46 & -- \\*
SiT+SPT+pCN & 0.3476 & 10.96 & 372.13 \\*
iMF+SPT+pCN & 0.3820 & 2.52 & 6.05 \\*
iMF+SPT+hybrid & \textbf{0.3849} & 11.40 & 27.35 \\
\midrule
\multicolumn{4}{@{}p{0.94\textwidth}@{}}{\textbf{Aircraft carrier}: \textit{``A photograph of an aircraft carrier sailing across the ocean at sunset."}} \\*
Best-of-K & 0.2879 & 2.49 & -- \\*
SiT+SPT+pCN & 0.3286 & 10.38 & 373.72 \\*
iMF+SPT+pCN & \textbf{0.3667} & 2.42 & 5.80 \\*
iMF+SPT+hybrid & 0.3659 & 10.90 & 26.16 \\
\midrule
\multicolumn{4}{@{}p{0.94\textwidth}@{}}{\textbf{Aircraft carrier}: \textit{``A watercolor painting of an aircraft carrier on a calm blue sea."}} \\*
Best-of-K & 0.2956 & 2.52 & -- \\*
SiT+SPT+pCN & 0.3055 & 10.63 & 367.95 \\*
iMF+SPT+pCN & 0.3736 & 2.49 & 5.97 \\*
iMF+SPT+hybrid & \textbf{0.3774} & 11.07 & 26.55 \\
\end{longtable}
\endgroup

\begin{figure}[!htb]
    \centering
    % Left: golden retriever
    \begin{minipage}[t]{0.48\linewidth}
        \vspace{0pt}
        \captionsetup{type=figure}
        \centering

        \begin{subfigure}{\linewidth}
            \centering
            \includegraphics[width=\linewidth]{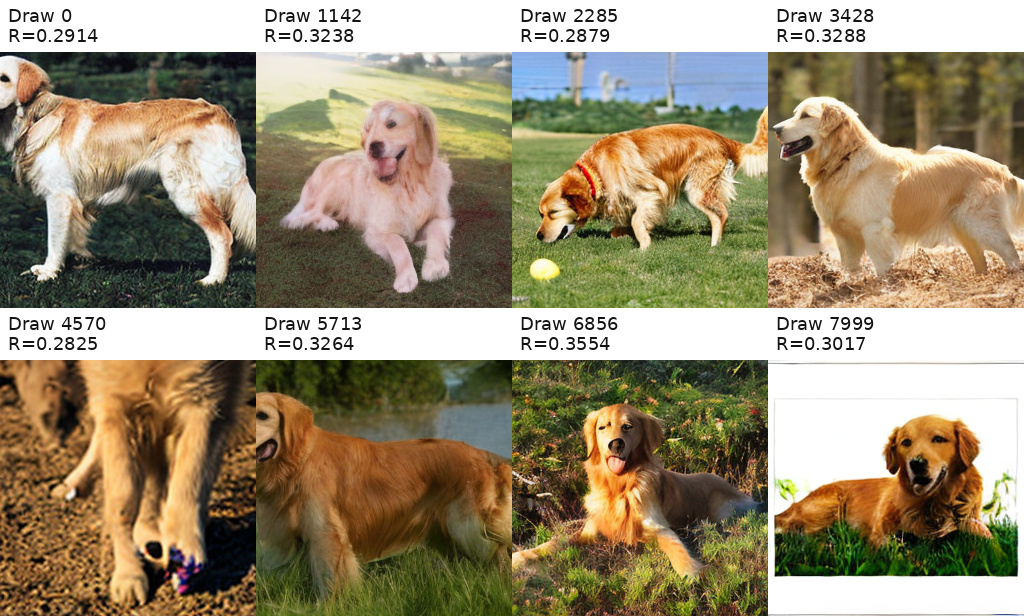}
            \caption{best-of-K.}
            \label{fig:reward_final_golden_retriever_meadow_tau200_r8_l8_v1:imf_best_of_12}
        \end{subfigure}
        \par\medskip

        \begin{subfigure}{\linewidth}
            \centering
            \includegraphics[width=\linewidth]{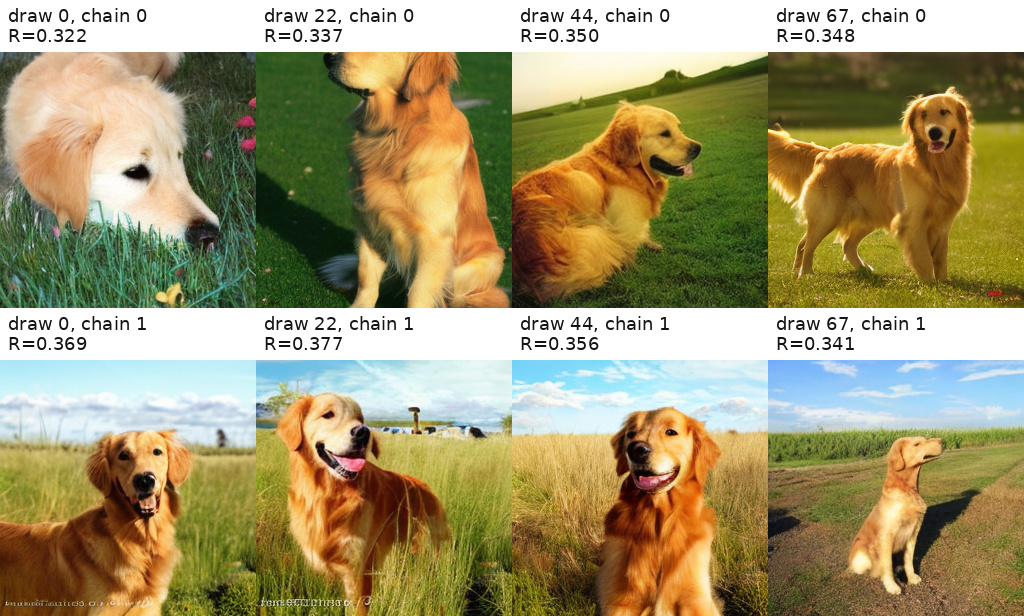}
            \caption{SiT + SPT + pCN.}
            \label{fig:reward_final_golden_retriever_meadow_tau200_r8_l8_v1:sit_spt_pcn}
        \end{subfigure}
        \par\medskip

        \begin{subfigure}{\linewidth}
            \centering
            \includegraphics[width=\linewidth]{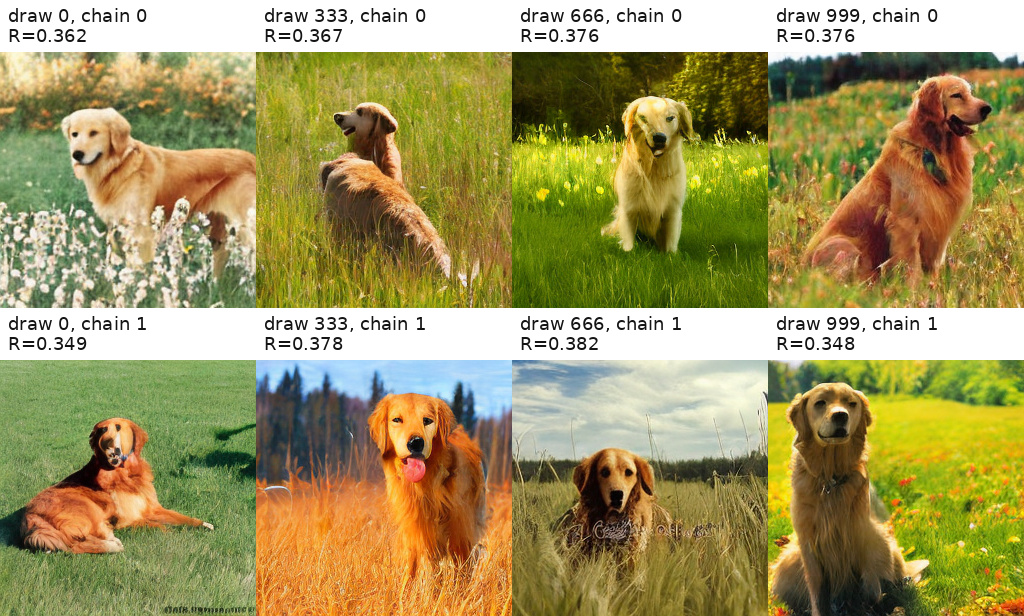}
            \caption{iMF + SPT + pCN.}
            \label{fig:reward_final_golden_retriever_meadow_tau200_r8_l8_v1:imf_spt_pcn}
        \end{subfigure}
        \par\medskip

        \begin{subfigure}{\linewidth}
            \centering
            \includegraphics[width=\linewidth]{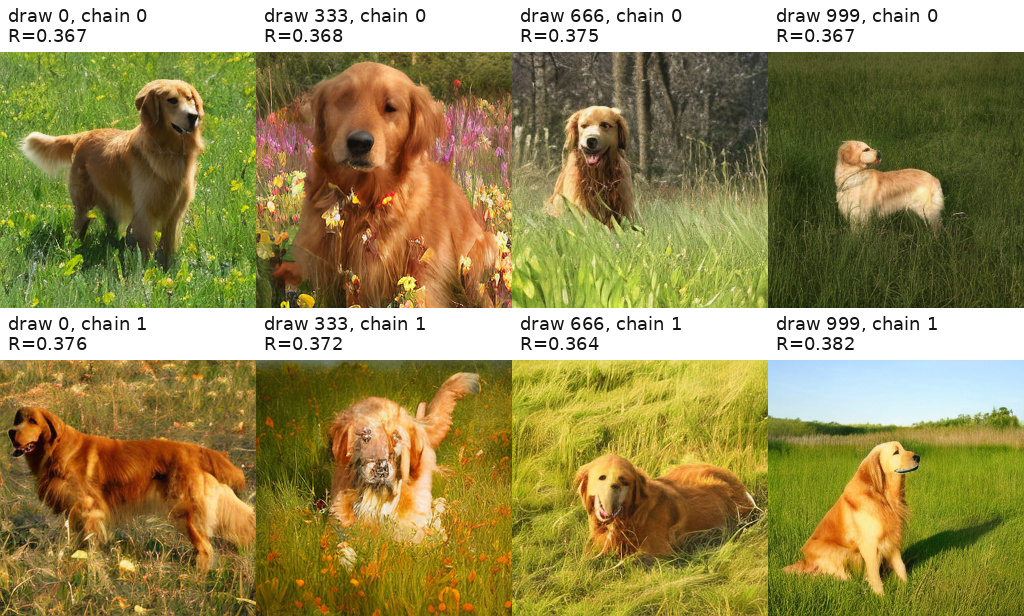}
            \caption{iMF + SPT + hybrid.}
            \label{fig:reward_final_golden_retriever_meadow_tau200_r8_l8_v1:imf_spt_hybrid}
        \end{subfigure}

        \caption{Class: golden retriever, prompt: \textit{``A beautiful professional photograph of a golden retriever in a sunlit meadow, sharp focus, natural colors."}}
        \label{fig:goldenmeadow}
    \end{minipage}%
    \hfill
    % Right: bald eagle
    \begin{minipage}[t]{0.48\linewidth}
        \vspace{0pt}
        \captionsetup{type=figure}
        \centering

        \begin{subfigure}{\linewidth}
            \centering
            \includegraphics[width=\linewidth]{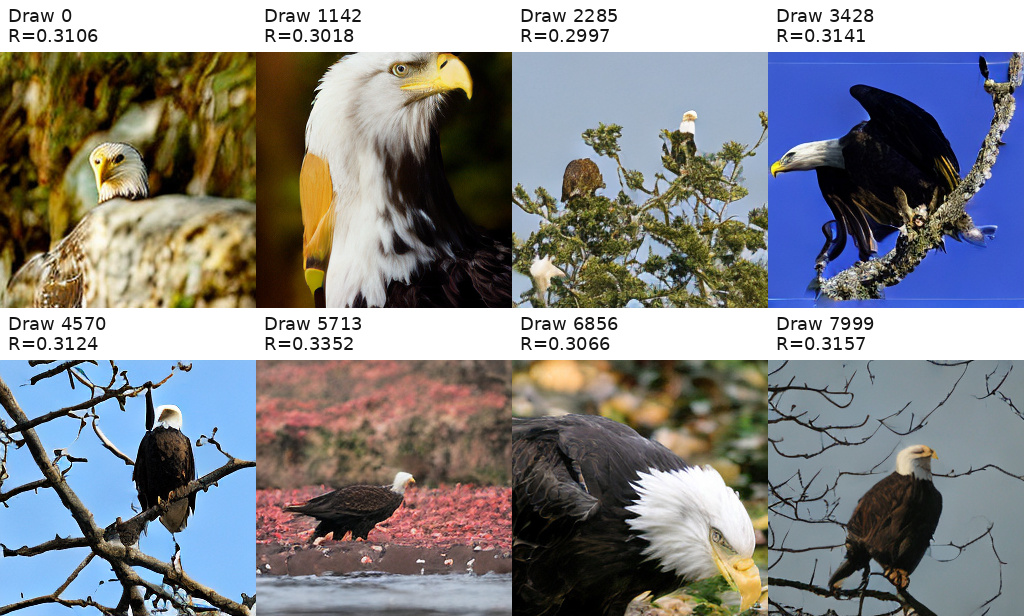}
            \caption{best-of-K.}
            \label{fig:reward_final_bald_eagle_autumn_perched_tau200_r8_l8_v1:imf_best_of_12}
        \end{subfigure}
        \par\medskip

        \begin{subfigure}{\linewidth}
            \centering
            \includegraphics[width=\linewidth]{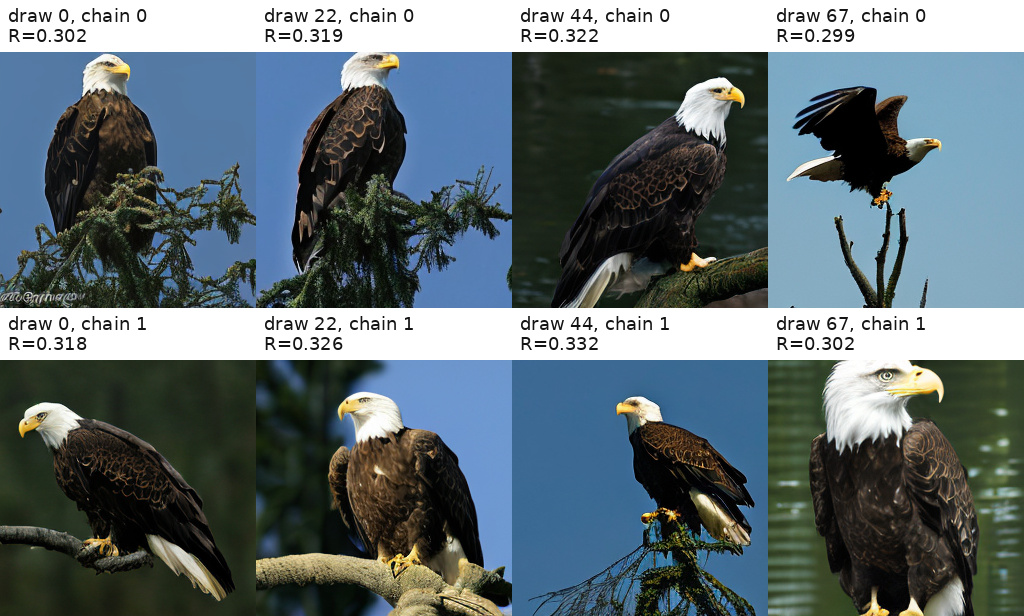}
            \caption{SiT + SPT + pCN.}
            \label{fig:reward_final_bald_eagle_autumn_perched_tau200_r8_l8_v1:sit_spt_pcn}
        \end{subfigure}
        \par\medskip

        \begin{subfigure}{\linewidth}
            \centering
            \includegraphics[width=\linewidth]{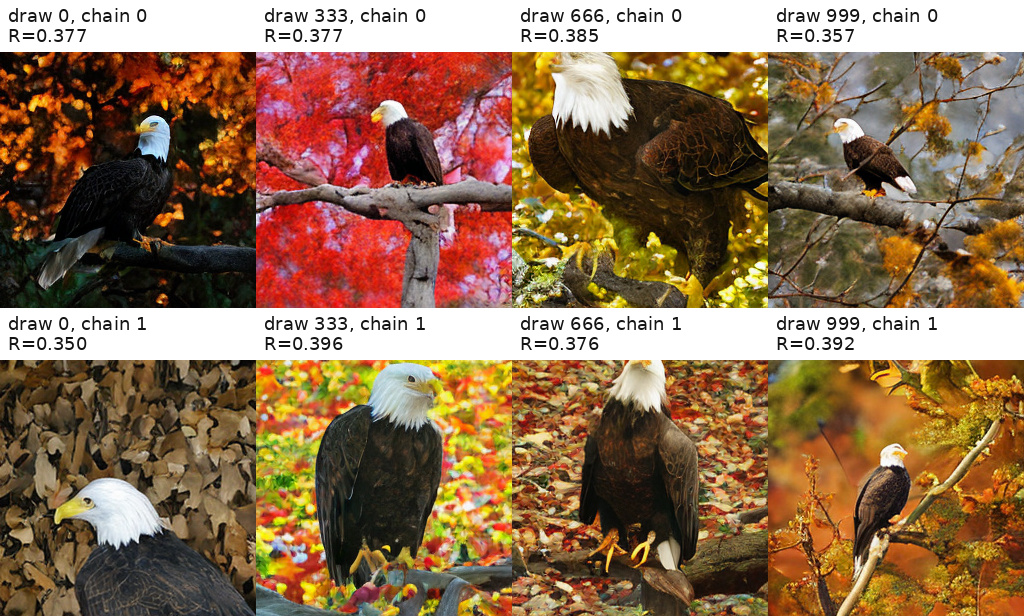}
            \caption{iMF + SPT + pCN.}
            \label{fig:reward_final_bald_eagle_autumn_perched_tau200_r8_l8_v1:imf_spt_pcn}
        \end{subfigure}
        \par\medskip

        \begin{subfigure}{\linewidth}
            \centering
            \includegraphics[width=\linewidth]{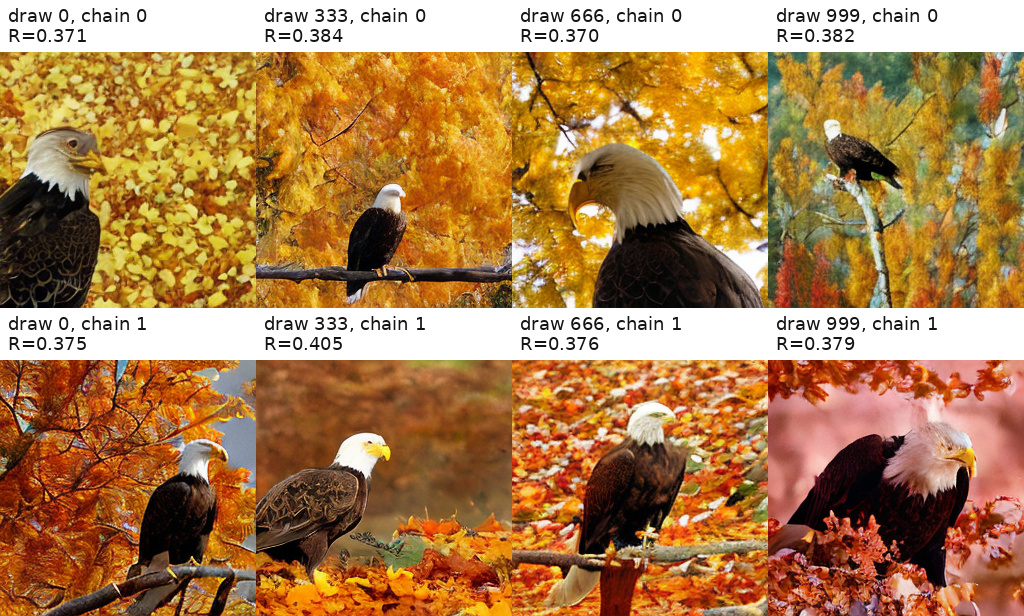}
            \caption{iMF + SPT + hybrid.}
            \label{fig:reward_final_bald_eagle_autumn_perched_tau200_r8_l8_v1:imf_spt_hybrid}
        \end{subfigure}

        \caption{Class: bald eagle, prompt: \textit{``A photograph of a bald eagle perched on a branch surrounded by red and golden autumn leaves."}}
        \label{fig:eagleautum}
    \end{minipage}
\end{figure}

\begin{figure}[!htb]
    \centering
    % Left: golden retriever snow painting
    \begin{minipage}[t]{0.48\linewidth}
        \vspace{0pt}
        \captionsetup{type=figure}
        \centering

        \begin{subfigure}{\linewidth}
            \centering
            \includegraphics[width=\linewidth]{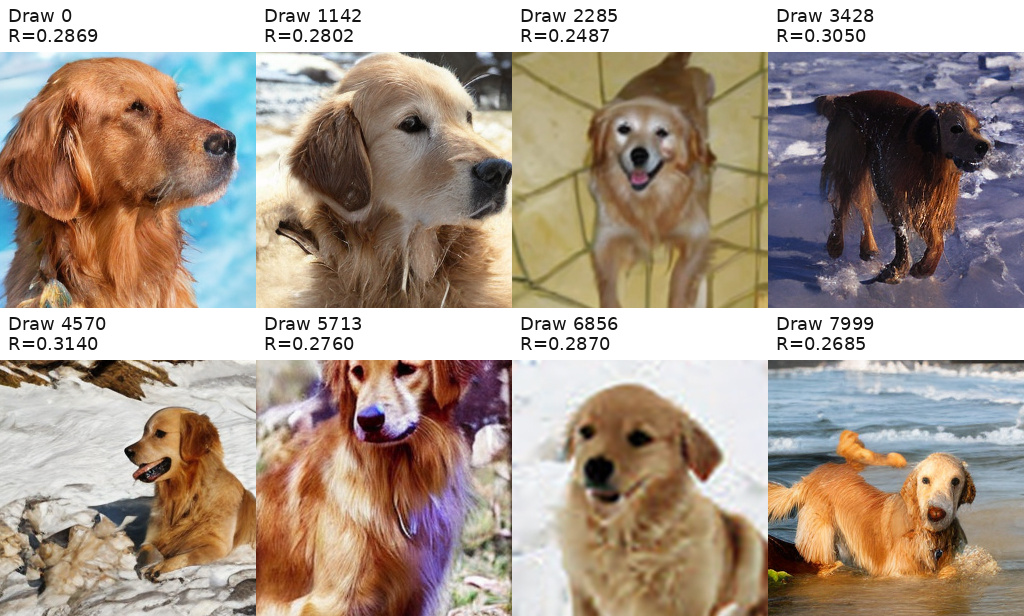}
            \caption{best-of-K.}
            \label{fig:reward_final_golden_retriever_snow_painting_tau200_r8_l8_v1:imf_best_of_12}
        \end{subfigure}
        \par\medskip

        \begin{subfigure}{\linewidth}
            \centering
            \includegraphics[width=\linewidth]{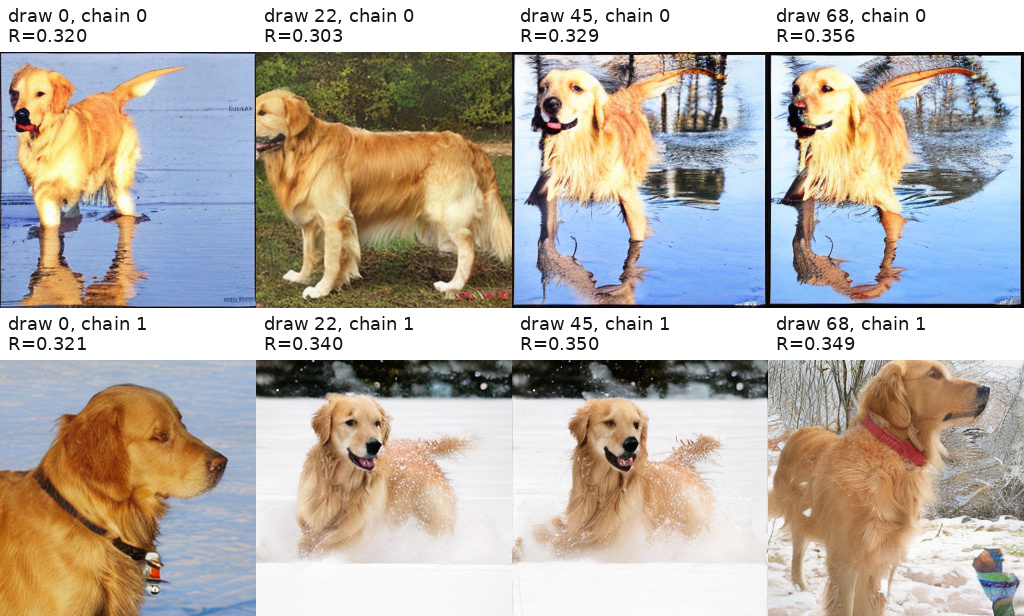}
            \caption{SiT + SPT + pCN.}
            \label{fig:reward_final_golden_retriever_snow_painting_tau200_r8_l8_v1:sit_spt_pcn}
        \end{subfigure}
        \par\medskip

        \begin{subfigure}{\linewidth}
            \centering
            \includegraphics[width=\linewidth]{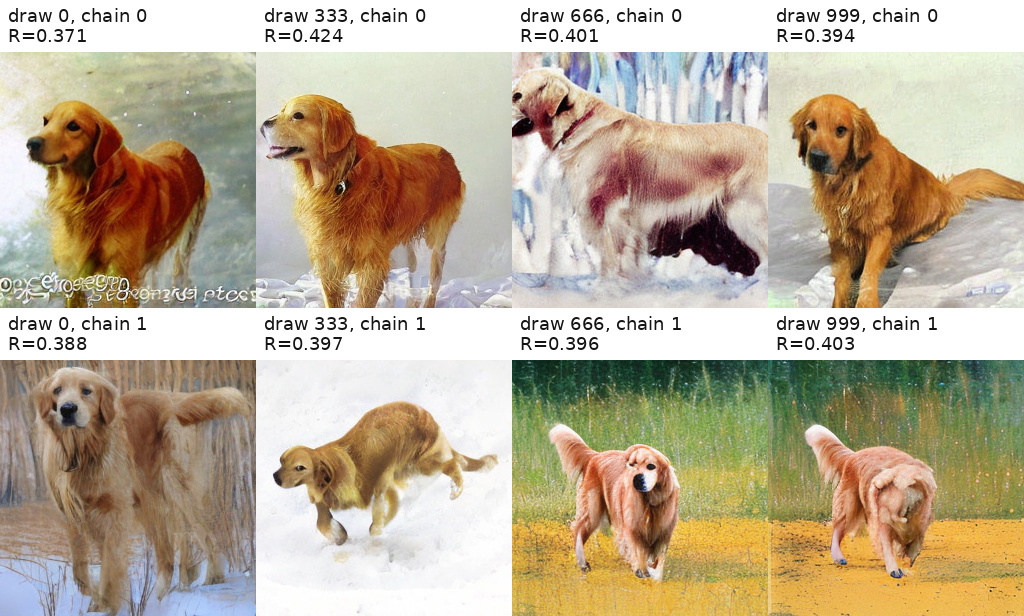}
            \caption{iMF + SPT + pCN.}
            \label{fig:reward_final_golden_retriever_snow_painting_tau200_r8_l8_v1:imf_spt_pcn}
        \end{subfigure}
        \par\medskip

        \begin{subfigure}{\linewidth}
            \centering
            \includegraphics[width=\linewidth]{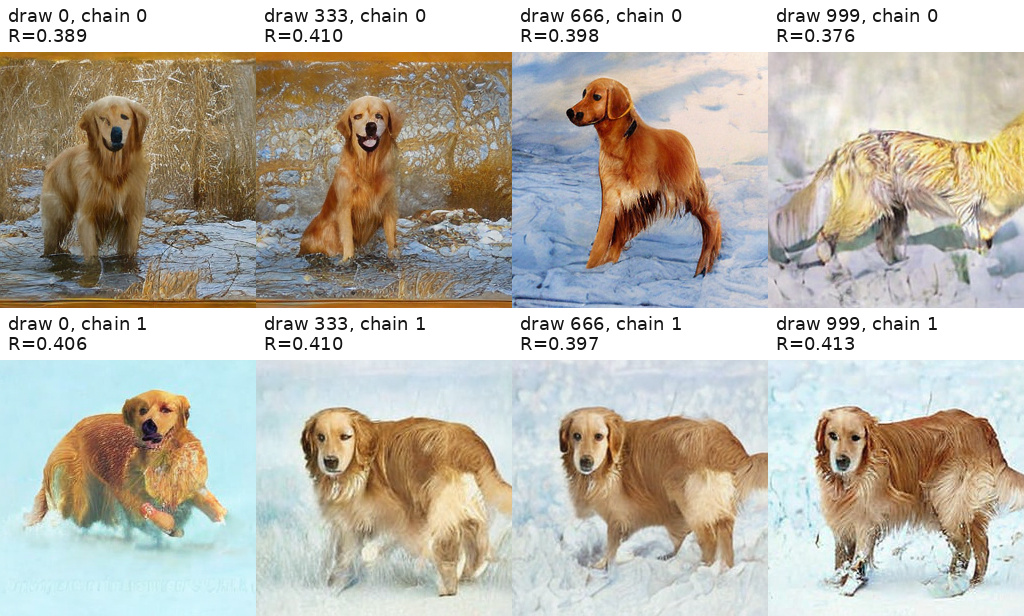}
            \caption{iMF + SPT + hybrid.}
            \label{fig:reward_final_golden_retriever_snow_painting_tau200_r8_l8_v1:imf_spt_hybrid}
        \end{subfigure}

        \caption{Class: golden retriever, prompt: \textit{``A colorful painting of a golden retriever standing in falling snow.''}}
        \label{fig:reward_final_golden_retriever_snow_painting_tau200_r8_l8_v1}
    \end{minipage}%
    \hfill
    % Right: bald eagle snow painting
    \begin{minipage}[t]{0.48\linewidth}
        \vspace{0pt}
        \captionsetup{type=figure}
        \centering

        \begin{subfigure}{\linewidth}
            \centering
            \includegraphics[width=\linewidth]{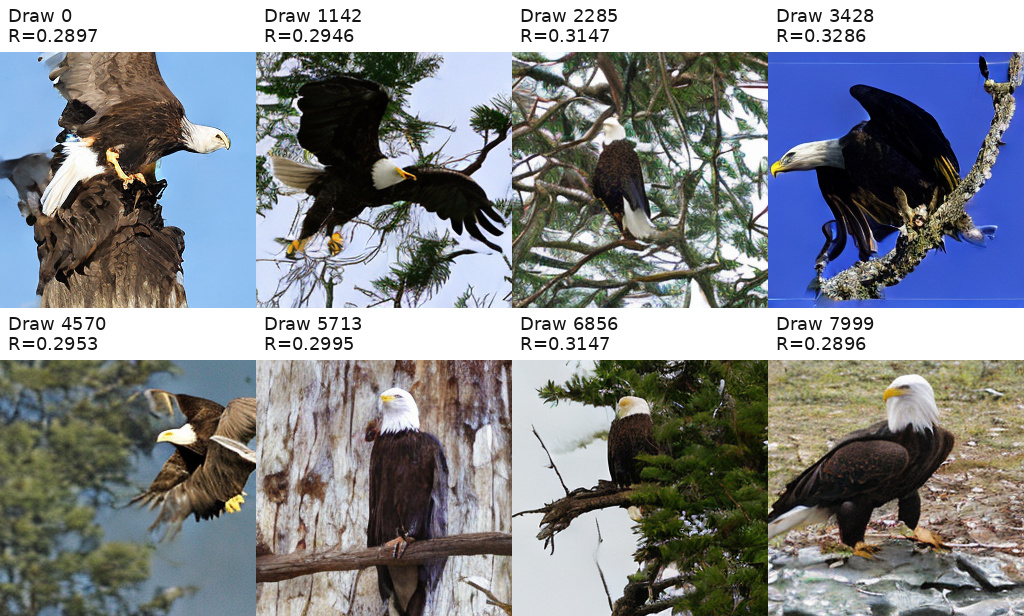}
            \caption{best-of-K.}
            \label{fig:reward_final_bald_eagle_snow_painting_tau200_r8_l8_v1:imf_best_of_12}
        \end{subfigure}
        \par\medskip

        \begin{subfigure}{\linewidth}
            \centering
            \includegraphics[width=\linewidth]{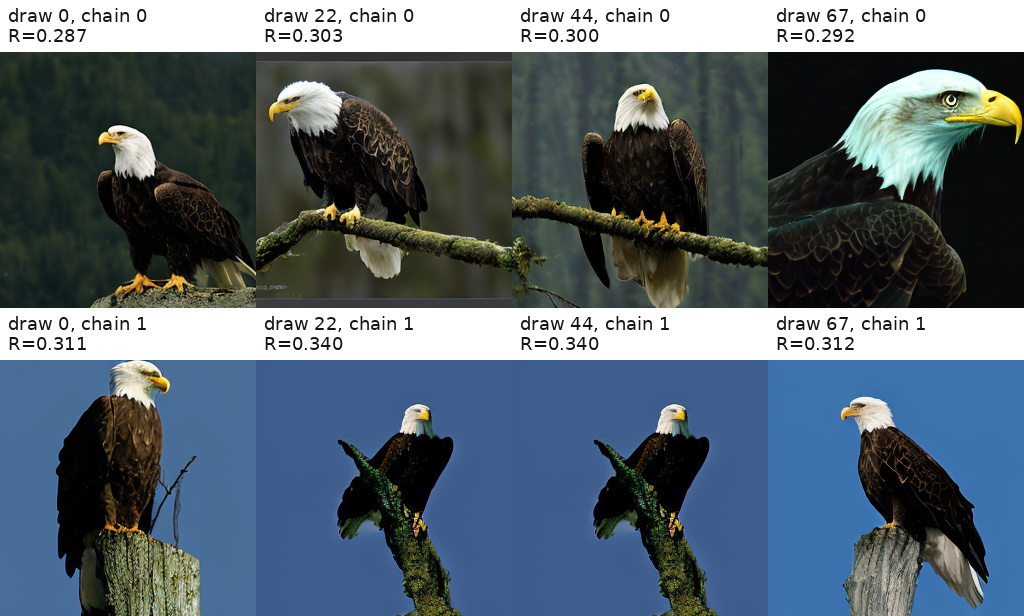}
            \caption{SiT + SPT + pCN.}
            \label{fig:reward_final_bald_eagle_snow_painting_tau200_r8_l8_v1:sit_spt_pcn}
        \end{subfigure}
        \par\medskip

        \begin{subfigure}{\linewidth}
            \centering
            \includegraphics[width=\linewidth]{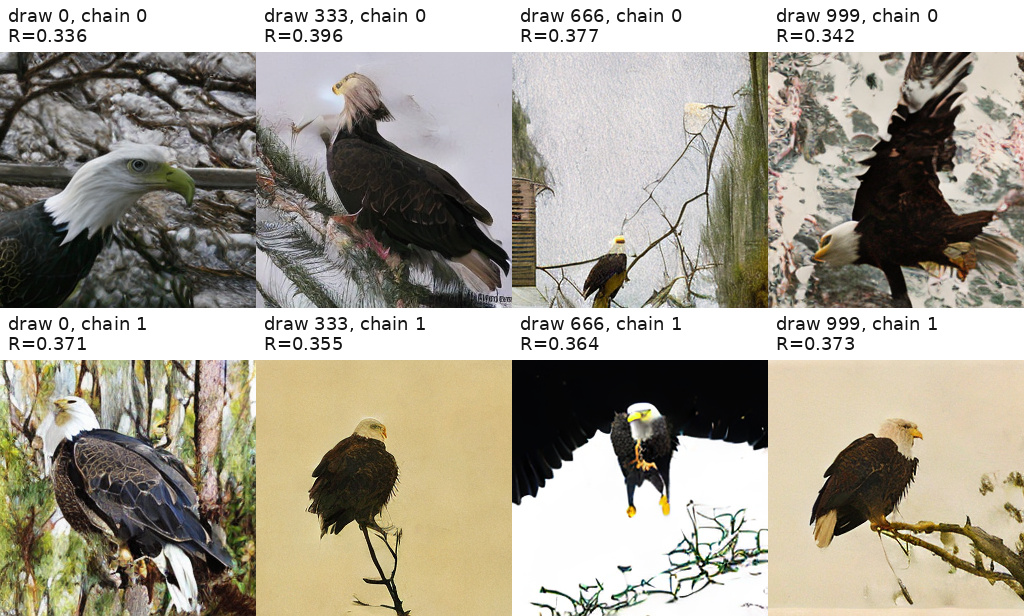}
            \caption{iMF + SPT + pCN.}
            \label{fig:reward_final_bald_eagle_snow_painting_tau200_r8_l8_v1:imf_spt_pcn}
        \end{subfigure}
        \par\medskip

        \begin{subfigure}{\linewidth}
            \centering
            \includegraphics[width=\linewidth]{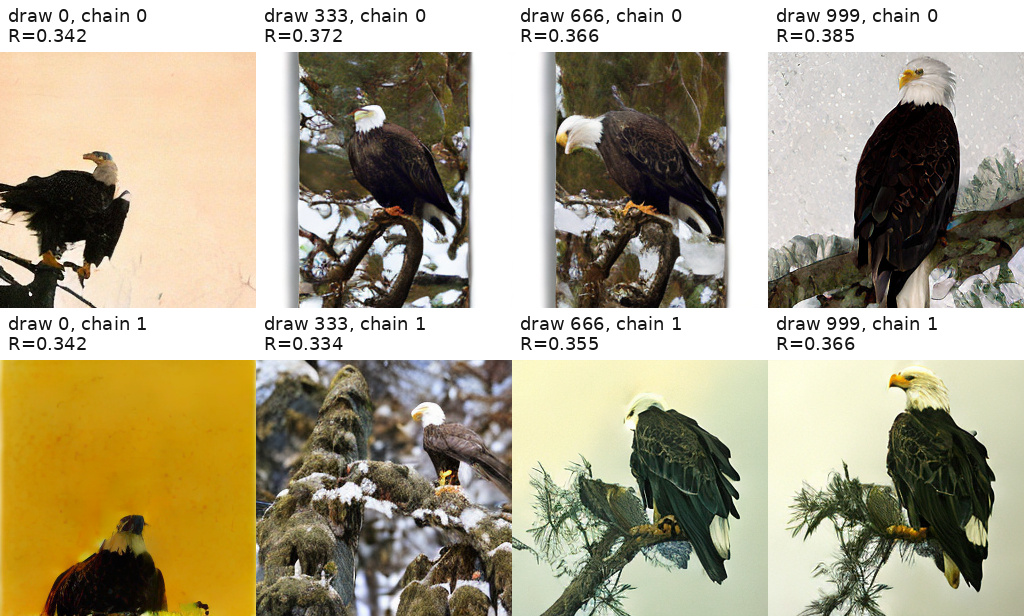}
            \caption{iMF + SPT + hybrid.}
            \label{fig:reward_final_bald_eagle_snow_painting_tau200_r8_l8_v1:imf_spt_hybrid}
        \end{subfigure}

        \caption{Class: bald eagle, prompt: \textit{``A colorful painting of a bald eagle perched on a pine branch in falling snow.''}}
        \label{fig:reward_final_bald_eagle_snow_painting_tau200_r8_l8_v1}
    \end{minipage}
\end{figure}

\begin{figure}[!htb]
    \centering
    % Left: golden retriever swimming
    \begin{minipage}[t]{0.48\linewidth}
        \vspace{0pt}
        \captionsetup{type=figure}
        \centering

        \begin{subfigure}{\linewidth}
            \centering
            \includegraphics[width=\linewidth]{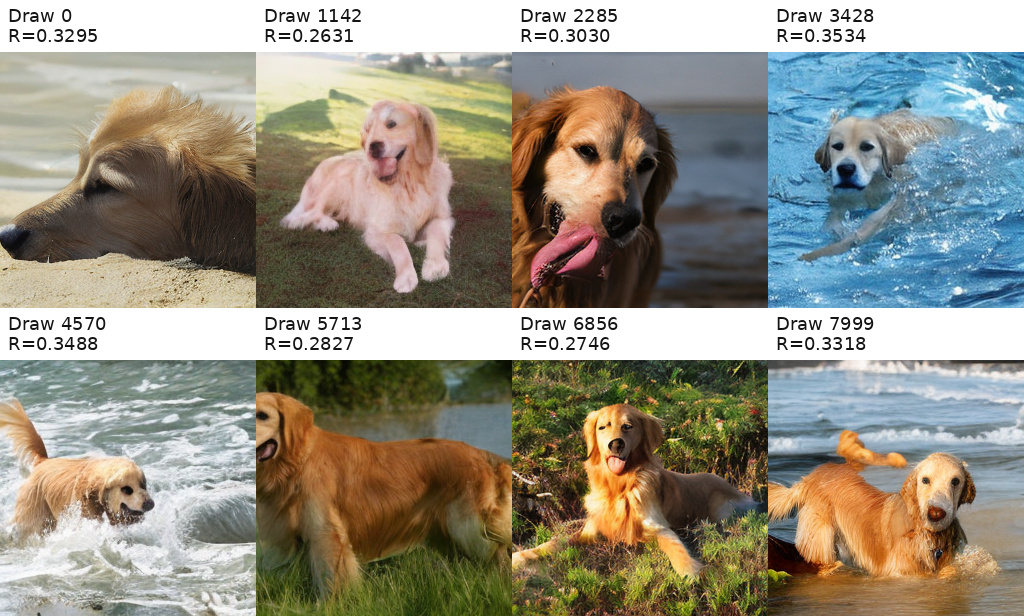}
            \caption{best-of-K.}
            \label{fig:reward_final_golden_retriever_swimming_tau200_r8_l8_v1:imf_best_of_12}
        \end{subfigure}
        \par\medskip

        \begin{subfigure}{\linewidth}
            \centering
            \includegraphics[width=\linewidth]{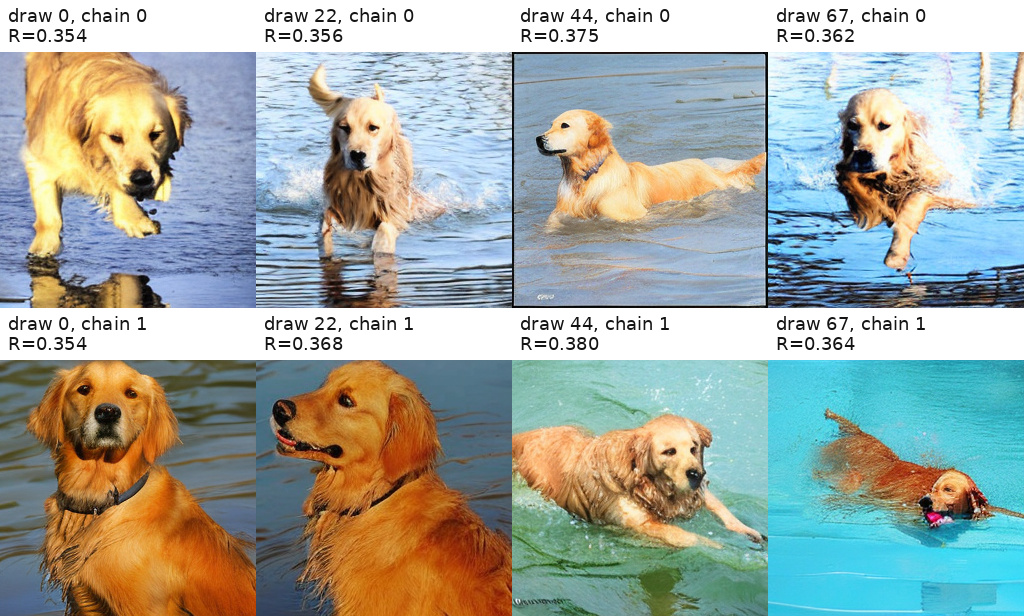}
            \caption{SiT + SPT + pCN.}
            \label{fig:reward_final_golden_retriever_swimming_tau200_r8_l8_v1:sit_spt_pcn}
        \end{subfigure}
        \par\medskip

        \begin{subfigure}{\linewidth}
            \centering
            \includegraphics[width=\linewidth]{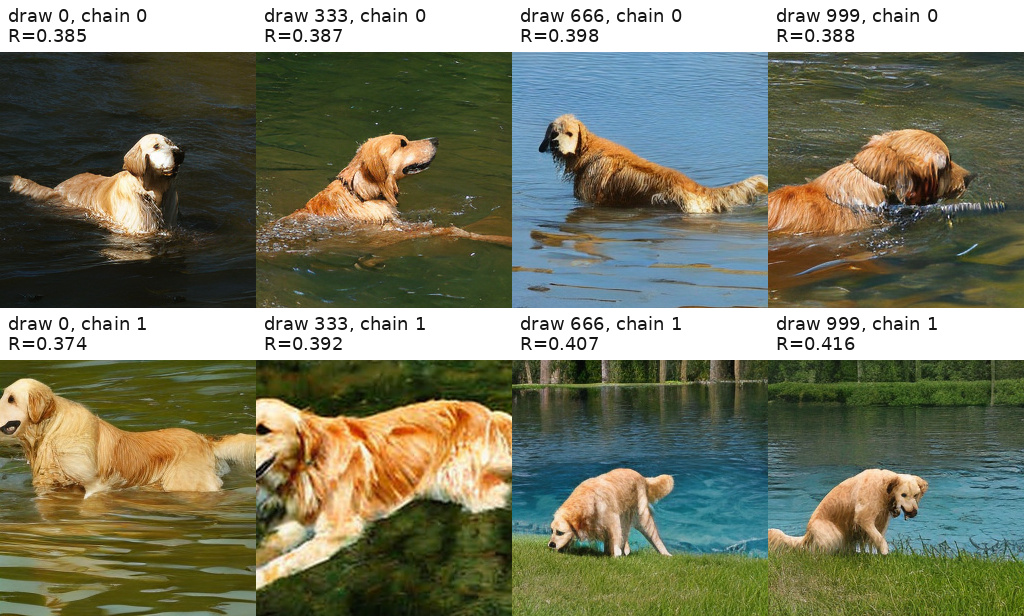}
            \caption{iMF + SPT + pCN.}
            \label{fig:reward_final_golden_retriever_swimming_tau200_r8_l8_v1:imf_spt_pcn}
        \end{subfigure}
        \par\medskip

        \begin{subfigure}{\linewidth}
            \centering
            \includegraphics[width=\linewidth]{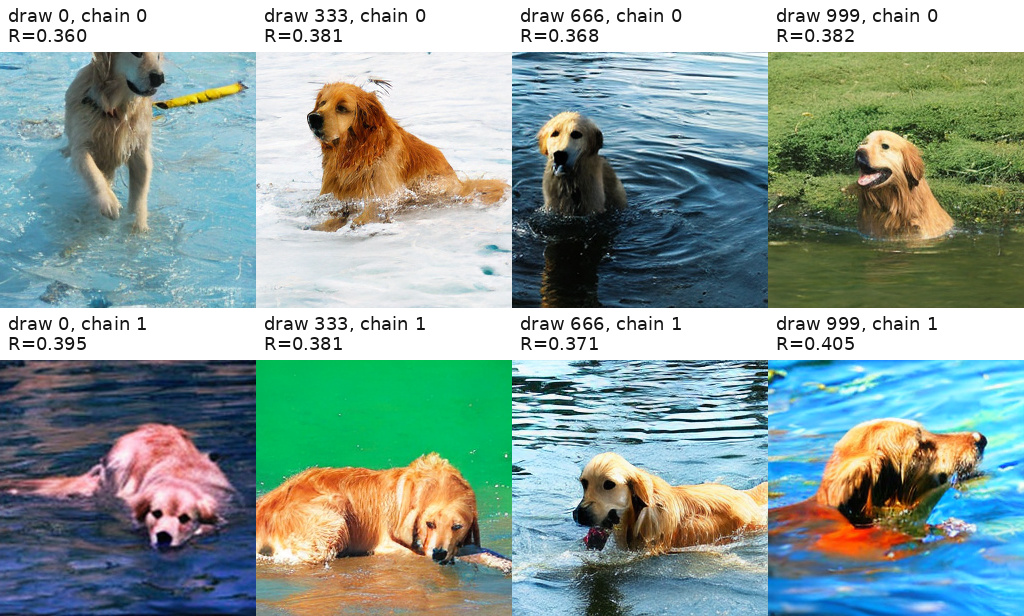}
            \caption{iMF + SPT + hybrid.}
            \label{fig:reward_final_golden_retriever_swimming_tau200_r8_l8_v1:imf_spt_hybrid}
        \end{subfigure}

        \caption{Class: golden retriever, prompt: \textit{``A photograph of a golden retriever swimming in a clear blue lake.''}}
        \label{fig:reward_final_golden_retriever_swimming_tau200_r8_l8_v1}
    \end{minipage}%
    \hfill
    % Right: bald eagle fishing
    \begin{minipage}[t]{0.48\linewidth}
        \vspace{0pt}
        \captionsetup{type=figure}
        \centering

        \begin{subfigure}{\linewidth}
            \centering
            \includegraphics[width=\linewidth]{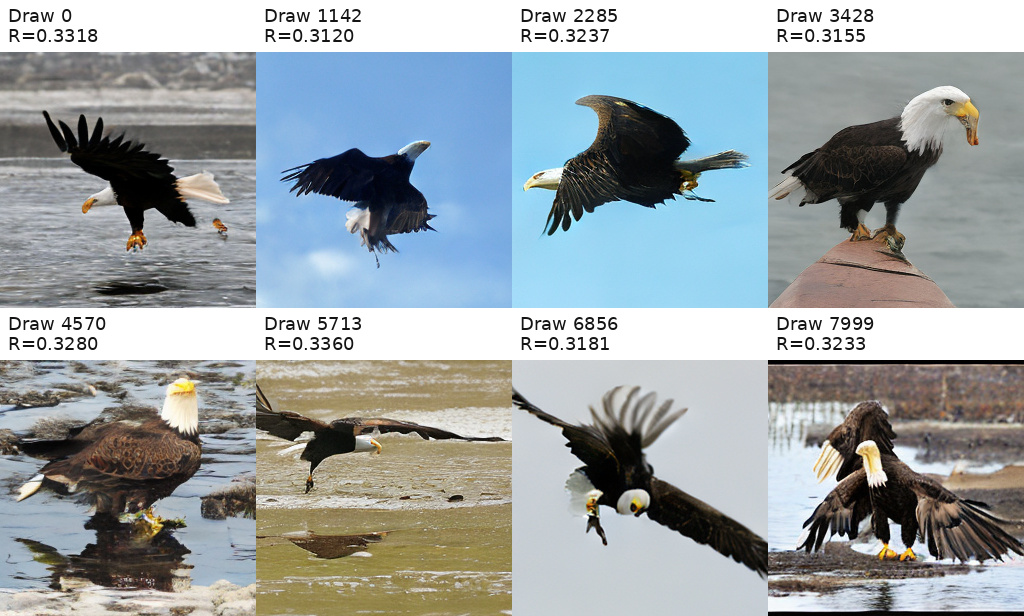}
            \caption{best-of-K.}
            \label{fig:reward_final_bald_eagle_fishing_tau200_r8_l8_v1:imf_best_of_12}
        \end{subfigure}
        \par\medskip

        \begin{subfigure}{\linewidth}
            \centering
            \includegraphics[width=\linewidth]{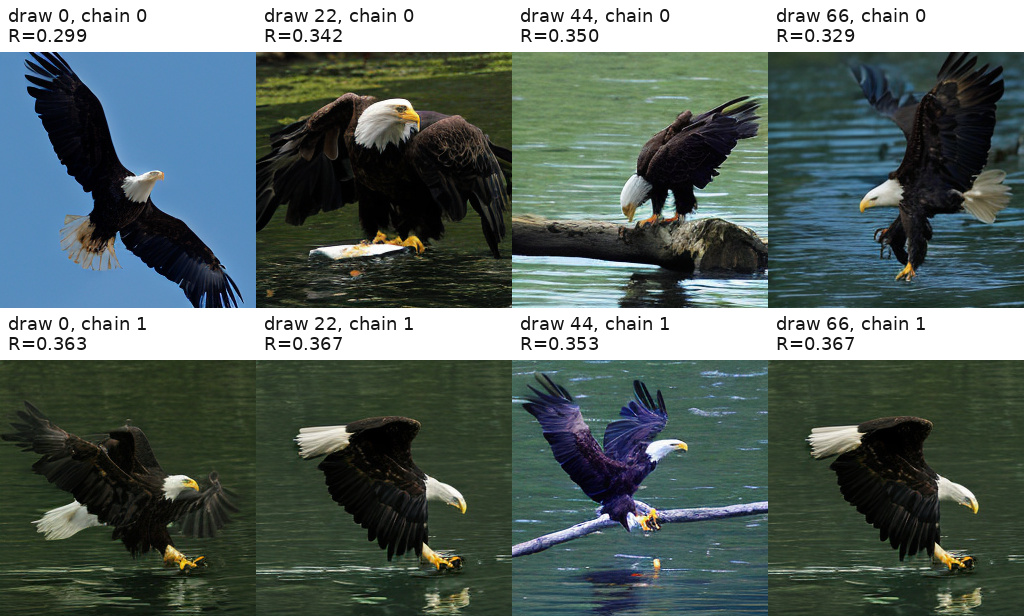}
            \caption{SiT + SPT + pCN.}
            \label{fig:reward_final_bald_eagle_fishing_tau200_r8_l8_v1:sit_spt_pcn}
        \end{subfigure}
        \par\medskip

        \begin{subfigure}{\linewidth}
            \centering
            \includegraphics[width=\linewidth]{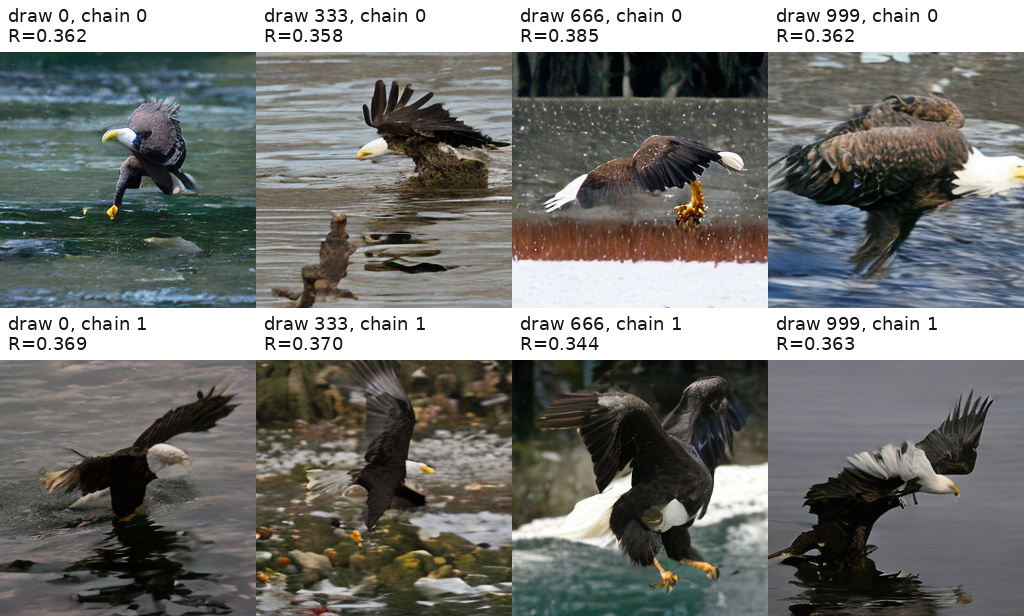}
            \caption{iMF + SPT + pCN.}
            \label{fig:reward_final_bald_eagle_fishing_tau200_r8_l8_v1:imf_spt_pcn}
        \end{subfigure}
        \par\medskip

        \begin{subfigure}{\linewidth}
            \centering
            \includegraphics[width=\linewidth]{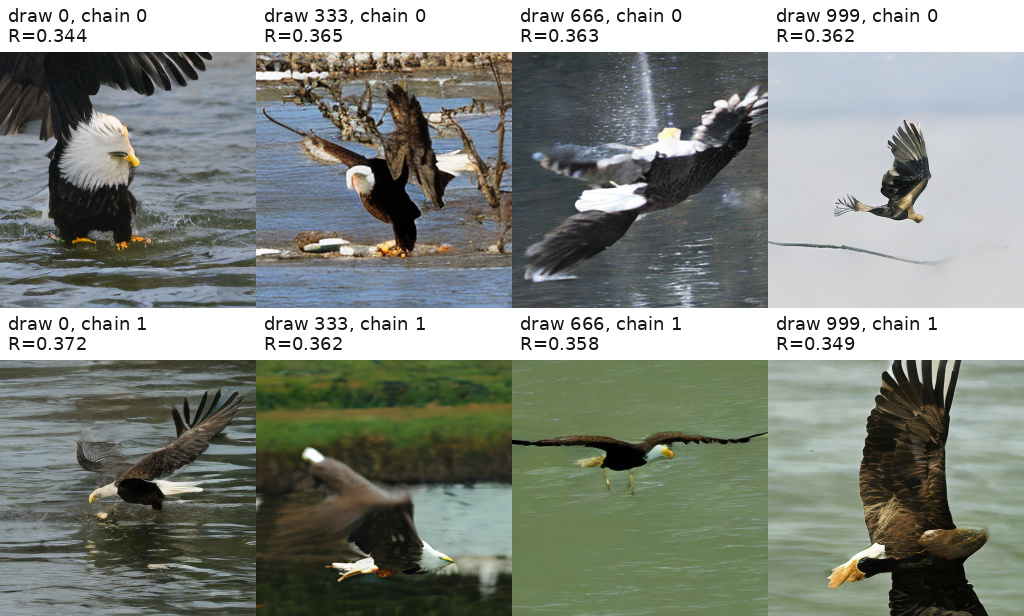}
            \caption{iMF + SPT + hybrid.}
            \label{fig:reward_final_bald_eagle_fishing_tau200_r8_l8_v1:imf_spt_hybrid}
        \end{subfigure}

        \caption{Class: bald eagle, prompt: \textit{``A photograph of a bald eagle catching a fish from a lake, with splashing water.''}}
        \label{fig:reward_final_bald_eagle_fishing_tau200_r8_l8_v1}
    \end{minipage}
\end{figure}

\begin{figure}[!htb]
    \centering
    % Left: mink forest stream
    \begin{minipage}[t]{0.48\linewidth}
        \vspace{0pt}
        \captionsetup{type=figure}
        \centering

        \begin{subfigure}{\linewidth}
            \centering
            \includegraphics[width=\linewidth]{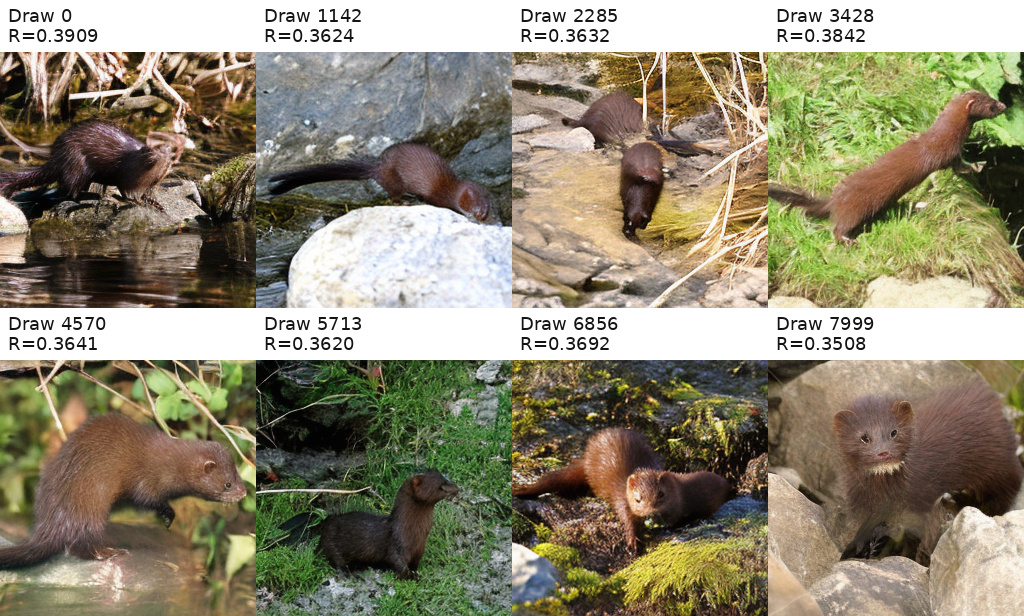}
            \caption{best-of-K.}
            \label{fig:reward_final_mink_forest_stream_tau200_r8_l8_v1:imf_best_of_12}
        \end{subfigure}
        \par\medskip

        \begin{subfigure}{\linewidth}
            \centering
            \includegraphics[width=\linewidth]{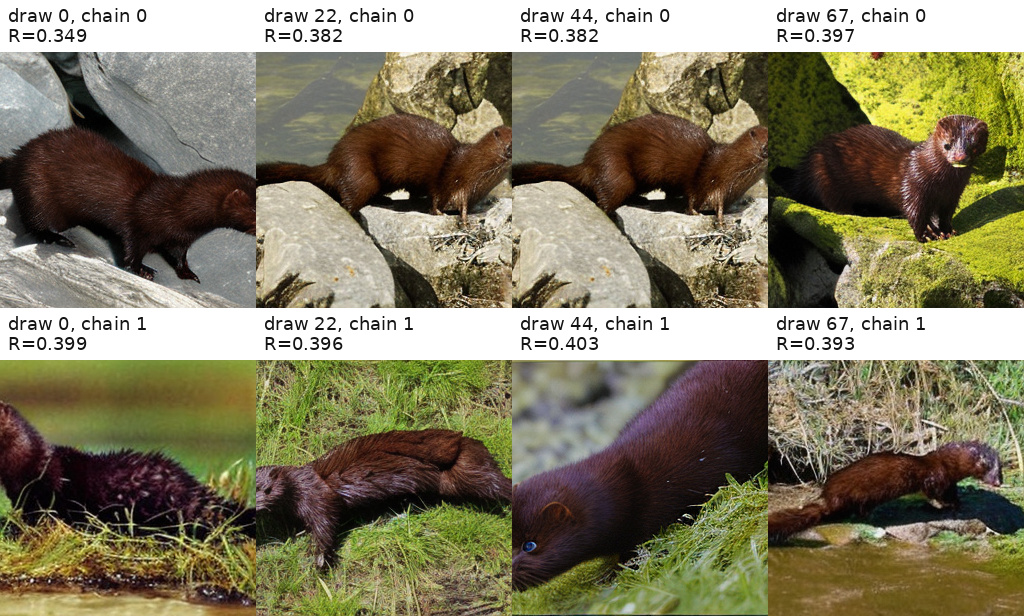}
            \caption{SiT + SPT + pCN.}
            \label{fig:reward_final_mink_forest_stream_tau200_r8_l8_v1:sit_spt_pcn}
        \end{subfigure}
        \par\medskip

        \begin{subfigure}{\linewidth}
            \centering
            \includegraphics[width=\linewidth]{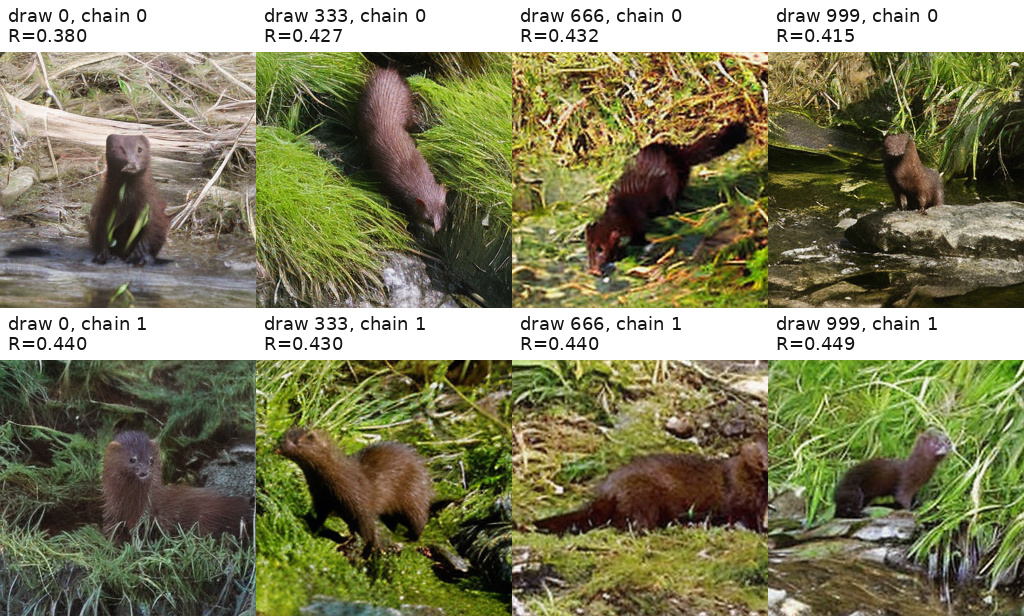}
            \caption{iMF + SPT + pCN.}
            \label{fig:reward_final_mink_forest_stream_tau200_r8_l8_v1:imf_spt_pcn}
        \end{subfigure}
        \par\medskip

        \begin{subfigure}{\linewidth}
            \centering
            \includegraphics[width=\linewidth]{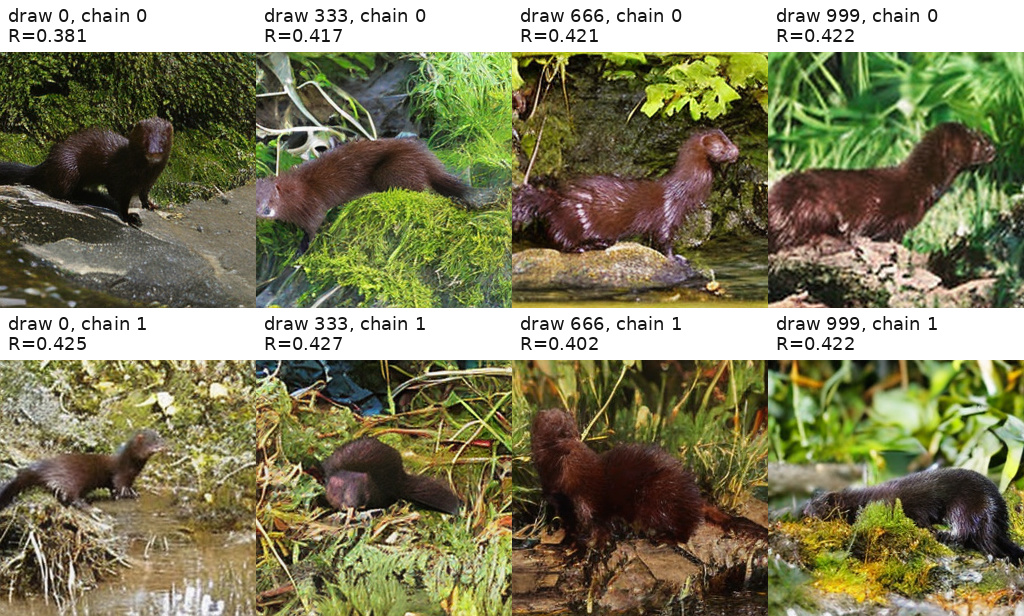}
            \caption{iMF + SPT + hybrid.}
            \label{fig:reward_final_mink_forest_stream_tau200_r8_l8_v1:imf_spt_hybrid}
        \end{subfigure}

        \caption{Class: mink, prompt: \textit{``A photograph of a mink beside a clear forest stream, surrounded by green moss.''}}
        \label{fig:reward_final_mink_forest_stream_tau200_r8_l8_v1}
    \end{minipage}%
    \hfill
    % Right: mink snow painting
    \begin{minipage}[t]{0.48\linewidth}
        \vspace{0pt}
        \captionsetup{type=figure}
        \centering

        \begin{subfigure}{\linewidth}
            \centering
            \includegraphics[width=\linewidth]{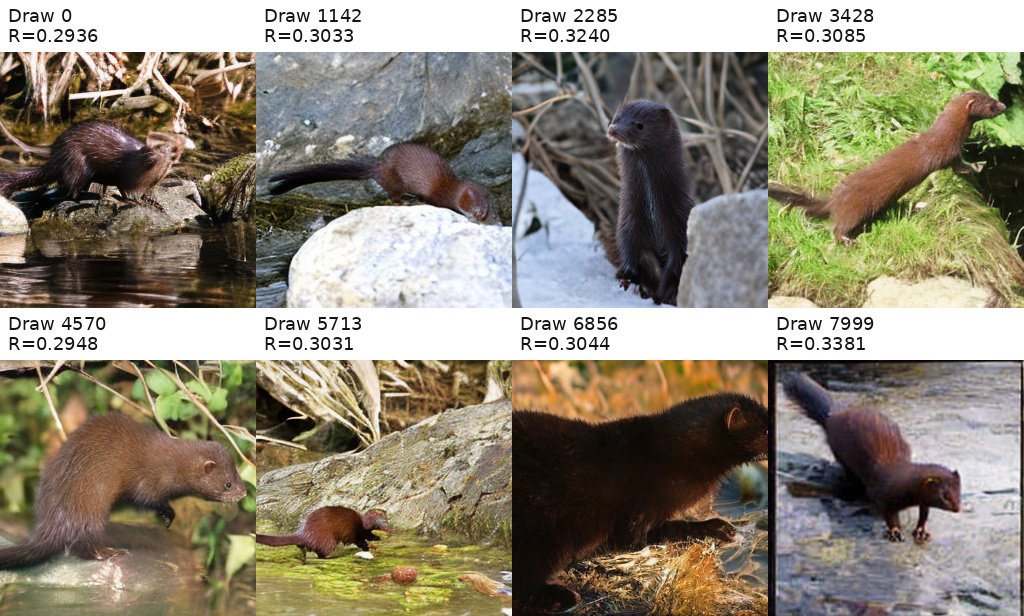}
            \caption{best-of-K.}
            \label{fig:reward_final_mink_snow_painting_tau200_r8_l8_v1:imf_best_of_12}
        \end{subfigure}
        \par\medskip

        \begin{subfigure}{\linewidth}
            \centering
            \includegraphics[width=\linewidth]{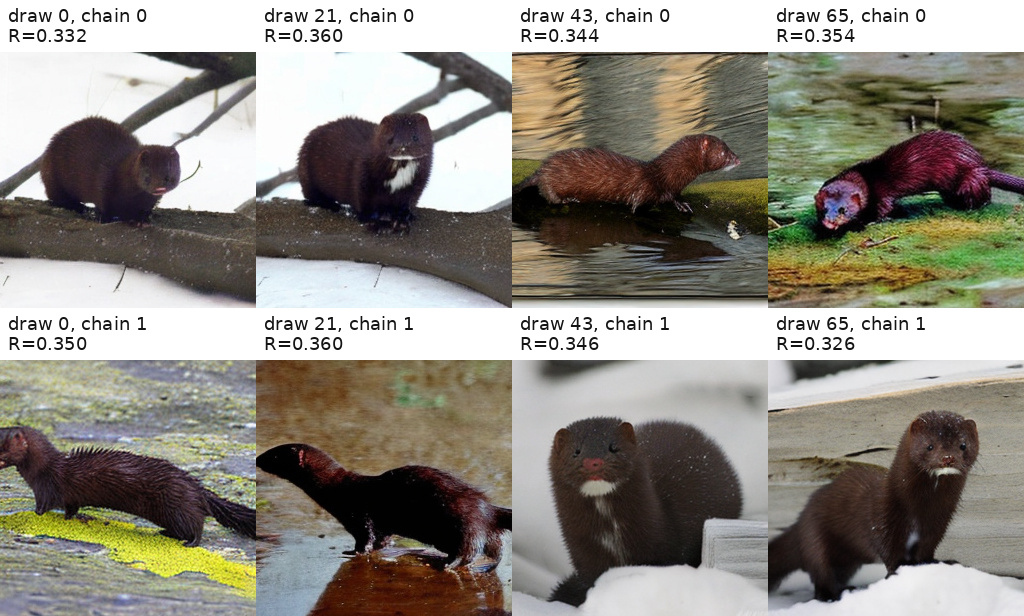}
            \caption{SiT + SPT + pCN.}
            \label{fig:reward_final_mink_snow_painting_tau200_r8_l8_v1:sit_spt_pcn}
        \end{subfigure}
        \par\medskip

        \begin{subfigure}{\linewidth}
            \centering
            \includegraphics[width=\linewidth]{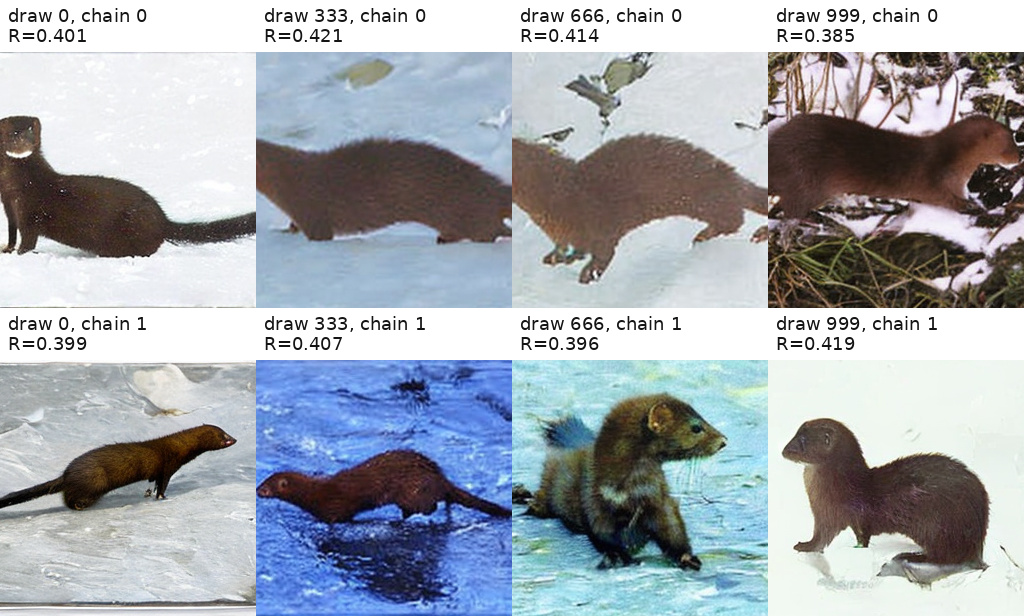}
            \caption{iMF + SPT + pCN.}
            \label{fig:reward_final_mink_snow_painting_tau200_r8_l8_v1:imf_spt_pcn}
        \end{subfigure}
        \par\medskip

        \begin{subfigure}{\linewidth}
            \centering
            \includegraphics[width=\linewidth]{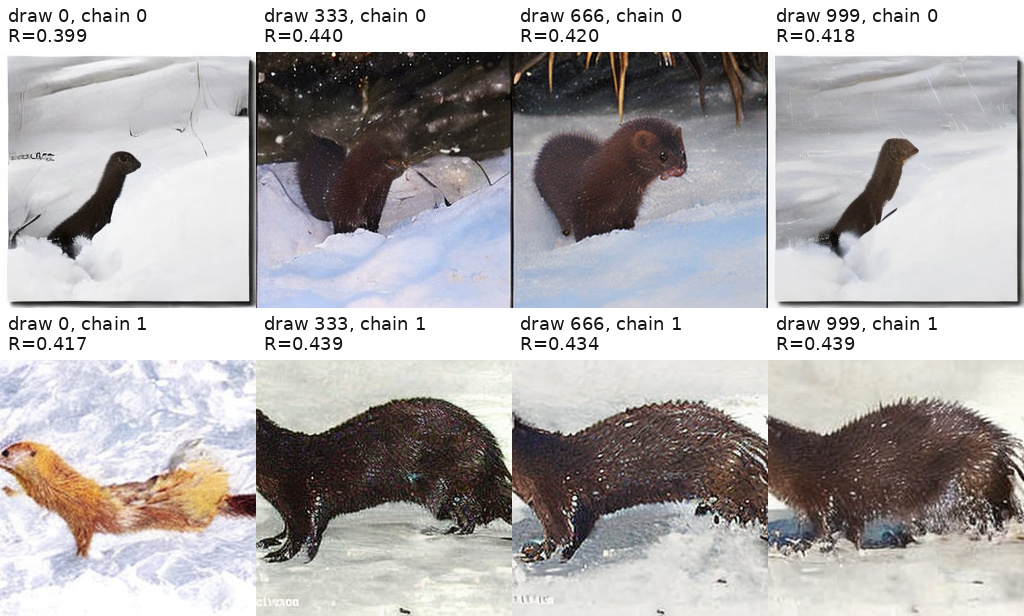}
            \caption{iMF + SPT + hybrid.}
            \label{fig:reward_final_mink_snow_painting_tau200_r8_l8_v1:imf_spt_hybrid}
        \end{subfigure}

        \caption{Class: mink, prompt: \textit{``A colorful painting of a mink standing in falling snow.''}}
        \label{fig:reward_final_mink_snow_painting_tau200_r8_l8_v1}
    \end{minipage}
\end{figure}

\begin{figure}[!htb]
    \centering
    % Left: castle sunrise
    \begin{minipage}[t]{0.48\linewidth}
        \vspace{0pt}
        \captionsetup{type=figure}
        \centering

        \begin{subfigure}{\linewidth}
            \centering
            \includegraphics[width=\linewidth]{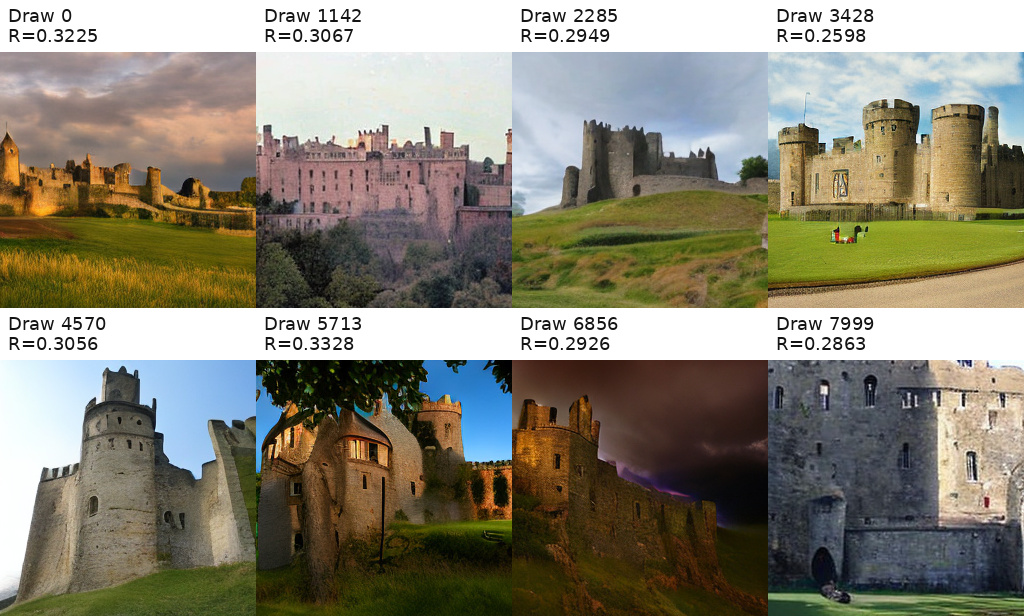}
            \caption{best-of-K.}
            \label{fig:reward_final_castle_sunrise_tau200_r8_l8_v1:imf_best_of_12}
        \end{subfigure}
        \par\medskip

        \begin{subfigure}{\linewidth}
            \centering
            \includegraphics[width=\linewidth]{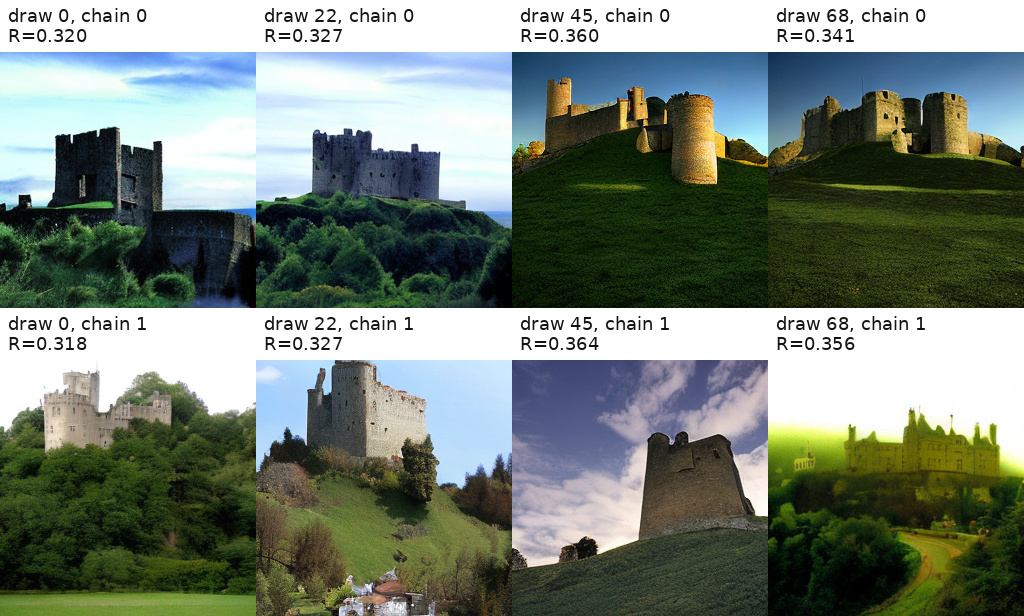}
            \caption{SiT + SPT + pCN.}
            \label{fig:reward_final_castle_sunrise_tau200_r8_l8_v1:sit_spt_pcn}
        \end{subfigure}
        \par\medskip

        \begin{subfigure}{\linewidth}
            \centering
            \includegraphics[width=\linewidth]{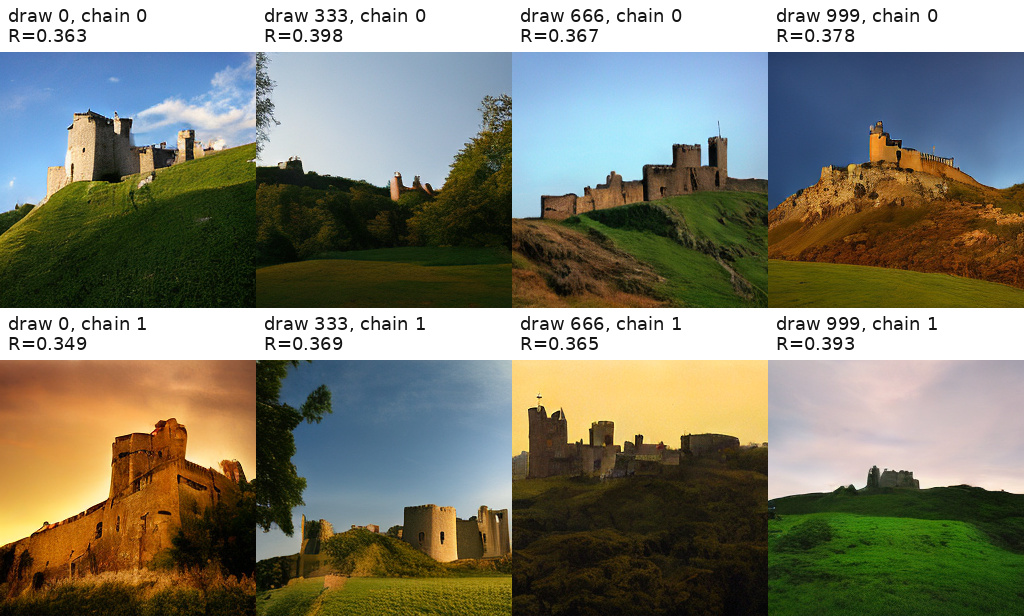}
            \caption{iMF + SPT + pCN.}
            \label{fig:reward_final_castle_sunrise_tau200_r8_l8_v1:imf_spt_pcn}
        \end{subfigure}
        \par\medskip

        \begin{subfigure}{\linewidth}
            \centering
            \includegraphics[width=\linewidth]{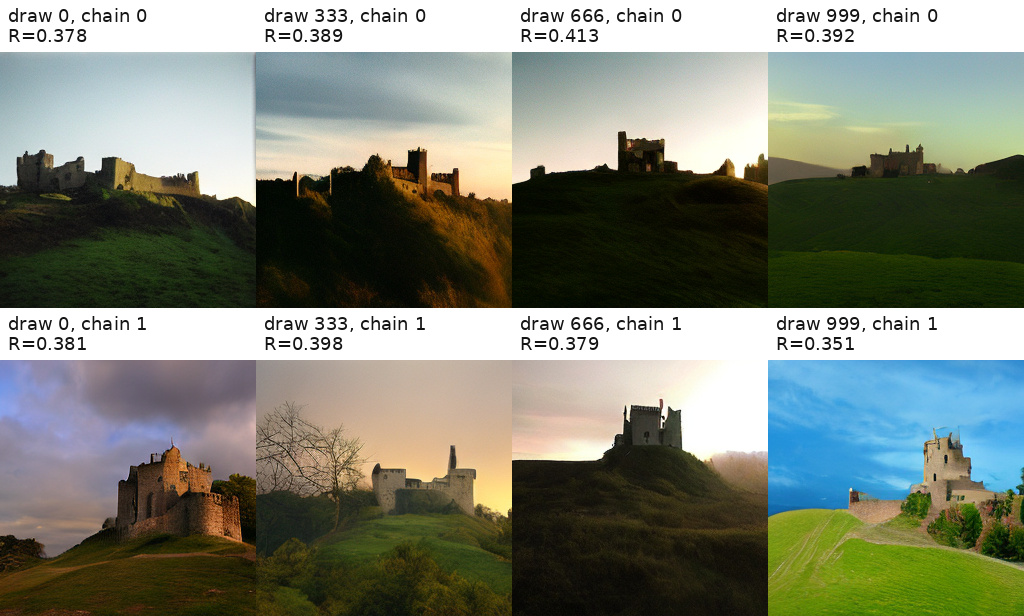}
            \caption{iMF + SPT + hybrid.}
            \label{fig:reward_final_castle_sunrise_tau200_r8_l8_v1:imf_spt_hybrid}
        \end{subfigure}

        \caption{Class: castle, prompt: \textit{``A photograph of a medieval castle on a hill at sunrise, surrounded by green countryside.''}}
        \label{fig:reward_final_castle_sunrise_tau200_r8_l8_v1}
    \end{minipage}%
    \hfill
    % Right: castle snow painting
    \begin{minipage}[t]{0.48\linewidth}
        \vspace{0pt}
        \captionsetup{type=figure}
        \centering

        \begin{subfigure}{\linewidth}
            \centering
            \includegraphics[width=\linewidth]{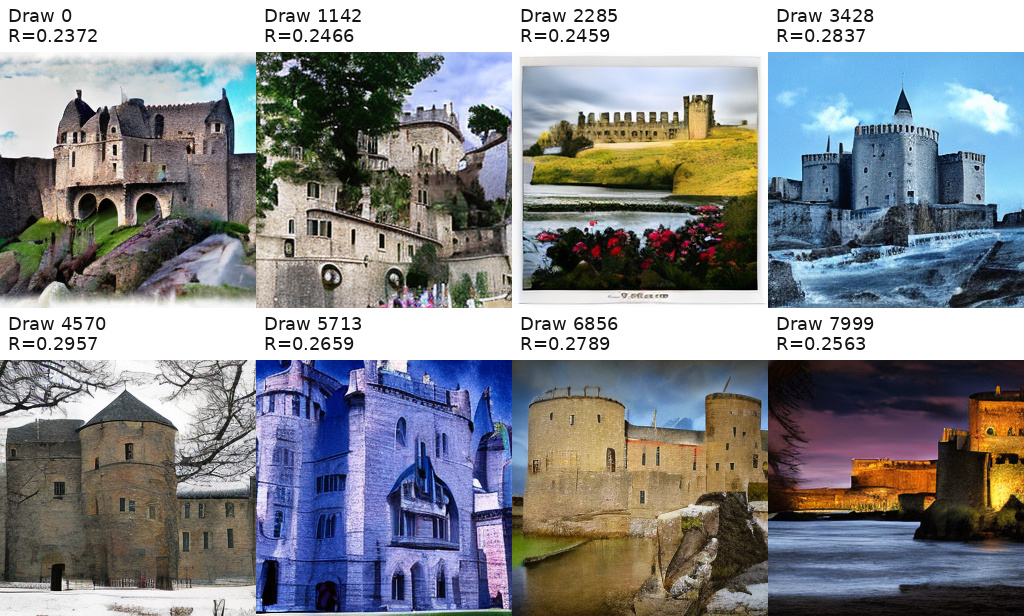}
            \caption{best-of-K.}
            \label{fig:reward_final_castle_snow_painting_tau200_r8_l8_v1:imf_best_of_12}
        \end{subfigure}
        \par\medskip

        \begin{subfigure}{\linewidth}
            \centering
            \includegraphics[width=\linewidth]{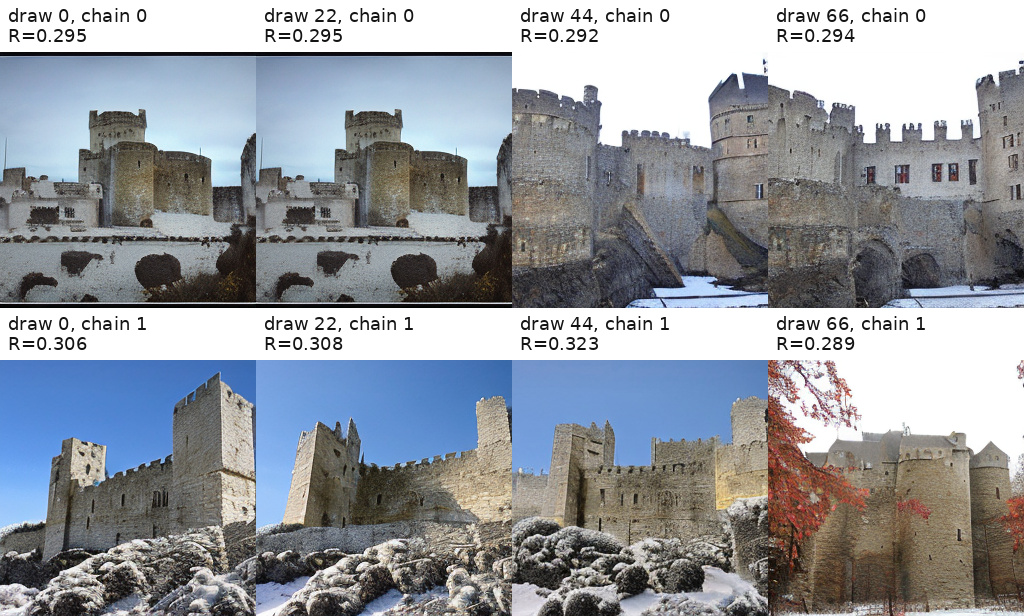}
            \caption{SiT + SPT + pCN.}
            \label{fig:reward_final_castle_snow_painting_tau200_r8_l8_v1:sit_spt_pcn}
        \end{subfigure}
        \par\medskip

        \begin{subfigure}{\linewidth}
            \centering
            \includegraphics[width=\linewidth]{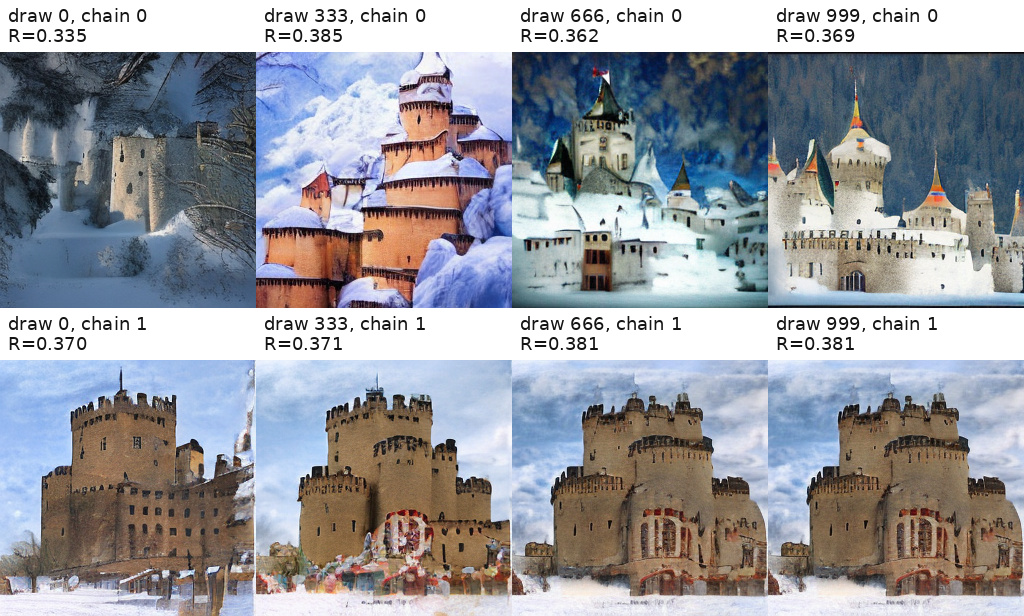}
            \caption{iMF + SPT + pCN.}
            \label{fig:reward_final_castle_snow_painting_tau200_r8_l8_v1:imf_spt_pcn}
        \end{subfigure}
        \par\medskip

        \begin{subfigure}{\linewidth}
            \centering
            \includegraphics[width=\linewidth]{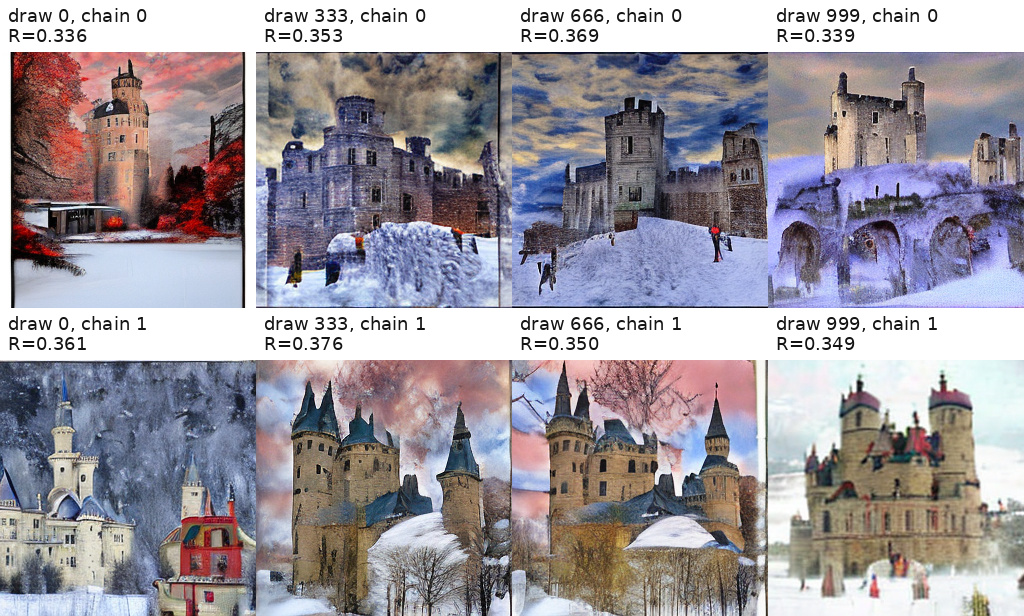}
            \caption{iMF + SPT + hybrid.}
            \label{fig:reward_final_castle_snow_painting_tau200_r8_l8_v1:imf_spt_hybrid}
        \end{subfigure}

        \caption{Class: castle, prompt: \textit{``A colorful painting of a medieval castle covered in snow.''}}
        \label{fig:reward_final_castle_snow_painting_tau200_r8_l8_v1}
    \end{minipage}
\end{figure}

\begin{figure}[!htb]
    \centering
    % Left: aircraft carrier sunset ocean
    \begin{minipage}[t]{0.48\linewidth}
        \vspace{0pt}
        \captionsetup{type=figure}
        \centering

        \begin{subfigure}{\linewidth}
            \centering
            \includegraphics[width=\linewidth]{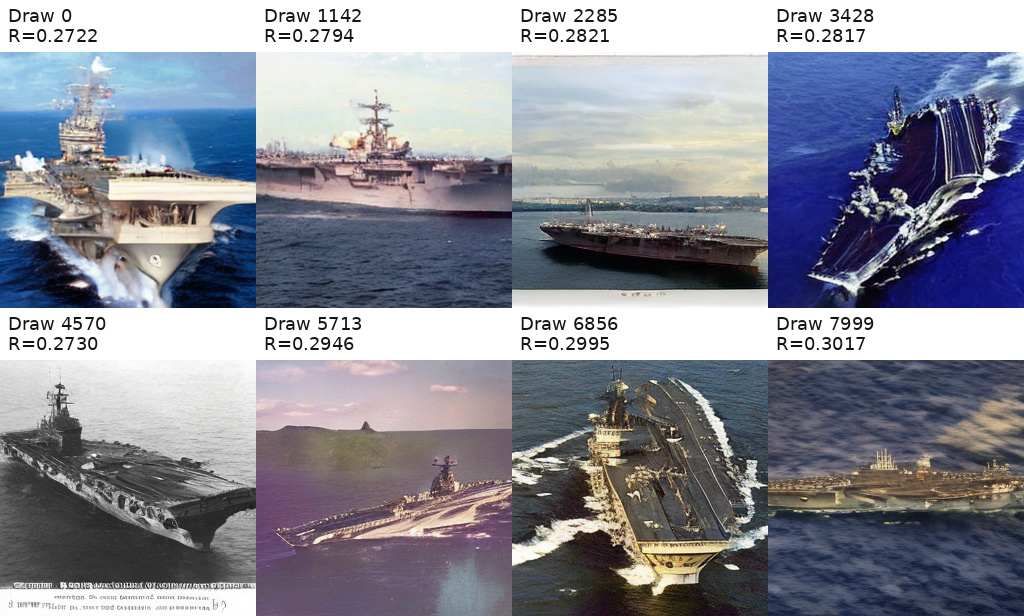}
            \caption{best-of-K.}
            \label{fig:reward_final_aircraft_carrier_sunset_ocean_tau200_r8_l8_v1:imf_best_of_12}
        \end{subfigure}
        \par\medskip

        \begin{subfigure}{\linewidth}
            \centering
            \includegraphics[width=\linewidth]{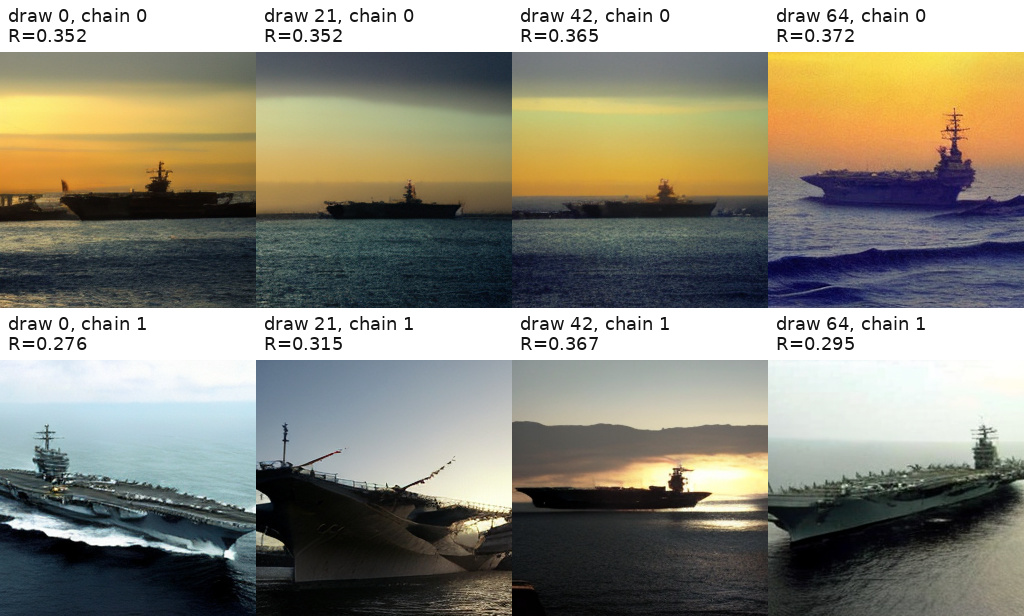}
            \caption{SiT + SPT + pCN.}
            \label{fig:reward_final_aircraft_carrier_sunset_ocean_tau200_r8_l8_v1:sit_spt_pcn}
        \end{subfigure}
        \par\medskip

        \begin{subfigure}{\linewidth}
            \centering
            \includegraphics[width=\linewidth]{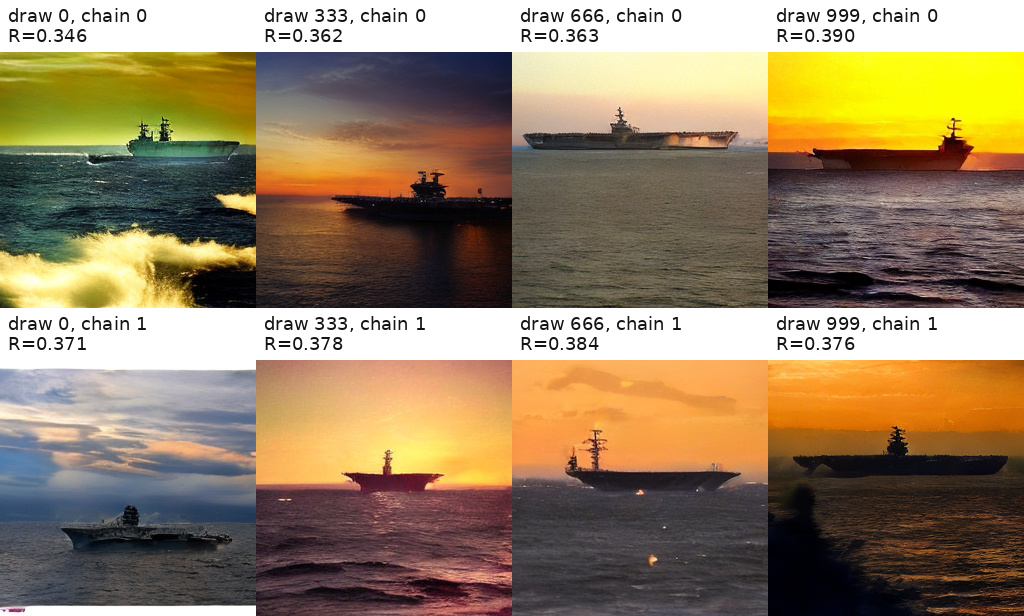}
            \caption{iMF + SPT + pCN.}
            \label{fig:reward_final_aircraft_carrier_sunset_ocean_tau200_r8_l8_v1:imf_spt_pcn}
        \end{subfigure}
        \par\medskip

        \begin{subfigure}{\linewidth}
            \centering
            \includegraphics[width=\linewidth]{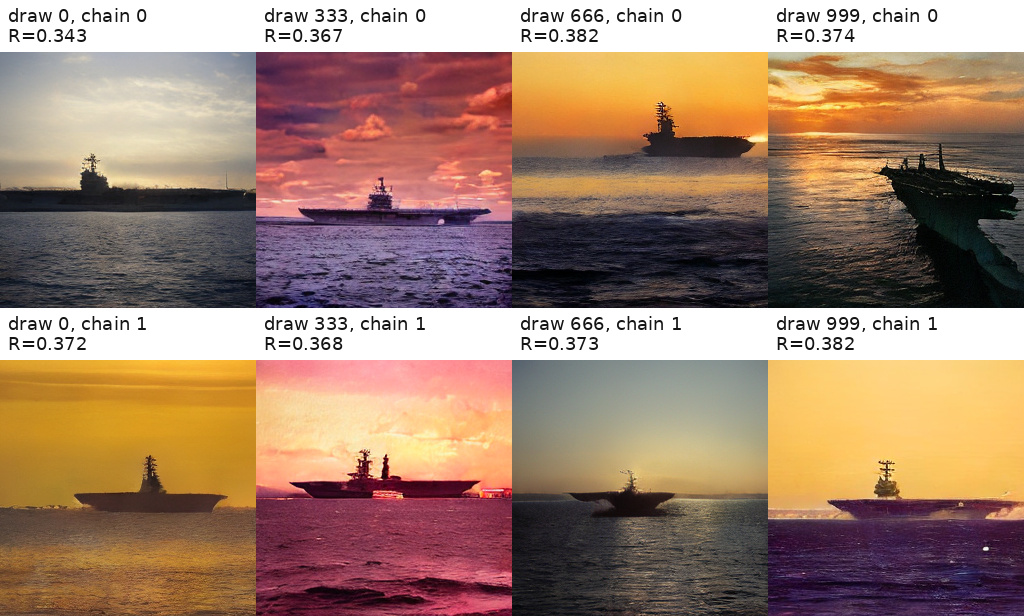}
            \caption{iMF + SPT + hybrid.}
            \label{fig:reward_final_aircraft_carrier_sunset_ocean_tau200_r8_l8_v1:imf_spt_hybrid}
        \end{subfigure}

        \caption{Class: aircraft carrier, prompt: \textit{``A photograph of an aircraft carrier sailing across the ocean at sunset.''}}
        \label{fig:reward_final_aircraft_carrier_sunset_ocean_tau200_r8_l8_v1}
    \end{minipage}%
    \hfill
    % Right: aircraft carrier watercolor
    \begin{minipage}[t]{0.48\linewidth}
        \vspace{0pt}
        \captionsetup{type=figure}
        \centering

        \begin{subfigure}{\linewidth}
            \centering
            \includegraphics[width=\linewidth]{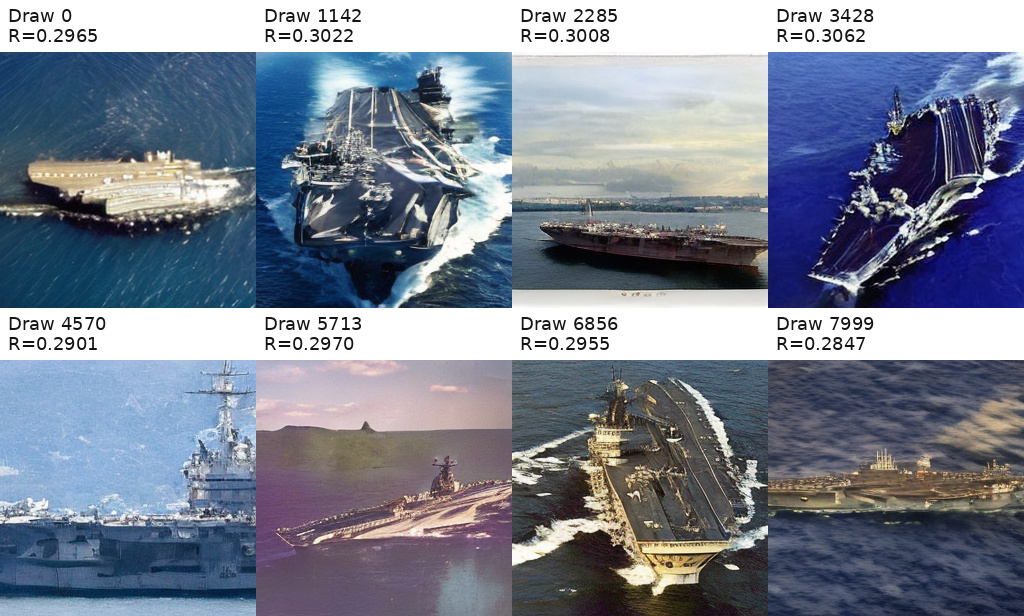}
            \caption{best-of-K.}
            \label{fig:reward_final_aircraft_carrier_watercolor_tau200_r8_l8_v1:imf_best_of_12}
        \end{subfigure}
        \par\medskip

        \begin{subfigure}{\linewidth}
            \centering
            \includegraphics[width=\linewidth]{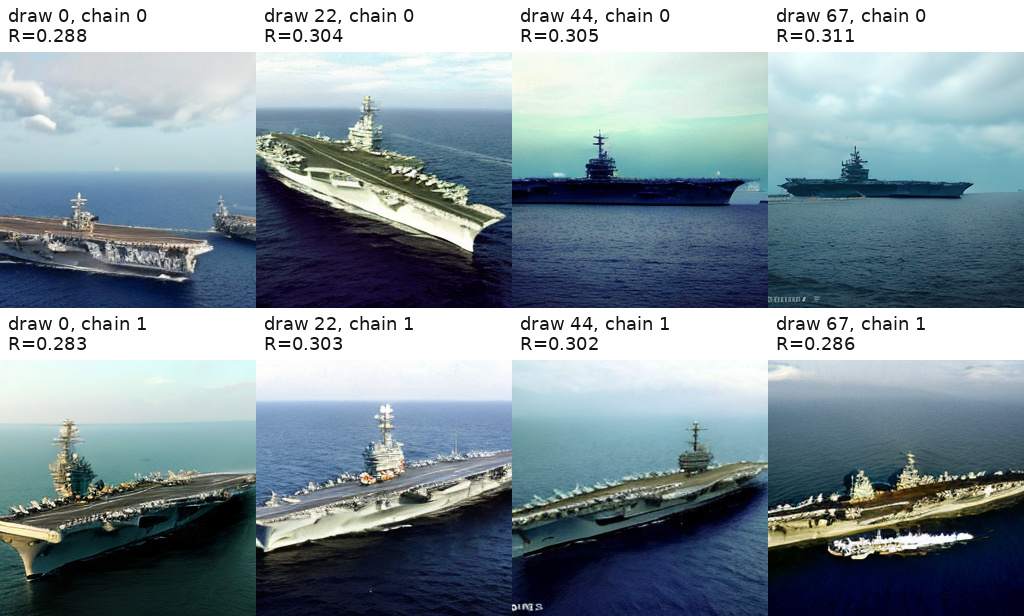}
            \caption{SiT + SPT + pCN.}
            \label{fig:reward_final_aircraft_carrier_watercolor_tau200_r8_l8_v1:sit_spt_pcn}
        \end{subfigure}
        \par\medskip

        \begin{subfigure}{\linewidth}
            \centering
            \includegraphics[width=\linewidth]{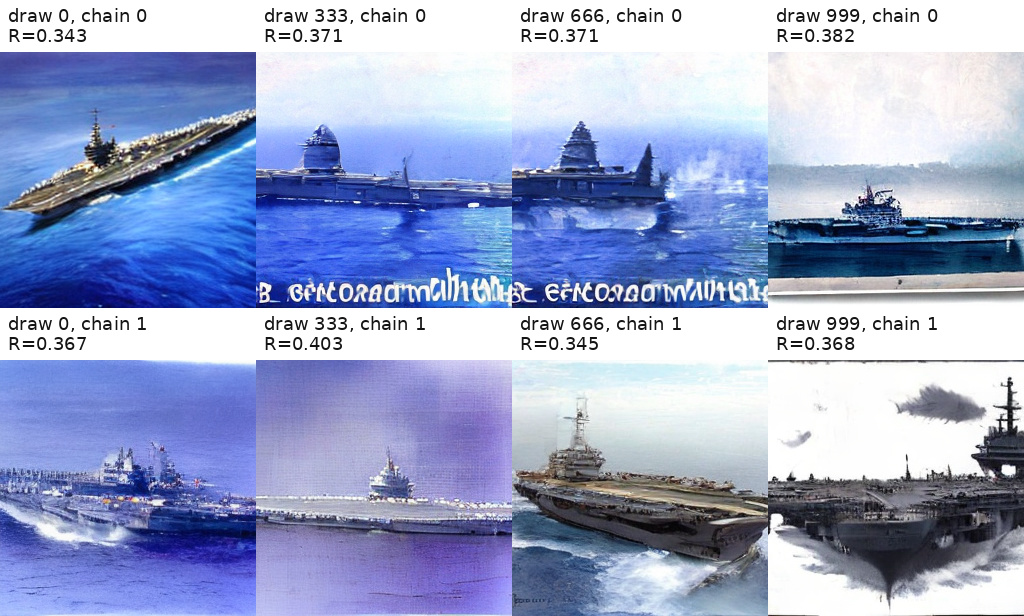}
            \caption{iMF + SPT + pCN.}
            \label{fig:reward_final_aircraft_carrier_watercolor_tau200_r8_l8_v1:imf_spt_pcn}
        \end{subfigure}
        \par\medskip

        \begin{subfigure}{\linewidth}
            \centering
            \includegraphics[width=\linewidth]{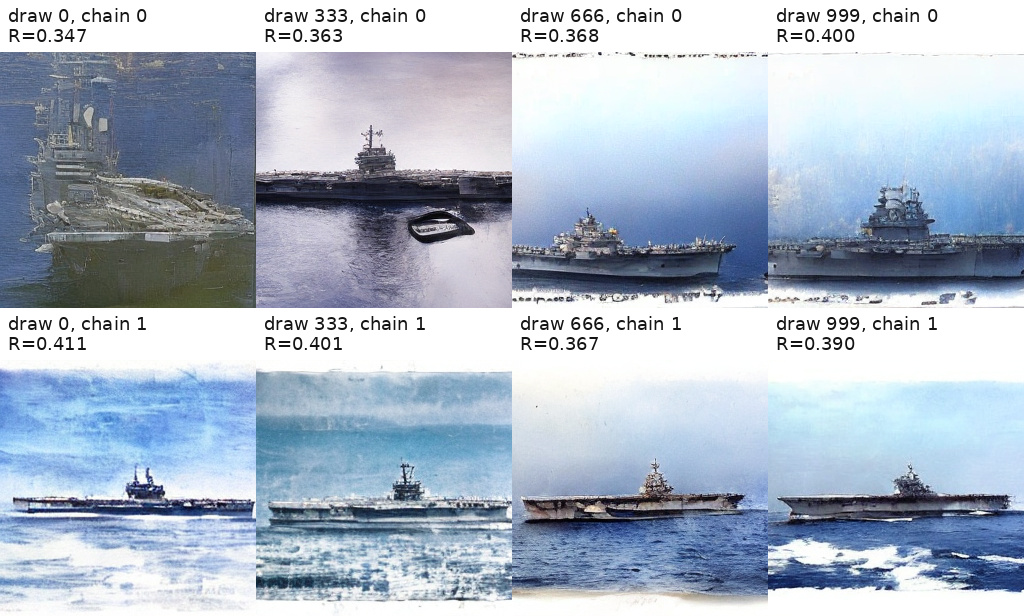}
            \caption{iMF + SPT + hybrid.}
            \label{fig:reward_final_aircraft_carrier_watercolor_tau200_r8_l8_v1:imf_spt_hybrid}
        \end{subfigure}

        \caption{Class: aircraft carrier, prompt: \textit{``A watercolor painting of an aircraft carrier on a calm blue sea.''}}
        \label{fig:reward_final_aircraft_carrier_watercolor_tau200_r8_l8_v1}
    \end{minipage}
\end{figure}

\end{document}